\documentclass[letterpaper]{article}
\usepackage{uai2020}
\usepackage[margin=1in]{geometry}

\usepackage{times}

\usepackage{amsmath,amsbsy,amsgen,amscd,amsthm,amsfonts,amssymb, bm, centernot} 

\usepackage{algorithm, algorithmic}

\usepackage{natbib}
\setcitestyle{open={(},close={)}}

\newcommand{\algo}{{\sc\textsf{DPS}}}
\newcommand{\algolong}{\textsc{Dueling Posterior Sampling}}
\newcommand{\rr}{\overline{r}}
\newcommand{\reg}{\textsc{Reg}}
\newcommand{\adv}{\textsc{Advance}}
\newcommand{\feedback}{\textsc{feedback}}

\usepackage{graphicx}
\usepackage{capt-of}
\usepackage[caption=false,font=footnotesize]{subfig}
\usepackage{enumitem}

\newtheorem{lemma}{Lemma}
\newtheorem{theorem}{Theorem}
\newtheorem{proposition}{Proposition}
\newtheorem{coro}{Corollary}
\newtheorem{fact}{Fact}
\newtheorem{assumption}{Assumption}
\newtheorem{definition}{Definition}
\newtheorem{remark}{Remark}

\usepackage{xcolor}

\title{Dueling Posterior Sampling for Preference-Based Reinforcement Learning}

\author{ {\bf Ellen R. Novoseller\textsuperscript{1}, Yibing Wei\textsuperscript{1}, Yanan Sui\textsuperscript{2}, Yisong Yue\textsuperscript{1}, Joel W. Burdick\textsuperscript{1}} \\
\textsuperscript{1}Department of Computing and Mathematical Sciences, California Institute of Technology, Pasadena, CA 91125 \\
\textsuperscript{2} School of Aerospace Engineering, Tsinghua University, Beijing, China 100084 \\
\{enovoseller, ywwei, yyue\}@caltech.edu, ysui@tsinghua.edu.cn, jwb@robotics.caltech.edu
}

\begin{document}

\maketitle


\begin{abstract}
In preference-based reinforcement learning (RL), an agent interacts with the environment while receiving preferences instead of absolute feedback. While there is increasing research activity in preference-based RL, the design of formal frameworks that admit tractable theoretical analysis remains an open challenge. Building upon ideas from preference-based bandit learning and posterior sampling in RL, we present \algolong~(\algo), which employs preference-based posterior sampling to learn both the system dynamics and the underlying utility function that governs the preference feedback. As preference feedback is provided on trajectories rather than individual state-action pairs, we develop a Bayesian approach for the credit assignment problem, translating  preferences to a posterior distribution over state-action reward models. We prove an asymptotic Bayesian no-regret rate for \algo~with a Bayesian linear regression credit assignment model. This is the first regret guarantee for preference-based RL to our knowledge. We also discuss possible avenues for extending the proof methodology to  other credit assignment models. Finally, we evaluate the approach empirically, showing competitive performance against existing baselines.
\end{abstract}


\section{INTRODUCTION}\label{sec:introduction}

Reinforcement learning (RL) agents interact with humans in many domains, from clinical trials \citep{sui2018stagewise} to autonomous driving \citep{sadigh2017active} to human-robot interaction \citep{kupcsik2018learning}, and take human preferences as feedback. While many RL algorithms assume the existence of a numerical reward signal, in settings involving humans, it is often unclear how to define a reward signal that accurately reflects optimal system-human interaction. For instance, in autonomous driving \citep{basu2017you} and robotics \citep{argall2009survey, akrour2012april}, users can have difficulty with both specifying numerical reward functions and providing demonstrations of desired behavior. Moreover, a misspecified reward function can result in ``reward hacking'' \citep{amodei2016concrete}, in which undesirable actions achieve high rewards. In such situations, the user's \textit{preferences} could more reliably measure her intentions.

This work studies the problem of preference-based reinforcement learning (PBRL), in which the RL agent executes a pair of trajectories of interaction with the environment, and the user provides (noisy) pairwise preference feedback, revealing which of the two trajectories is preferred. Though the study of PBRL has seen increased interest in recent years \citep{christiano2017deep,wirth2017survey},
it remains an open challenge to design formal frameworks that admit tractable theoretical analysis. While the preference-based bandit setting---in which the agent observes preferences between selected actions---has seen significant theoretical progress (e.g., \cite{yue2012k,zoghi2014relative,ailon2014reducing,szorenyi2015online,dudik2015contextual,zoghi2015copeland,ramamohan2016dueling,wu2016double,sui2017multi,sui2018advancements}), the PBRL setting is more challenging, as the environment's dynamics can stochastically translate the agent's policies for interaction (analogous to actions in the bandit setting) to the observed trajectories.

In this paper, we present the \algolong~(\algo) algorithm, which uses preference-based posterior sampling to tackle  PBRL  in the Bayesian regime.  Posterior sampling \citep{thompson1933likelihood}, also called Thompson sampling, is a Bayesian model-based approach to balancing exploration and exploitation, which enables the algorithm to efficiently learn models of both the environment's state transition dynamics and reward function.  Previous work on posterior sampling in RL \citep{osband2013more,gopalan2015thompson,agrawal2017optimistic,osband2017posterior} is focused on learning from absolute rewards, while we extend posterior sampling to both elicit and learn from trajectory-level preference feedback.  

To elicit preference feedback, at every episode of learning, \algo~draws two independent samples from the posterior to generate two trajectories.  This approach is inspired by the Self-Sparring algorithm proposed for the bandit setting \citep{sui2017multi}, but has a quite different theoretical analysis, as we need to incorporate trajectory-level preference learning and state transition dynamics.

Learning from trajectory-level preferences is in general a very challenging problem, as information about the rewards is sparse (often just one bit), is only relative to the pair of trajectories being compared, and does not explicitly include information about actions within trajectories.
\algo~learns from preference feedback by internally maintaining a Bayesian state-action reward model that explains the preferences; this reward model is a solution to the \textit{temporal credit assignment problem} \citep{akrour2012april,zoghi2014relative,szorenyi2015online, christiano2017deep, wirth2016model, wirth2017survey}, i.e., determining which of the encountered states and actions are responsible for the trajectory-level preference feedback. 

We developed \algo~concurrently with an analysis framework for characterizing regret convergence in the episodic setting, based upon information-theoretic techniques for bounding the Bayesian regret of posterior sampling \citep{russo2016information}. We mathematically integrate Bayesian credit assignment and preference elicitation within the conventional posterior sampling framework, evaluate several credit assignment models, and prove a Bayesian asymptotic no-regret rate for \algo~with a Bayesian linear regression credit assignment model.  To our knowledge, this is the first PBRL approach with theoretical guarantees. We also demonstrate that \algo~delivers competitive performance empirically.

\section{RELATED WORK}\label{sec:related_work}

\textbf{Posterior sampling.}
Balancing exploration and exploitation is a key problem in RL. In the episodic learning setting, the agent typically aims to balance exploration and exploitation to minimize its regret, i.e., the gap between the expected total rewards of the agent and the optimal policy. Posterior sampling, first proposed in \citet{thompson1933likelihood}, is a Bayesian model-based approach toward achieving this goal, which iterates between (1) updating the posterior of a Bayesian environment model and (2) sampling from this posterior to select the next policy. In both the bandit and  RL settings, posterior sampling has been demonstrated to perform competitively in experiments and enjoy favorable theoretical regret guarantees \citep{osband2017posterior, osband2013more, agrawal2017optimistic, chapelle2011empirical}.

Our approach builds upon two existing posterior sampling algorithms: Self-Sparring  \citep{sui2017multi} for preference-based bandit learning (also known as dueling bandits \citep{yue2012k}) and  posterior sampling RL \citep{osband2013more}. Self-Sparring maintains a posterior over each action's reward, and in each iteration, draws multiple samples from this posterior to ``duel'' or ``spar'' via preference elicitation. For each set of sampled rewards, the algorithm executes the action with the highest reward sample, obtaining new preferences to update the model posterior.  \cite{sui2017multi}  prove an asymptotic no-regret guarantee for Self-Sparring with independent Beta-Bernoulli reward models for each action.

Within RL, posterior sampling has been applied to the finite-horizon setting with absolute rewards to learn Bayesian posteriors over both the dynamics and rewards. Each posterior sample yields models of both dynamics and rewards, which are used to compute the optimal policy for the sampled system. This policy is executed to get a roll-out trajectory, used to update the dynamics and reward posteriors. In \cite{osband2013more}, the authors show an expected regret of $O(hS\sqrt{AT \text{log}(SAT)})$ after $T$ time-steps, with finite time horizon $h$ and discrete state and action spaces of sizes $S$ and $A$, respectively.

Our theoretical analysis studies the Bayesian linear regression credit assignment model, which most closely resembles Bayesian reward modeling in the linear bandit setting \citep{abbasi2011improved, agrawal2013thompson, abeille2017linear}. While both the PBRL and linear bandit settings apply Bayesian linear regression to recover model parameters, PBRL additionally requires learning the dynamics, determining policies via value iteration, and receiving feedback as  preferences between trajectory pairs. 

Several regret analyses in the linear bandit domain \citep{abbasi2011improved, agrawal2013thompson, abeille2017linear} rely upon martingale concentration properties introduced in \citet{abbasi2011improved}, and depend upon a bound that is not applicable in the preference-based setting (see Appendix \ref{sec:why_info}). Intuitively, these analyses assume that the agent learns about rewards with respect to every observation's feature vector. In contrast, the preference-based setting assumes that only the \textit{difference} in the total rewards of two trajectories affects human preferences. Thus, while the algorithm incurs regret with respect to every sampled trajectory, only differences between compared trajectory feature vectors yield information about rewards.

Our regret analysis takes inspiration from the information-theoretic perspective on Thompson sampling introduced in \cite{russo2016information}, a framework for quantifying Bayesian regret in terms of the information gained at each step about the optimal action. This analysis focuses upon upper-bounding the \textit{information ratio}, which quantifies the trade-off between exploration (via the information gain) and exploitation (via the instantaneous regret) at each step. Several studies \citep{zanette2017information, nikolov2018information} consider extensions of this work to the RL setting, but to our knowledge, it has not previously been applied toward preference-based learning.

\textbf{Preference-based learning.}
Previous work on PBRL has shown successful performance in a number of applications, including Atari games and the Mujoco environment \citep{christiano2017deep}, learning human preferences for autonomous driving \citep{sadigh2017active}, and selecting a robot's controller parameters \citep{kupcsik2018learning, akrour2014programming}. Yet, to our knowledge, the PBRL literature still lacks theoretical guarantees.

Much of the existing work in PBRL handles a distinct setting from ours. While we seek online regret minimization, several existing algorithms minimize the number of preference queries \citep{christiano2017deep, wirth2016model}. Such algorithms, for instance those which apply deep learning, typically assume that many simulations can be cheaply run between preference queries. In contrast, our setting assumes that experimentation is as expensive as preference elicitation; this could include such domains as adaptive experiment design and human-robot interaction without well-understood human dynamics.

Existing approaches for trajectory-level preference-based RL may be broadly divided into three categories \citep{wirth2017efficient}: a) directly optimizing policy parameters \citep{wilson2012bayesian, busa2013preference, kupcsik2018learning}; b) modeling action preferences in each state \citep{furnkranz2012preference}; and c) learning a utility function to characterize the rewards, returns, or values of state-action pairs \citep{wirth2013epmc, wirth2013policy, akrour2012april, wirth2016model, christiano2017deep}. In c), the utility is often modeled as linear in the trajectory features. If those features are defined in terms of visitations to each state-action pair, then utility directly corresponds to the total (undiscounted) reward.

We adopt the third of these paradigms: PBRL with underlying utility functions.
By inferring state-action rewards from preference feedback, one can derive relatively-interpretable reward models and employ such methods as value iteration. In addition, utility-based approaches may be more sample efficient compared to policy search and preference relation methods \citep{wirth2017efficient}, as they extract more information from each observation. Notably, \cite{wilson2012bayesian} learn a Bayesian model over policy parameters, and sample from its posterior to inform actions. From existing PBRL methods, their algorithm perhaps most resembles ours; however, compared to utility-based approaches, policy search methods typically require either more samples or expert knowledge to craft the policy parameters \citep{wirth2017survey, kupcsik2018learning}.

Beyond RL, preference-based learning has been the subject of much research.  The bandit setting \citep{yue2012k,zoghi2014relative,ailon2014reducing,szorenyi2015online,dudik2015contextual,zoghi2015copeland,ramamohan2016dueling,wu2016double,sui2017multi,sui2018advancements} is closest, as it is essentially a single-state variant of RL. Other settings include:  active learning \citep{sadigh2017active,houlsby2011bayesian,eric2008active}, which is focused exclusively on learning an accurate model rather than maximizing utility of decision-making; learning with more structured preference feedback \citep{radlinski2005query,shivaswamy2012online,raman2013stable,shivaswamy2015coactive}, where the learner receives more than one bit of information per preference elicitation; and batch supervised settings such as learning to rank \citep{herbrich1999support,chu2005preference,joachims2005support,burges2005learning,yue2007support,burges2007learning,liu2009learning}.


\section{PROBLEM STATEMENT} \label{sec:problem_statement}

\textbf{Preliminaries.} We consider fixed-horizon Markov Decision Processes (MDPs), in which rewards are replaced by preferences over trajectories. This class of MDPs can be represented as a tuple, $\mathcal{M} = (\mathcal{S}, \mathcal{A}, \phi, p, p_0, h)$, where the state space $\mathcal{S}$ and action space $\mathcal{A}$ are finite sets with cardinalities $S$ and $A$, respectively. The agent episodically interacts with the environment in length-$h$ roll-out trajectories of the form $\tau = \{s_1, \, a_1, \, s_2, \, a_2, \ldots, s_h, \, a_h, \, s_{h + 1} \}$. In the $i$\textsuperscript{th} iteration, the agent executes two trajectory roll-outs $\tau_{i1}$ and $\tau_{i2}$ and observes a preference between them; we use the notation $\tau \succ \tau^\prime$ to indicate a preference for trajectory $\tau$ over $\tau^\prime$. The initial state is sampled from $p_0$, while  $p$ defines the transition dynamics: $s_{t + 1} \sim p(\cdot | s_t, a_t)$. Finally, the function $\phi$ captures the preference feedback generation mechanism: $\phi(\tau,\tau') := P(\tau \succ \tau') \in [0, 1]$.

A \textit{policy}, $\pi: \mathcal{S} \times \{1, \ldots, h\} \longrightarrow \mathcal{A}$, is a (possibly-stochastic) mapping from states and time indices to actions. In each iteration $i$, the agent selects two policies, $\pi_{i1}$ and $\pi_{i2}$, which are rolled out to obtain trajectories $\tau_{i1}$ and $\tau_{i2}$ and preference label $y_i$. We represent each trajectory as a feature vector, where the features record the number of times each state-action pair is visited. In iteration $i$, rolled-out trajectories $\tau_{i1}$ and $\tau_{i2}$ correspond, respectively, to feature vectors $\bm{x}_{i1}, \bm{x}_{i2} \in \mathbb{R}^d$, where $d := SA$ is the total number of state-action pairs, and the $k$\textsuperscript{th} element of $\bm{x}_{ij}$, $j \in \{1, 2\}$, is the number of times that $\tau_{ij}$ visits state-action pair $k$. The preference for iteration $i$ is denoted $y_i := \mathbb{I}_{[\tau_{i2} \succ \tau_{i1}]} - \frac{1}{2} \in \left\{-\frac{1}{2}, \frac{1}{2}\right\}$, where $\mathbb{I}_{[\cdot]}$ denotes the indicator function, so that $P\left(y_i = \frac{1}{2}\right) = 1 - P\left(y_i = -\frac{1}{2}\right) = \phi(\tau_{i2}, \tau_{i1}) - \frac{1}{2}$; there are no ties in any comparisons. Lastly, we define $\bm{x}_i := \bm{x}_{i2} - \bm{x}_{i1}$.

Our analysis builds upon two main assumptions. Firstly, we assume the existence of underlying utilities, quantifying the user's satisfaction with each trajectory:
\begin{assumption}
Each trajectory $\tau$ has utility $\rr(\tau)$, which decomposes additively: $\rr(\tau) \equiv \sum_{t = 1}^h \rr(s_t,a_t)$ for the state-action pairs in $\tau$. Defining $\bm{\rr} \in \mathbb{R}^d$ as the vector of all state-action rewards, $\rr(\tau)$ can also be expressed in terms of $\tau$'s state-action visit counts $\bm{x}$: $\rr(\tau) = \bm{\rr}^T \bm{x}$.
\end{assumption}
Secondly, we assume that the utilities $\rr(\tau)$ are stochastically translated to preferences via the noise model $\phi$, such that the probability of observing $\tau_{i2} \succ \tau_{i1}$ is a function of the \textit{difference} in their utilities. Intuitively, the greater the disparity in two trajectories' utilities, the more accurate the user's preference between them:
\begin{assumption}
$P(\tau_{i2} \succ \tau_{i1}) = \phi(\tau_{i2}, \tau_{i1}) = g(\rr(\tau_{i2}) - \rr(\tau_{i1})) + \frac{1}{2} = g(\bm{\rr}^T \bm{x}_{i2} - \bm{\rr}^T \bm{x}_{i1}) + \frac{1}{2}$, where $g:\mathbb{R} \longrightarrow \left[-\frac{1}{2}, \frac{1}{2}\right]$ is a \emph{link function} such that a)  $g$ is non-decreasing, and b) $g(x) = -g(-x)$ to ensure that $P(\tau \succ \tau^\prime) = 1 - P(\tau^\prime \succ \tau)$. Note that if $\rr(\tau) = \rr(\tau^\prime)$, we have $P(\tau \succ \tau^\prime) = \frac{1}{2}$, and that $P(\tau_{i2} \succ \tau_{i1}) > \frac{1}{2} \Leftrightarrow g(\bm{\rr}^T \bm{x}_i) > 0 \Leftrightarrow \bm{\rr}^T \bm{x}_{i2} > \bm{\rr}^T \bm{x}_{i1}$.
\end{assumption}
For noiseless preferences, $g_{\text{ideal}}(x) := \mathbb{I}_{[x > 0]} - \frac{1}{2}$. Alternatively, the logistic or Bradley-Terry link function is defined as $g_{\text{log}}(x) := [1+\exp(-x/c)]^{-1} - \frac{1}{2}$ with ``temperature'' $c\in(0,\infty)$. Our theoretical analysis assumes the linear link function \citep{ailon2014reducing}: $g_{\text{lin}}(x) := cx$, for $c > 0$ and $x \in [-\frac{1}{2c}, \frac{1}{2c}]$. Then, $\mathbb{E}[y_i] = P(\tau_{i2} \succ \tau_{i1}) - \frac{1}{2} = c\bm{\rr}^T(\bm{x}_{i2} - \bm{x}_{i1})$. Without loss of generality, we set $c = 1$ by subsuming $c$ into $\bm{\rr}$. Denote the observation noise associated with $g_{\text{lin}}$ on iteration $i$ as $\eta_i$, such that $y_i = \bm{\rr}^T(\bm{x}_{i2} - \bm{x}_{i1}) + \eta_i$.

Given a policy $\pi$, we can define the standard RL value function as the expected total utility when starting in state $s$ at step $j$, and following $\pi$:
\begin{equation}
    V_{\pi, j}(s) = \mathbb{E}\left[\sum_{t = j}^h \rr(s_t,\pi(s_t, t)) \,\bigr|\, s_j = s \right].\label{eqn:value}
\end{equation}
The optimal policy $\pi^*$ is then defined as one that maximizes the expected value over all input states: $\pi^* = \sup_\pi \sum_{s \in \mathcal{S}} p_0(s) V_{\pi, 1}(s)$.
Note that $\mathbb{E}_{s_1\sim p_0} \left[V_{\pi,1}(s_1)\right]\equiv \mathbb{E}_{\tau \sim (\pi,\mathcal{M})}\left[\rr(\tau)\right]$. Given  fully specified dynamics and rewards, $p$ and $\rr$, it is straightforward to apply standard dynamic programming approaches such as value iteration to arrive at the optimal policy under $p$ and $\rr$. The learning goal, then, is to infer $p$ and $\rr$ to the extent necessary for good decision-making.

\textbf{Learning problem.}
We quantify the learning agent's performance via its cumulative $T$-step Bayesian regret relative to the optimal policy:
\begin{flalign}\label{eqn:regret}
    \mathbb{E}[\reg(T)] &= \mathbb{E}\Bigg\{\sum_{i = 1}^{\lceil T/(2h) \rceil} \sum_{s \in \mathcal{S}} p_0(s) \bigr[2V_{\pi^*, 1}(s) \nonumber \\ & - V_{\pi_{i1}, 1}(s) - V_{\pi_{i2}, 1}(s)\bigr]\Bigg\}.
\end{flalign}
To minimize regret, the agent must balance exploration (collecting new data) with exploitation (behaving optimally given current knowledge). Over-exploration of bad trajectories will incur large regret, and under-exploration can prevent convergence to optimality. In contrast to the standard regret formulation in RL, at each iteration we measure regret of both selected policies.

\textbf{Assumptions.} We make two further assumptions. The first imposes a regularity condition upon the noise $\eta_i$:
\begin{assumption}\label{assump:sub_Gauss}
The label noise $\eta_i = y_i - \bm{\rr}^T\bm{x}_i$ is conditionally $R$-sub-Gaussian, that is, there exists $R \ge 0$ such that $\forall \, \lambda \in \mathbb{R}$:
\begin{equation*}
    \mathbb{E}\left[e^{\lambda \eta_i} \, \big| \, \bm{x}_1, \ldots, \bm{x}_{i - 1}, \eta_1, \ldots, \eta_{i - 1}\right] \le \exp\left(\frac{\lambda^2 R^2}{2}\right).
\end{equation*}
\end{assumption}
Note that bounded, zero-mean noise lying in an interval of length at most $2R$ is $R$-sub-Gaussian, and that sub-Gaussianity requires $\mathbb{E}[\eta_i \,|\,  \bm{x}_1, \ldots, \bm{x}_i, \eta_1, \ldots, \eta_{i - 1}] = 0$ \citep{abbasi2011improved}. Since $y_i \in \left\{-\frac{1}{2}, \frac{1}{2}\right\}$ and $\mathbb{E}[y_i \,|\, \bm{x}_i] = \bm{\rr}^T\bm{x}_i \in \left[-\frac{1}{2}, \frac{1}{2}\right]$, we must have $\eta_i \in [-1, 1]$. Thus, $\eta_i$ is R-sub-Gaussian with $R \le 1$, provided that $\mathbb{E}[\eta_i \,|\,  \bm{x}_1, \ldots, \bm{x}_i, \eta_1, \ldots, \eta_{i - 1}] = 0$. The latter holds by the assumption that $\mathbb{E}[y_i \,|\, \bm{x}_i] = \bm{\rr}^T\bm{x}_i$.

\begin{assumption}
For some known $S_r < \infty$, $||\bm{\rr}||_2 \le S_r$. 
\end{assumption}

\textbf{Additional notation.}
For random variables $X$ and $X_n$, $n \in \mathbb{N}$, $X_n \overset{D}\longrightarrow X$ denotes that $X_n$ converges to $X$ in distribution. For $\bm{x} \in \mathbb{R}^d$ and positive definite matrix $B \in \mathbb{R}^{d \times d}$, we define the norm $||\bm{x}||_B := \sqrt{\bm{x}^T B \bm{x}}$.


\begin{algorithm}[t]
\caption{\algolong~(\algo)}\label{alg:DPS}
\begin{small}
\begin{algorithmic}
\STATE $\mathcal{H}_0 = \emptyset$ \COMMENT{Initialize history}
\STATE Initialize prior for $f_p$ \COMMENT{Initialize state transition model}
\STATE Initialize prior for $f_r$ \COMMENT{Initialize utility model}
\FOR{$i = 1, 2, \ldots$}
\STATE $\pi_{i1} \leftarrow$ \adv($f_p$, $f_r$)
\STATE $\pi_{i2} \leftarrow$ \adv($f_p$, $f_r$)
\STATE Sample trajectories $\tau_{i1}$ and $\tau_{i2}$ from $\pi_{i1}$ and $\pi_{i2}$
\STATE Observe feedback $y_i = \mathbb{I}_{[\tau_{i2} \succ \tau_{i1}]} - \frac{1}{2}$
\STATE $\mathcal{H}_i = \mathcal{H}_{i - 1} \cup (\tau_{i1}, \tau_{i2}, y_i)$
\STATE $f_p$, $f_r$ = \feedback($\mathcal{H}_i$, $f_p$, $f_r$)
\ENDFOR
\end{algorithmic}
\end{small}
\end{algorithm}

\section{ALGORITHM} \label{sec:algorithm}

As outlined in Algorithm \ref{alg:DPS}, \algolong~(\algo) iterates among three steps: (a) sampling two policies $\pi_{i1}, \pi_{i2}$ from the Bayesian posteriors of the dynamics and utility models (\adv~-- Algorithm \ref{alg:advance}); (b) rolling out $\pi_{i1}$ and $\pi_{i2}$ to obtain trajectories $\tau_{i1}$ and $\tau_{i2}$, and receiving a preference $y_i$ between them; and (c) updating the posterior (\feedback~-- Algorithm \ref{alg:feedback}). In contrast to conventional posterior sampling with absolute feedback, \algo~samples two policies rather than one at each iteration and solves a credit assignment problem to learn from feedback.

\adv~(Algorithm \ref{alg:advance}) samples from the Bayesian posteriors of the dynamics and utility models to select a policy to roll out. The sampled dynamics and utilities form an MDP, for which value iteration derives the optimal policy $\pi$ under the sample.  One can also view $\pi$ as a random function whose randomness depends on the sampling of the dynamics and utility models.
In the Bayesian setting, it can be shown that $\pi$ is sampled according to its posterior probability of being the optimal policy $\pi^*$. Intuitively, peaked (i.e., certain) posteriors lead to less variability when  sampling $\pi$, which implies less exploration, while diffuse (i.e., uncertain) posteriors lead to greater variability when sampling $\pi$, implying more exploration.

\feedback~(Algorithm \ref{alg:feedback}) updates the Bayesian posteriors of the dynamics and utility models based on new data. Updating the dynamics posterior is relatively straightforward, as we assume that the dynamics are fully-observed; we model the dynamics prior via a Dirichlet distribution for each state-action pair, with conjugate multinomial observation likelihoods. In contrast, performing Bayesian inference over state-action utilities from trajectory-level feedback is much more challenging.  We consider a range of approaches (see Appendix \ref{sec:credit_assignment_models}), and found Bayesian linear regression (Section \ref{sec:credit_assignment}) to both perform well and admit tractable analysis within our theoretical framework.

\begin{algorithm}[t]
\caption{\adv: Sample policy from dynamics and utility models}\label{alg:advance}
\begin{small}
\begin{algorithmic}
\STATE {\bfseries Input:} $f_p, f_r$
\STATE Sample $\bm{\tilde{p}} \sim f_p(\cdot)$ \COMMENT{Sample MDP transition dynamics parameters from posterior}
\STATE Sample $\bm{\tilde{r}} \sim f_r(\cdot)$ \COMMENT{Sample utilities from posterior}
\STATE Compute $\pi = \mathrm{argmax}_{\pi} V(\bm{\tilde{p}}, \bm{\tilde{r}})$ \COMMENT{Value iteration yields sampled MDP's optimal policy}
\STATE Return $\pi$
\end{algorithmic}
\end{small}
\end{algorithm}

\begin{algorithm}[t]
\caption{\feedback: Update dynamics and utility models based on new user feedback}\label{alg:feedback}
\begin{small}
\begin{algorithmic}
\STATE {\bfseries Input:} history $\mathcal{H}, f_p, f_r$
\STATE Apply Bayesian update to $f_p$, given $\mathcal{H}$ \COMMENT{Update dynamics model given history}
\STATE Apply Bayesian update to $f_r$, given $\mathcal{H}$ \COMMENT{Update utility model given preferences}
\STATE Return $f_p$, $f_r$
\end{algorithmic}
\end{small}
\end{algorithm}

\subsection{BAYESIAN LINEAR REGRESSION FOR UTILITY INFERENCE AND CREDIT ASSIGNMENT}
\label{sec:credit_assignment}

\textit{Credit assignment} is the problem of inferring which state-action pairs are responsible for observed trajectory-level preferences. We detail a Bayesian linear regression approach to addressing this task in our setting.

Let $n$ be the number of iterations, or trajectory pairs, observed so far. Then, the maximum a posteriori (MAP) estimate of the rewards $\bm{\rr}$ is calculated via ridge regression, similarly to algorithms for the linear bandit setting:
\begin{flalign}
\bm{\hat{r}}_n = M_n^{-1} \sum_{i = 1}^{n - 1} y_i \bm{x}_i, \text{ where } \label{eqn:reward_MAP} \\
M_n = \lambda I + \sum_{i = 1}^{n - 1} \bm{x}_i \bm{x}_i^T, \text{ and } \lambda \ge 1. \label{eqn:M}
\end{flalign}
We perform Thompson sampling as in \cite{agrawal2013thompson} and \cite{abeille2017linear}, such that in iteration $n$, rewards are sampled from the distribution:
\begin{flalign}\label{eqn:defn_beta}
    \bm{\tilde{r}}_{n1}, \bm{\tilde{r}}_{n2} &\sim \mathcal{N}(\bm{\hat{r}}_n, \beta_n(\delta)^2M_n^{-1}), \text{ where } \\
    \beta_n(\delta) &= R \sqrt{2 \log\left(\frac{\det(M_n)^{1/2}\lambda^{-d/2}}{\delta}\right)} + \sqrt{\lambda}S_r \nonumber\\ &\le R \sqrt{d \log\left(\frac{1 + \frac{L^2n}{d\lambda}}{\delta}\right)} + \sqrt{\lambda}S_r, \nonumber
\end{flalign}
and where $\delta \in (0, 1)$ is a failure probability and for all $n$, $||\bm{x}_n||_2 \le L$. Note that $L \le 2h$, since $||\bm{x}_n||_2 = ||\bm{x}_{n2} - \bm{x}_{n1}||_2 \le ||\bm{x}_{n2} - \bm{x}_{n1}||_1 \le ||\bm{x}_{n2}||_1 + ||\bm{x}_{n1}||_1 = 2h$.

The factor $\beta_n(\delta)$, introduced in \cite{abbasi2011improved}, is critical to deriving the theoretical guarantees for posterior sampling with linear bandits in \cite{agrawal2013thompson} and \cite{abeille2017linear}, due to their dependence on Theorems 1 and 2 of \cite{abbasi2011improved}. Our analysis invokes these results as well. Both of the theorems require any noise in the labels $y_n$ to be sub-Gaussian; in our case, sub-Gaussianity holds by Assumption \ref{assump:sub_Gauss}, as we adopted the linear preference noise model with link function $g_{\text{lin}}$.

Our theoretical analysis is quite different from that for linear bandits in \cite{agrawal2013thompson} and \cite{abeille2017linear}, because in our setting, observations $\bm{x}_n$ are \textit{differences} of trajectory feature vectors, policies are chosen via value iteration, and trajectories are obtained by rolling out RL policies while subject to the environment's state transition dynamics.


\section{THEORETICAL RESULTS} \label{sec:theory}

This section sketches our analysis of the asymptotic Bayesian regret of \algo~under a Bayesian linear regression credit assignment model. Appendix \ref{sec:proofs} details the full proof, while Appendix \ref{sec:credit_assignment_extend} discusses possible future extensions to additional credit assignment models.

The analysis follows three main steps: 1) we prove that \algo~is asymptotically-consistent, that is, the probability with which \algo~selects the optimal policy approaches 1 over time (Appendix \ref{sec:asy_consistency}); 2) we asymptotically bound the one-sided Bayesian regret for $\pi_{i2}$  under the setting where, at each iteration $i$,  \algo~only selects policy $\pi_{i2}$, while policy $\pi_{i1}$ is sampled from a fixed distribution over policies
(Appendix \ref{sec:bound_one_sided_fixed}); and lastly, 3) we assume \algo~selects policy $\pi_{i2}$, while the $\pi_{i1}$-distribution is drifting but converging, and then we asymptotically bound the one-sided regret for $\pi_{i2}$ (Appendix \ref{sec:bound_one_sided_converging}). Due to the asymptotic consistency shown in 1), the policies are indeed sampled from converging distributions, and so the asymptotic regret rate in 3) holds.

This outline is inspired by the analysis for Self-Sparring \citep{sui2017multi}; however, because their guarantee is for dueling bandits with independent Beta-Bernoulli reward models for each action, the details of our analysis  are completely different from theirs. Below, we give intuition for each of the three portions of the proof.

\textbf{Asymptotic consistency of \algo.} To prove that \algo~is asymptotically consistent, we first prove that samples of the dynamics and reward parameters converge in distribution to their true values:
\setcounter{proposition}{0}
\begin{proposition}
The sampled dynamics converge in distribution to their true values as the \algo~iteration increases.
\end{proposition}

\textit{Proof sketch.} Applying standard concentration inequalities to the Dirichlet dynamics posterior, one can show that the sampled dynamics converge in distribution to their true values if every state-action pair is visited infinitely-often. The latter condition can be proven via contradiction: assuming that certain state-action pairs are visited finitely-often, \algo~does not receive new information about their rewards. Examining their reward posteriors, we show that \algo~ is guaranteed to eventually sample high enough rewards in the unvisited state-actions that its policies will attempt to reach them.

We also show that with high probability, the sampled rewards exhibit aymptotic consistency:
\setcounter{proposition}{1}
\begin{proposition}
With probability $1 - \delta$, where $\delta$ is a parameter of the Bayesian linear regression model, the sampled rewards converge in distribution to the true reward parameters, $\bm{\overline{r}}$, as the \algo~iteration increases.
\end{proposition}
\textit{Proof sketch.} We leverage Theorem 2 from \cite{abbasi2011improved} (Lemma \ref{lemma:abbasi_yadkori} in Appendix \ref{sec:asy_regret_rate}): under stated conditions and for any $\delta > 0$, with probability $1 - \delta$ and for all $i > 0$, $||\bm{\hat{r}}_i - \bm{\overline{r}}||_{M_i} \le \beta_i(\delta)$. This result defines a high-confidence ellipsoid, which can be linked to the posterior sampling distribution. We demonstrate that it suffices to show that all eigenvalues of the posterior covariance matrix, $\beta_i(\delta)^2 M_i^{-1}$, converge in distribution to zero. This statement is proven via contradiction: we analyze the behavior of posterior sampling if this does not hold. The probability of failure $\delta$ comes entirely from Theorem 2 in \cite{abbasi2011improved}.

From the asymptotic consistency of the dynamics and reward samples, it is straightforward to show that the sampled policies converge to the optimal policy:
\setcounter{theorem}{0}
\begin{theorem}
With probability $1 - \delta$, the sampled policies $\pi_{i1}, \pi_{i2}$ converge in distribution to the optimal policy, $\pi^*$, as $i \longrightarrow \infty$. That is, $P(\pi_{i1} = \pi^*) \longrightarrow 1$ and $P(\pi_{i2} = \pi^*) \longrightarrow 1$ as $i \longrightarrow \infty$.
\end{theorem}

\textbf{Bounding the one-sided regret under a fixed $\pi_{i1}$-distribution.} To analyze the Bayesian regret of \algo, we adapt the information-theoretic posterior sampling analysis in \cite{russo2016information} to the PBRL setting. In comparison to Russo and Van Roy's work, this requires accounting for preference feedback and incorporating state transition dynamics. Their analysis hinges upon defining a quantity called the \textit{information ratio}, which captures the trade-off between exploration and exploitation. In our setting, we define the information ratio corresponding to the one-sided regret of $\pi_{i2}$ as:
\begin{equation*}
    \Gamma_i := \frac{\mathbb{E}_i[y_i^* - y_i]^2}{I_i(\pi^*; (\pi_{i2}, \tau_{i1}, \tau_{i2}, \bm{x}_{i2} - \bm{x}_{i1}, y_i))},
\end{equation*}
where $y_i$ is the label in iteration $i$, $y_i^*$ is the label in iteration $i$ given $\pi_{i2} = \pi^*$, $I(\cdot; \cdot)$ denotes mutual information, and the subscripts $i$ in $\mathbb{E}_i[\cdot]$ and $I_i(\cdot; \cdot)$ indicate conditioning upon the history, as formalized in Appendix \ref{sec:bound_one_sided_fixed}. The ratio $\Gamma_i$ is between the squared instantaneous one-sided regret of $\pi_{i2}$ (exploitation) and the information gained about the optimal policy (exploration).

When $\pi_{i1}$ is drawn from a fixed distribution, we show that analogously to \cite{russo2016information}, the Bayesian one-sided regret $\mathbb{E}[\reg_2(T)]$ for $\pi_{i2}$ can be bounded in terms of an upper bound on $\Gamma_i$:
\setcounter{lemma}{11}
\begin{lemma}
If $\Gamma_i \le \overline{\Gamma}$ almost surely for each $i \in \{1, \ldots, N\}$, where $N$ is the number of \algo~iterations (over which the policies $\pi_{i2}$ take $T = Nh$ actions), then:
\begin{equation*}
    \mathbb{E}[\reg_2(T)] = \mathbb{E}[\reg_2(Nh)] \le \sqrt{\overline{\Gamma} H(\pi^*) N},
\end{equation*}
where $H(\pi^*)$ is the entropy of the optimal policy $\pi^*$. Because there are at most $A^{Sh}$ deterministic policies, $H(\pi^*) \le \log |A^{Sh}| = Sh \log A$. Substituting this,
\begin{equation*}
    \mathbb{E}[\reg_2(T)] \le \sqrt{\overline{\Gamma}Sh N \log A} = \sqrt{\overline{\Gamma}ST \log A}.
\end{equation*}
\end{lemma}
We show that $\Gamma_i$ can be asymptotically upper-bounded such that $\lim_{i \longrightarrow \infty} \Gamma_i \le \frac{SA}{2}$, and consequently:
\setcounter{theorem}{1}
\begin{theorem}
If the policy $\pi_{i1}$ is drawn from a fixed distribution for all $i$, then for the competing policy $\pi_{i2}$, \algo~achieves a one-sided asymptotic Bayesian regret rate of $S\sqrt{\frac{A T \log A}{2}}$.
\end{theorem}

The bounds in Lemma \ref{lemma:regret_inf_ratio} and Theorem \ref{thm:one_sided_bound} are asymptotic rather than finite-time, due to the convergence in distribution of the dynamics. If the dynamics are known a priori, then these would be finite-time guarantees; in fact, to prove Lemma \ref{lemma:regret_inf_ratio}, we first show that under known dynamics, $\Gamma_i \le \frac{SA}{2}$ for all $i$, and then extend the analysis to prove that under converging dynamics, the result still holds asymptotically. Note that in the PBRL setting, it is significantly more difficult to learn the rewards via credit assignment than to learn the dynamics, which are assumed to be fully-observed. Thus, in practice, we expect that \algo~would learn the dynamics much faster than the rewards, and so it is reasonable to consider convergence of the dynamics model only asymptotically.

\textbf{Bounding the one-sided regret under a converging $\pi_{i1}$-distribution.} Finally, we assume that the distribution of $\pi_{i1}$ is no longer fixed, but rather converges to some fixed distribution over deterministic policies. To asymptotically bound the one-sided regret incurred by $\pi_{i2}$, we leverage that when two discrete random variables converge in distribution, such that $X_n \overset{D}\longrightarrow X$ and $Y_n \overset{D}\longrightarrow Y$, their mutual information also converges: $\lim_{n \longrightarrow \infty} I(X_n, Y_n) = I(X, Y)$. This fact allows us to bound the one-sided regret for $\pi_{i2}$ as follows:
\setcounter{lemma}{16}
\begin{lemma}
Assume that the sampling distribution of $\pi_{i1}$ converges to a fixed probability distribution. Then, the information ratio $\Gamma_i$ corresponding to $\pi_{i2}$'s one-sided regret $\mathbb{E}[\reg_2(T)]$ satisfies $\lim_{i \longrightarrow \infty} \, \Gamma_i \le \frac{SA}{2}$.
\end{lemma}

Combining Lemma \ref{lemma:drift_converge} with the asymptotic consistency of sampled policies as shown in Theorem \ref{thm:policy_consistency}, $P(\pi_{i1} = \pi^*) \longrightarrow 1$, yields our main theoretical result:
\setcounter{theorem}{2}
\begin{theorem}
With probability $1 - \delta$, where $\delta$ is a parameter of the Bayesian linear regression model, the expected Bayesian regret $\mathbb{E}[\reg(T)]$ of \algo~achieves an asymptotic rate of $S\sqrt{2A T \log A}$.
\end{theorem}

\textbf{Discussion.} The specific theoretical results presented yield a high-probability asymptotic Bayesian no-regret rate for \algo~under Bayesian linear regression credit assignment. The proof consists of first demonstrating that the algorithm is asymptotically consistent, and then analyzing its information ratio to characterize the Bayesian regret. We adopted this information-theoretic perspective because we found it more amenable to preference-based feedback than other prevalent methods from the linear bandits literature.

In particular, while several existing regret analyses for posterior sampling with linear bandits \citep{agrawal2013thompson, abeille2017linear} are based upon martingale concentration properties derived in \cite{abbasi2011improved}, we found that these techniques cannot readily extend to the preference-feedback setting (Appendix \ref{sec:why_info}). These linear bandit analyses assume that each observation $\bm{x}_i$ that incurs regret contributes fully toward learning the rewards. In contrast, we assume that while regret is incurred with respect to the observations $\bm{x}_{i1}, \bm{x}_{i2}$, learning occurs only with respect to observation \textit{differences}, $\bm{x}_i = \bm{x}_{i2} - \bm{x}_{i1}$. In preference-based learning settings, it is common to make such assumptions as $P(\tau_{i2} \succ \tau_{i1}) = g(\bm{x}_{i2} - \bm{x}_{i1})$, for some function $g$. In comparison to the martingale-based techniques, the information ratio provides a more direct method for quantifying the trade-off between exploration and exploitation.

Theoretically analyzing other credit assignment models, in addition to Bayesian linear regression, is an important direction for future work. We conjecture that our proof methodology could extend toward other asymptotically-consistent credit assignment models. Indeed, recent work \citep{dong2018information} has analyzed the information ratio for more general link functions, including for logistic bandits. It would be interesting to study the information ratio's behavior under general link functions, as well as to characterize its relationship to the dynamics model's convergence.  It would also be interesting to develop methodology for extending the analysis to achieve finite-time convergence guarantees.


\section{EXPERIMENTS} \label{sec:experiments}

\begin{figure*}[ht]
  \centering
\subfloat[][RiverSwim, $c = 0.0001$]{\includegraphics[width = 0.34\linewidth]{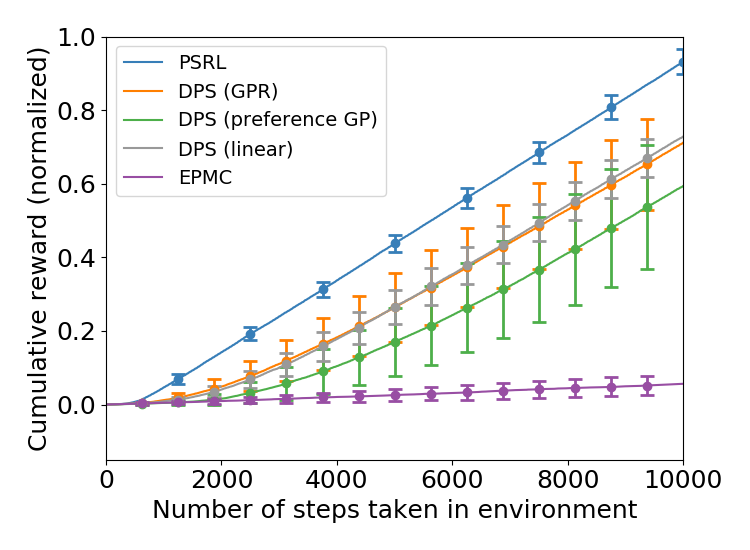}}
\subfloat[][Random MDPs, $c = 0.0001$]{\includegraphics[width = 0.34\linewidth]{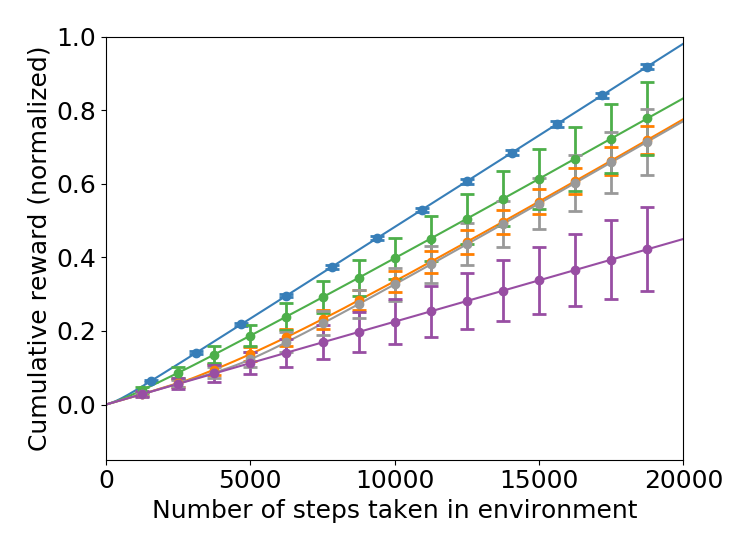}}
\subfloat[][Mountain Car, $c = 0.0001$]{\includegraphics[width = 0.34\linewidth]{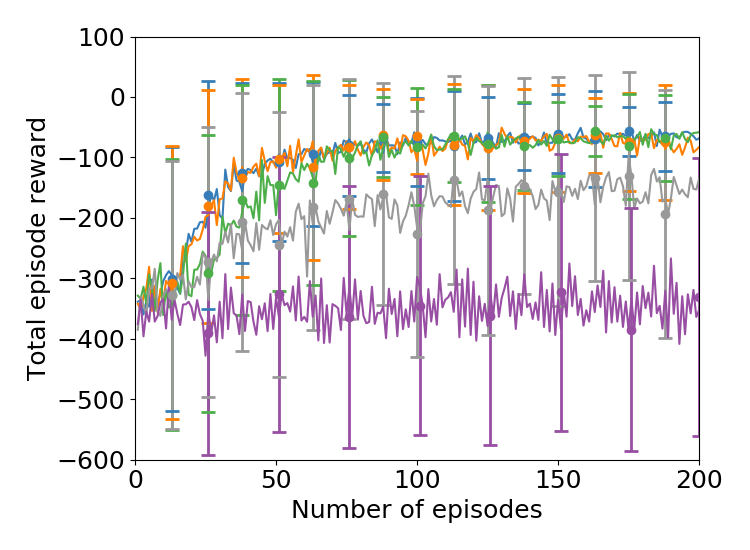}} \\
\subfloat[][RiverSwim, $c = 1$]{\includegraphics[width = 0.34\linewidth]{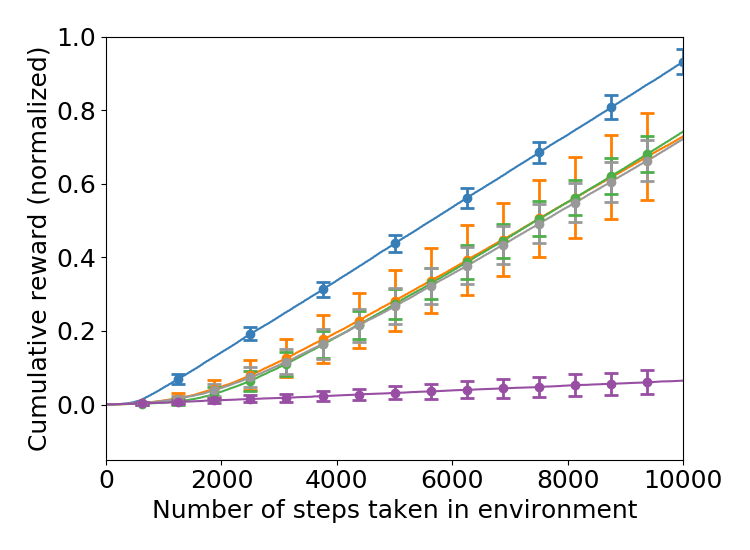}}
\subfloat[][Random MDPs, $c = 1$]{\includegraphics[width = 0.34\linewidth]{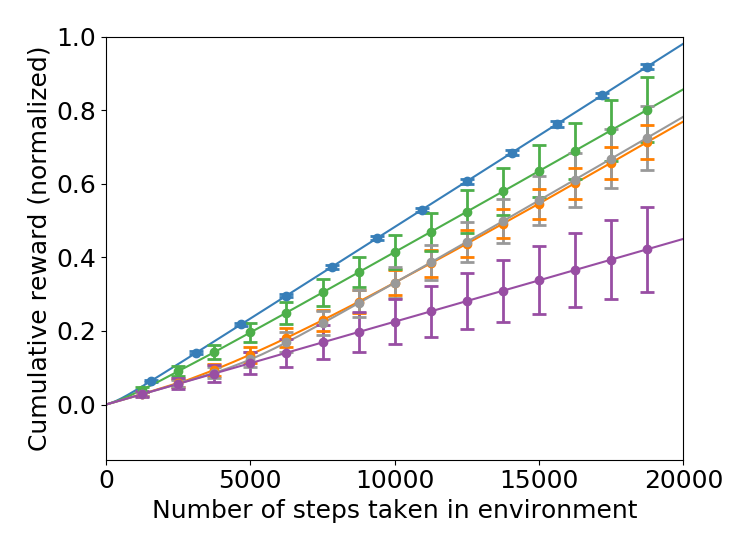}}
\subfloat[][Mountain Car, $c = 0.1$]{\includegraphics[width = 0.34\linewidth]{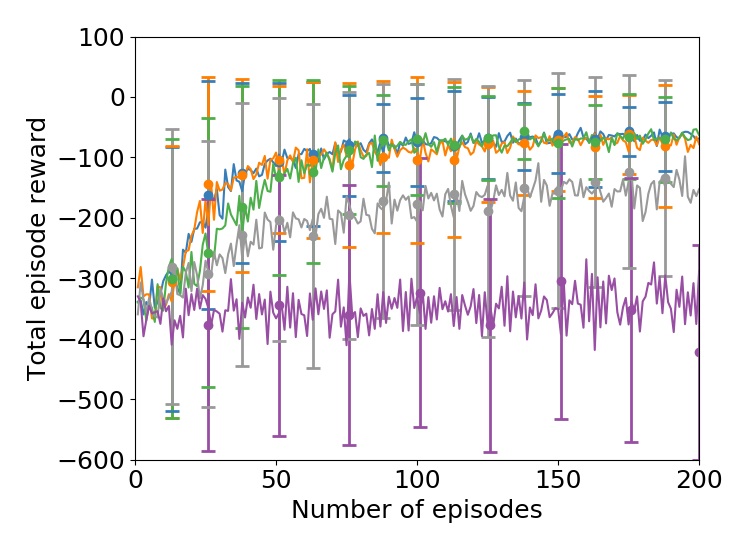}}
  \caption{Empirical performance of \algo; each simulated environment is shown under the two least-noisy user preference models evaluated. The plots show \algo~with three credit assignment models: Gaussian process regression (GPR), Bayesian linear regression, and a Gaussian process preference model. PSRL is an upper bound that receives numerical rewards, while EPMC is a baseline. Plots display the mean +/- one standard deviation over 100 runs of each algorithm tested. The remaining user noise models are plotted in Appendix \ref{sec:additional_experiments}. For RiverSwim and Random MDPs, normalization is with respect to the total reward achieved by the optimal policy. Overall, we see that \algo~performs well and is robust to the choice of credit assignment model.}
  \label{fig:experiments}
\end{figure*}

We validate the empirical performance of \algo~in three simulated domains with varying degrees of preference noise and using three alternative credit assignment models. We find that \algo~generally performs well and compares favorably against standard PBRL baselines.

\textbf{Experimental setup.}
We evaluate on three simulated environments: RiverSwim and random MDPs (described in \cite{osband2013more}) and the Mountain Car problem as detailed in \cite{wirth2017efficient}. The RiverSwim environment has six states and two actions (actions 0 and 1); the optimal policy always chooses action 1, which maximizes the probability of reaching a goal state-action pair. Meanwhile, a suboptimal policy---yielding a small reward compared to the goal---is quickly and easily discovered and incentivizes the agent to always select action 0. The algorithm must demonstrate sufficient exploration to have hope of discovering the optimal policy quickly.

In the second environment, we generate random MDPs with 10 states and 5 actions. The transition dynamics and rewards are respectively generated from Dirichlet (all parameters set to 0.1) and exponential (rate parameter $=$ 5) distributions. These distribution parameters were chosen to generate MDPs with sparse dynamics and rewards. For each random MDP, the sampled reward values were shifted and normalized so that the minimum reward is zero and their mean is one.

Thirdly, in the Mountain Car problem, an under-powered car in a valley must reach the top of a hill by accelerating in both directions to build its momentum. The state space is two-dimensional (position and velocity), while there are three actions (left, right, and neutral). Our implementation begins each episode in a uniformly-random state and has a maximum episode length of 500. We discretize the state space into 10 states in each dimension. Each episode terminates either when the car reaches the goal or after 500 steps, and rewards are -1 in every step.

In each environment, preferences between trajectory pairs were generated by (noisily) comparing their total accrued rewards; this reward information was hidden from the learning algorithm, which observed only the trajectory preferences and state transitions. For trajectories $\tau_i$ and $\tau_j$ with total rewards $\rr(\tau_i)$ and $\rr(\tau_j)$, we consider two models for generating preferences: a) a logistic model, $P(\tau_i \succ \tau_j) = \{1 + \text{exp}[-(\rr(\tau_i) - \rr(\tau_j))/c]\}^{-1}$, and b) a linear model, $P(\tau_i \succ \tau_j) = (\rr(\tau_i) - \rr(\tau_j))/c$, where in both cases, the temperature $c$ controls the degree of noisiness. In the linear case, $c$ is assumed to be large enough that $P(\tau_i \succ \tau_j) \in [0, 1]$. Note that in ties where $\rr(\tau_i) = \rr(\tau_j)$, preferences are uniformly-random.

\textbf{Methods compared.}
We evaluate \algo~under three credit assignment models (Appendix \ref{sec:credit_assignment_models}): 1) Bayesian linear regression, 2) Gaussian process regression, and 3) a Gaussian process preference model. User noise generated via the logistic model has noise levels: $c \in \{10, 2, 1, 0.001\}$ for RiverSwim and random MDPs and $c \in \{100, 20, 10, 0.001\}$ for the Mountain Car. We selected higher values of $c$ for the Mountain Car because $|\rr(\tau_i) - \rr(\tau_j)|$ has a wider range. Additionally, we evaluate the linear preference noise model with $c = 2h\Delta \bm{\rr}$, where $\Delta \bm{\rr}$ is the difference between the maximum and minimum element of $\bm{\rr}$ for each MDP; this choice of $c$ guarantees that $P(\tau_i \succ \tau_j) \in [0, 1]$, but yields noisier preferences than the logistic noise models considered.

As discussed in Section \ref{sec:related_work}, many existing PBRL algorithms handle a somewhat distinct setting from ours, as they assume access to a simulator between preference queries and/or prioritize minimizing preference queries rather than online regret. As a baseline, we evaluate the Every-Visit Preference Monte Carlo (EPMC) algorithm with probabilistic credit assignment \citep{wirth2013policy, wirth2017efficient}. While EPMC does not require simulations between preference queries, it has several limitations, including: 1) the exploration approach always takes uniformly-random actions with some probability, and thus, the authors' plots do not depict online reward accumulation, and 2) EPMC assumes that compared trajectories start in the same state. Lastly, we compare against the posterior sampling RL algorithm (PSRL) from \cite{osband2013more}, which receives the true numerical rewards at each step, and thus upper-bounds the achievable performance of a preference-based algorithm. 

\textbf{Results.}
Figure \ref{fig:experiments} depicts performance curves for the three environments, each with two noise models (Appendix \ref{sec:additional_experiments} contains additional results and details). \algo~performs well in all simulations, and significantly outperforms the EPMC baseline. In RiverSwim, most credit assignment models perform best in the second-to-least-noisy case (logistic noise, $c = 1$), since it is harder to escape the local minimum under the least-noisy preferences. We also see that \algo~is competitive with PSRL, which has access to the full cardinal rewards at each state-action. Additionally, while our theoretical guarantees for \algo~assume fixed-horizon episodes, the Mountain Car results demonstrate that it also succeeds with variable episode lengths. Finally, the performance of \algo~is robust to the choice of credit assignment model, and in fact using Gaussian processes (for which we do not have an end-to-end regret analysis) often leads to the best empirical performance. These results suggest that \algo~is a practically-promising approach that can robustly incorporate many models as subroutines.


\section{CONCLUSION} \label{sec:conclusion}

This work investigates the preference-based reinforcement learning problem, in which an RL agent receives comparative preferences instead of absolute real-valued rewards as feedback. We develop the \algolong~(\algo)~algorithm, which optimizes policies in a highly efficient and flexible way. To our knowledge, \algo~is the first preference-based RL algorithm with a regret guarantee. \algo~also performs well in our simulations, making it both a theoretically-justified and practically-promising algorithm.

There are many directions for future work. Assumptions governing the user's preferences, such as requiring an underlying utility model, could be relaxed. It would also be interesting to extend our theoretical analysis to additional credit assignment approaches and to pursue finite-time guarantees. We expect that \algo~would perform well with any asymptotically-consistent reward model that sufficiently captures users' preference behavior, and hope to develop models that are tractable with larger state and action spaces. For instance, incorporating kernelized input spaces could further improve sample efficiency.

\subsubsection*{Acknowledgments}

This work was supported by NIH grant EB007615 and an Amazon graduate fellowship.

\begin{small}

\end{small}

\newpage

\onecolumn
\appendix
\section*{APPENDICES}
\renewcommand{\thesection}{\Alph{section}}


\setcounter{lemma}{0}
\setcounter{theorem}{0}
\setcounter{proposition}{0}

\section{DERIVATION OF THE ASYMPTOTIC REGRET RATE}\label{sec:proofs}

As outlined in Section \ref{sec:theory}, the analysis follows three main steps:
\begin{enumerate}
    \item Prove that \algo~is asymptotically-consistent, that is, over time, the probability that \algo~selects the optimal policy approaches 1 (Appendix \ref{sec:asy_consistency}).
    \item Assume that in each iteration $i$, policy $\pi_{i1}$ is drawn from a fixed distribution while policy $\pi_{i2}$ is selected by \algo. Then, asymptotically bound the one-sided regret rate for $\pi_{i2}$ (Appendix \ref{sec:bound_one_sided_fixed}).
    \item Assume that policy $\pi_{i1}$ is drawn from a drifting but converging distribution while policy $\pi_{i2}$ is selected by \algo. Then, asymptotically bound the one-sided regret rate for $\pi_{i2}$ (Appendix \ref{sec:bound_one_sided_converging}).
\end{enumerate}

Finally, Appendix \ref{sec:asy_regret_rate} combines these results to asymptotically bound the expected regret rate.

\subsection{ASYMPTOTIC CONSISTENCY OF \algo}\label{sec:asy_consistency}

We show asymptotic consistency of \algo~in three parts: 1) samples from the model posterior over transition dynamics parameters converge in distribution to the true transition probabilities; 2) samples from the reward posterior converge in distribution to the true utilities; and 3) consequently, the sampled policies converge in distribution to the optimal policy.

State transition dynamics are modeled independently for each state-action pair. For a given state-action pair, a Dirichlet model estimates the probability of transitioning to each possible subsequent state. The prior and posterior distributions are both Dirichlet; because the Dirichlet and multinomial distributions are conjugate, each state-action pair's posterior can be updated easily using the observed transitions from that state-action. Each time that \algo~draws a sample from the dynamics distribution, values are sampled for all $S^2 A$ transition parameters, $\{P(s_{t + 1} = s^\prime \,|\, s_t = s, a_t = a) \,|\, s, s^\prime \in \mathcal{S}, a \in \mathcal{A}\}$.

\textbf{Notation.} This section uses the following notation. Let $\bm{\overline{p}} \in \mathbb{R}^{S^2 A}$ be the vector containing all true state transition dynamics parameters, $\{P(s_{t + 1} = s^\prime \,|\, s_t = s, a_t = a) \,|\, s, s^\prime \in \mathcal{S}, a \in \mathcal{A}\}$. Let $\bm{\tilde{p}}_{i1}, \bm{\tilde{p}}_{i2} \in \mathbb{R}^{S^2 A}$ be the two posterior samples of the transition dynamics $\bm{\overline{p}}$ in iteration $i$. Similarly, $\bm{\rr} \in \mathbb{R}^{SA}$ is the vector of true reward parameters, while $\bm{\tilde{r}}_{i1}, \bm{\tilde{r}}_{i2} \in \mathbb{R}^{SA}$ are posterior samples of $\bm{\rr}$ in iteration $i$. For a random variable $X$ and a sequence of random variables $(X_n)$, $n \in \mathbb{N}$, $X_n \overset{D}\longrightarrow X$ denotes that $X_n$ converges to $X$ in distribution, while $X_n \overset{P}\longrightarrow X$ denotes that $X_n$ converges to $X$ in probability. Notation for the value function and for policies given by value iteration follows.

\begin{definition}[Value function given transition dynamics, rewards, and a policy]\label{defn:value_function}
Define $V(\bm{p}, \bm{r}, \pi)$ as the value function over a length-$h$ episode---i.e., the expected total reward in the episode---under transition dynamics $\bm{p} \in \mathbb{R}^{S^2 A}$, rewards $\bm{r} \in \mathbb{R}^{SA}$, and policy $\pi$:
\begin{equation*}
    V(\bm{p}, \bm{r}, \pi) = \sum_{s \in \mathcal{S}} p_0(s) \mathbb{E} \left[ \sum_{t = 1}^h \overline{r}(s_t, \pi(s_t, t)) \,\Big|\, s_1 = s, \bm{\overline{p}} = \bm{p}, \bm{\overline{r}} = \bm{r} \right].
\end{equation*}
\end{definition}

\begin{definition}[Optimal deterministic policy given transition dynamics and rewards]\label{defn:value_iteration_policy}
Define $\pi_{vi}(\bm{p}, \bm{r}) := \mathrm{argmax}_{\pi} V(\bm{p}, \bm{r}, \pi)$ as the optimal \textit{deterministic} policy given transition dynamics $\bm{p} \in \mathbb{R}^{S^2 A}$ and rewards $\bm{r} \in \mathbb{R}^{SA}$ (breaking ties randomly if multiple deterministic policies achieve the maximum). Note that $\pi_{vi}(\bm{p}, \bm{r})$ can be found via finite-horizon value iteration: defining $V_{\pi, t}(s)$ as in \eqref{eqn:value}, set $V_{\pi, h + 1}(s) := 0$ for each $s \in \mathcal{S}$ and use the Bellman equation to calculate $V_{\pi, t}(s)$ successively for $t \in \{h, h - 1, \ldots, 1\}$ given $\bm{p}$ and $\bm{r}$:
\begin{flalign*}
    \pi(s, t) &= \mathrm{argmax}_{a \in \mathcal{A}} \left[\rr(s, a) + \sum_{s^\prime \in \mathcal{S}} P(s_{t + 1} = s^\prime \,|\, s_t = s, a_t = a)V_{\pi, t + 1}(s^\prime) \right], \\
    V_{\pi, t}(s) &= \sum_{a \in \mathcal{A}} \mathbb{I}_{[\pi(s, t) = a]} \left[\rr(s, a) + \sum_{s^\prime \in \mathcal{S}} P(s_{t + 1} = s^\prime \,|\, s_t = s, a_t = a)V_{\pi, t + 1}(s^\prime) \right].
\end{flalign*}

As value iteration results in only deterministic policies, of which there are finitely-many (more precisely, there are $A^{Sh}$), the maximum argument $\pi_{vi}(\bm{p}, \bm{r}) := \mathrm{argmax}_{\pi} V(\bm{p}, \bm{r}, \pi)$ is taken over a finite policy class.
\end{definition}

Finally, we define notation for the eigenvectors and eigenvalues of the matrix $M_i := \lambda I + \sum_{k = 1}^{i - 1} \bm{x}_k \bm{x}_k^T$ (see Equation \eqref{eqn:M}).
\begin{definition}[Eigenvalue notation]\label{defn:evecs_evals}
Let $\lambda_j^{(i)}$ refer to the $j$\textsuperscript{th}-largest eigenvalue of $M_i$, and $\bm{v}_j^{(i)}$ denote its corresponding eigenvector. Note that $M_i^{-1}$ also has eigenvectors $\bm{v}_j^{(i)}$, with corresponding eigenvalues $\frac{1}{\lambda_j^{(i)}}$. Because $M_i$ is positive definite, the eigenvectors $\{\bm{v}_j^{(i)}\}$ form an orthonormal basis, and $\lambda_j^{(i)} > 0$ for all $i, j$.
\end{definition}

We demonstrate convergence in distribution of the sampled transition dynamics parameters. First, Lemma \ref{lemma:dyn_concentration} shows that if every state-action pair is visited infinitely-often, the desired result holds. Then, Lemma \ref{lemma:all_states_actions_visited} completes the argument by showing that \algo~indeed visits each state-action pair infinitely-often.

\begin{lemma}\label{lemma:dyn_concentration}
If every state-action pair is visited infinitely-often, then the sampled transition dynamics parameters converge in distribution to their true values: $\bm{\tilde{p}}_{i1}, \bm{\tilde{p}}_{i2} \overset{D}\longrightarrow \bm{\overline{p}}$ as $i \longrightarrow \infty$, where $\overset{D}\longrightarrow$ denotes convergence in distribution.
\end{lemma}
\begin{proof}
Denote the $d = SA$ state-action pairs as $\tilde{s}_1, \ldots, \tilde{s}_d$. At a particular \algo~episode, let $n_j$ be the number of visits to $\tilde{s}_j$ and $n_{jk}$ be the number of observed transitions from $\tilde{s}_j$ to the $k$\textsuperscript{th} subsequent state. For the $j$\textsuperscript{th} state-action pair at iteration $i$, let $\bm{\overline{p}}^{(j)}$, $\bm{\tilde{p}}^{(j)}$, $\bm{\hat{p}}^{(j)}$, $\bm{\hat{p}}^{\prime(j)} \in \mathbb{R}^S$ be the true, sampled, MAP, and maximum likelihood dynamics parameters, respectively (hiding the dependency on the DPS episode $i1$ or $i2$ for the latter three quantities); thus, $[\bm{\overline{p}}^{(j)}]_k$ denotes the true probability of transitioning from state-action pair $\tilde{s}_j$ to the $k$\textsuperscript{th} state, and analogously for the $k$\textsuperscript{th} elements of $\bm{\tilde{p}}^{(j)}$, $\bm{\hat{p}}^{(j)}$, and $\bm{\hat{p}}^{\prime(j)}$. Then, from the Dirichlet model, 
    \[[\bm{\hat{p}}^{(j)}]_k = \frac{n_{jk} + \alpha_{jk, 0}}{n_j + \sum_{m = 1}^S \alpha_{jm, 0}},\]
where the prior for $\bm{\overline{p}}^{(j)}$ is $\frac{1}{\sum_{m = 1}^S \alpha_{jm, 0}}[\alpha_{j1, 0}, \ldots, \alpha_{jS, 0}]^T$ for user-defined hyperparameters $\alpha_{jk, 0} > 0$. Meanwhile, the maximum likelihood is given by $[\bm{\hat{p}}^{\prime(j)}]_k = \frac{n_{jk}}{\max(n_j, 1)}$ (this is equivalent to $[\bm{\hat{p}}^{(j)}]_k$, except with the prior parameters set to zero). Consider the sampled dynamics at state-action pair $\tilde{s}_j$. For any $\varepsilon > 0$,
\begin{flalign}
P\bigr(||\bm{\tilde{p}}^{(j)} - \bm{\overline{p}}^{(j)}||_1 \ge \varepsilon\bigr) &= P\bigr(||\bm{\tilde{p}}^{(j)} - \bm{\hat{p}}^{(j)} + \bm{\hat{p}}^{(j)} - \bm{\hat{p}}^{\prime(j)} + \bm{\hat{p}}^{\prime(j)} - \bm{\overline{p}}^{(j)}||_1 \ge \varepsilon\bigr) \nonumber \\ &\overset{(a)}{\le} P\bigr(||\bm{\tilde{p}}^{(j)} - \bm{\hat{p}}^{(j)}||_1 + ||\bm{\hat{p}}^{(j)} - \bm{\hat{p}}^{\prime(j)}||_1 + ||\bm{\hat{p}}^{\prime(j)} - \bm{\overline{p}}^{(j)}||_1 \ge \varepsilon\bigr) \nonumber \\ &\le P\Big(||\bm{\tilde{p}}^{(j)} - \bm{\hat{p}}^{(j)}||_1 \ge \frac{\varepsilon}{3} \,\bigcup\, ||\bm{\hat{p}}^{(j)} - \bm{\hat{p}}^{\prime(j)}||_1  \ge \frac{\varepsilon}{3} \,\bigcup\, ||\bm{\hat{p}}^{\prime(j)} - \bm{\overline{p}}^{(j)}||_1  \ge \frac{\varepsilon}{3}\Big) \nonumber \\ &\overset{(b)}{\le} P\Big(||\bm{\tilde{p}}^{(j)} - \bm{\hat{p}}^{(j)}||_1 \ge \frac{\varepsilon}{3}\Big) + P\Big(||\bm{\hat{p}}^{(j)} - \bm{\hat{p}}^{\prime(j)}||_1  \ge \frac{\varepsilon}{3}\Big) + P\Big(||\bm{\hat{p}}^{\prime(j)} - \bm{\overline{p}}^{(j)}||_1  \ge \frac{\varepsilon}{3}\Big), \label{eqn:dyn_concentration}
\end{flalign}
where (a) holds due to the triangle inequality and (b) follows from the union bound. For each term in \eqref{eqn:dyn_concentration}, we will upper-bound the quantity in terms of $n_j$ and show that it decays as $n_j \longrightarrow \infty$, that is, as $\tilde{s}_j$ is visited infinitely-often. For the first term, we will achieve this bound via Chebyshev's inequality:
\begin{flalign*}
P\left(||\bm{\tilde{p}}^{(j)} - \bm{\hat{p}}^{(j)}||_1  \ge \frac{\varepsilon}{3}\right) &\le P\left(\bigcup_{k = 1}^S\left\{\left|[\bm{\tilde{p}}^{(j)}]_k - [\bm{\hat{p}}^{(j)}]_k\right| \ge \frac{\varepsilon}{3S}\right\}\right) \overset{(a)}{\le} \sum_{k = 1}^S P\left(\left|[\bm{\tilde{p}}^{(j)}]_k - [\bm{\hat{p}}^{(j)}]_k\right| \ge \frac{\varepsilon}{3S} \right) \\ &\overset{(b)}{\le} \sum_{k = 1}^S\frac{9S^2}{\varepsilon^2}\mathrm{Var}\left[[\bm{\tilde{p}}^{(j)}]_k\right],
\end{flalign*}
where (a) follows from the union bound and (b) is an application of Chebyshev's inequality. For a Dirichlet variable $X$ with parameters $(\alpha_1, \ldots, \alpha_S)$, $\alpha_k > 0$ for each $k$, the variance of the $k$\textsuperscript{th} component $X_k$ is given by:
\begin{equation*}
    \mathrm{Var}[X_k] = \frac{\tilde{\alpha}_k(1 - \tilde{\alpha}_k)}{1 + \sum_{m = 1}^S \alpha_m} \le \frac{1}{2} * \frac{1}{1 + \sum_{m = 1}^S \alpha_m},
\end{equation*}
where $\tilde{\alpha}_k := \frac{\alpha_k}{\sum_{m = 1}^S \alpha_m}$. In the \algo~algorithm, $\bm{\tilde{p}}^{(j)}$ is drawn from a Dirichlet distribution with parameters $(\alpha_{j1}, \ldots, \alpha_{jS}) = (\alpha_{j1, 0} + n_{j1}, \ldots, \alpha_{jS, 0} + n_{jS})$, and so,
\begin{equation*}
\mathrm{Var}\left[[\bm{\tilde{p}}^{(j)}]_k\right] \le \frac{1}{2} * \frac{1}{1 + \sum_{m = 1}^S \alpha_{jm}} = \frac{1}{2} * \frac{1}{1 + \sum_{m = 1}^S (\alpha_{jm, 0} + n_{jm})} \le \frac{1}{2} * \frac{1}{1 + \sum_{m = 1}^S n_{jm}} = \frac{1}{2(1 + n_j)},
\end{equation*}
and so,
\begin{equation*}
    P\Big(||\bm{\tilde{p}}^{(j)} - \bm{\hat{p}}^{(j)}||_1  \ge \frac{\varepsilon}{3}\Big) \le \sum_{k = 1}^S \frac{9S^2}{\varepsilon^2} \frac{1}{2(1 + n_j)} = \frac{9S^3}{2\varepsilon^2 (1 + n_j)}.
\end{equation*}

Considering the second term in \eqref{eqn:dyn_concentration},
\begin{flalign*}
P\left(||\bm{\hat{p}}^{(j)} - \bm{\hat{p}}^{\prime(j)}||_1  \ge \frac{\varepsilon}{3}\right) &\le P\left(\bigcup_{k = 1}^S\left\{\left|[\bm{\hat{p}}^{(j)} - \bm{\hat{p}}^{\prime(j)}]_k\right| \ge \frac{\varepsilon}{3S}\right\}\right) \overset{(a)}{\le} \sum_{k = 1}^S P\Big(\left|[\bm{\hat{p}}^{(j)}]_k - [\bm{\hat{p}}^{\prime(j)}]_k\right| \ge \frac{\varepsilon}{3S} \Big) \\ &\overset{(b)}{\le} \sum_{k = 1}^S P\left(\frac{\alpha_{jk, 0} + \sum_{m = 1}^S \alpha_{jm, 0}}{n_j + \sum_{m = 1}^S \alpha_{jm, 0}} \ge \frac{\varepsilon}{3S} \right),
\end{flalign*}

where (a) holds via the union bound and (b) follows for $n_j \ge 1$ because when $n_j \ge 1$:
\begin{flalign*}
\left|[\bm{\hat{p}}^{(j)}]_k - [\bm{\hat{p}}^{\prime(j)}]_k\right| &= \left|\frac{n_{jk} + \alpha_{jk, 0}}{n_j + \sum_{m = 1}^S \alpha_{jm, 0}} - \frac{n_{jk}}{n_j}\right| = \left|\frac{\alpha_{jk, 0}}{n_j + \sum_{m = 1}^S \alpha_{jm, 0}} - \frac{n_{jk}\sum_{m = 1}^S \alpha_{jm, 0}}{n_j(n_j + \sum_{m = 1}^S \alpha_{jm, 0})}\right| \\ &\le \frac{\alpha_{jk, 0}}{n_j + \sum_{m = 1}^S \alpha_{jm, 0}} + \frac{n_{jk}}{n_j}\frac{\sum_{m = 1}^S \alpha_{jm, 0}}{n_j + \sum_{m = 1}^S \alpha_{jm, 0}} \le \frac{\alpha_{jk, 0} + \sum_{m = 1}^S \alpha_{jm, 0}}{n_j + \sum_{m = 1}^S \alpha_{jm, 0}}.
\end{flalign*}

For the third term in \eqref{eqn:dyn_concentration}, we apply the following concentration inequality for Dirichlet variables (see Appendix C.1 in \cite{jaksch2010near}):
\begin{equation*}
    P(||\bm{\hat{p}}^{\prime(j)} - \bm{\overline{p}}^{(j)}||_1 \ge \varepsilon) \le (2^S - 2)\exp\left( \frac{-n_j\varepsilon^2}{2} \right).
\end{equation*}
Therefore:
\begin{equation*}
    P\left(||\bm{\hat{p}}^{\prime(j)} - \bm{\overline{p}}^{(j)}||_1 \ge \frac{\varepsilon}{3}\right) \le (2^S - 2)\exp\left( \frac{-n_j\varepsilon^2}{18} \right).    
\end{equation*}

Thus, to upper-bound \eqref{eqn:dyn_concentration}, for any $\varepsilon > 0$:
\begin{equation*}
    P\bigr(||\bm{\tilde{p}}^{(j)} - \bm{\overline{p}}^{(j)}||_1 \ge \varepsilon\bigr) \le \frac{9S^3}{2\varepsilon^2(n_j + 1)} + \sum_{k = 1}^S P\left(\frac{\alpha_{jk, 0} + \sum_{m = 1}^S \alpha_{jm, 0}}{n_j + \sum_{m = 1}^S \alpha_{jm, 0}} \ge \frac{\varepsilon}{3S} \right) + (2^S - 2)\exp\left(\frac{-n_j\varepsilon^2}{18} \right).
\end{equation*}

On the right hand side, the first and third terms clearly decay as $n_j \longrightarrow \infty$. The middle term is identically zero for $n_j$ large enough, since the $\alpha_{jk, 0}$ values are user-defined constants. Given this inequality, it is clear that for any $\varepsilon > 0$, as $n_j \longrightarrow \infty$, $P\bigr(||\bm{\tilde{p}}^{(j)} - \bm{\overline{p}}^{(j)}||_1 \ge \varepsilon\bigr) \longrightarrow 0$. If every state-action pair is visited infinitely-often, then $n_j \longrightarrow \infty$ for each $j$, and so $\bm{\tilde{p}}^{(j)}$ converges in probability to $\bm{\overline{p}}^{(j)}$: $\bm{\tilde{p}}^{(j)} \overset{P}\longrightarrow \bm{\overline{p}}^{(j)}$. Convergence in probability implies convergence in distribution, the desired result.
\end{proof}

To continue proving that \algo's model of the transition dynamics converges, we next  prove the intermediate result that the magnitude of the reward MAP estimate, $||\bm{\hat{r}}_n||_2$, is uniformly upper-bounded:
\vskip 0.1 true in
\begin{lemma}\label{lemma:MAP_bound}
Across all $n\ge 1$, there exists some $b < \infty$ such that estimated reward at \algo~trial $n$ is bounded by $b$: $||\bm{\hat{r}}_n||_2 \le b$.
\end{lemma}
\begin{proof}
Recall that the MAP reward estimate $\bm{\hat{r}}_n$ is the solution to a ridge regression problem:
\begin{equation}\label{eqn:ridge_objective}
    \bm{\hat{r}}_n = \mathrm{arg\,inf}_{\bm{r}}\left\{\sum_{i = 1}^{n - 1} (\bm{x}_i^T\bm{r} - y_i)^2 + \lambda||\bm{r}||_2^2 \right\} = \mathrm{arg\,inf}_{\bm{r}}\left\{\sum_{i = 1}^{n - 1} \left[(\bm{x}_i^T\bm{r} - y_i)^2 + \frac{1}{n - 1}\lambda||\bm{r}||_2^2\right] \right\}.
\end{equation}

We will prove the desired result by contradiction. Assuming that there exists no upper bound $b$, we will identify a subsequence $(\bm{\hat{r}}_{n_i})$ of MAP estimates whose lengths increase unboundedly, but whose directions converge. Then, we will show that such vectors fail to minimize the objective in \eqref{eqn:ridge_objective}, achieving a contradiction.

Because there exist finitely-many state-action pairs, there are finitely-many possible length-$h$ trajectories.  Hence, the vector $\bm{x}_i$ in Equation (\ref{eqn:ridge_objective}) can take finitely-many possible values. The binary labels $y_i$ take values in $\left\{-\frac{1}{2}, \frac{1}{2} \right\}$. Note that for $\bm{r} = 0$, $(\bm{x}_i^T\bm{r} - y_i)^2 + \frac{1}{n - 1}\lambda||\bm{r}||_2^2 = \frac{1}{4}$. We prove the desired statement by contradiction: assume that there is no $b < \infty$ such that $||\bm{\hat{r}}_n||_2 \le b$ for all $n$. Then, the sequence $\bm{\hat{r}}_1, \bm{\hat{r}}_2, \ldots$ must have a subsequence indexed by $(n_i)$ such that $\lim_{i \longrightarrow \infty} ||\bm{\hat{r}}_{n_i}||_2 = \infty$. Consider the sequence of unit vectors $\frac{\bm{\hat{r}}_{n_i}}{||\bm{\hat{r}}_{n_i}||_2}$. This sequence lies within the compact set of unit vectors in $\mathbb{R}^d$, and so it must have a convergent subsequence; we index this subsequence of the sequence $(n_i)$ by $(n_{i_j})$. Then, the sequence $(\bm{\hat{r}}_{i_j})$ is such that $\lim_{j \longrightarrow \infty}||\bm{\hat{r}}_{i_j}||_2 = \infty$ and $\lim_{j \longrightarrow \infty}\frac{\bm{\hat{r}}_{i_j}}{||\bm{\hat{r}}_{i_j}||_2} = \bm{\hat{r}}_{unit}$, where $\bm{\hat{r}}_{unit} \in \mathbb{R}^d$ is a fixed unit vector. 

For any $\bm{x}_i$ such that $|\bm{x}_i^T\bm{\hat{r}}_{unit}| \neq 0$, $\lim_{n_{i_j} \longrightarrow \infty}(\bm{x}_i^T \bm{\hat{r}}_{n_{i_j}} - y_i)^2 = \infty$, and so the corresponding terms in \eqref{eqn:ridge_objective} approach infinity.  However, a lower value of the optimization objective in (\ref{eqn:ridge_objective}) can be realized by replacing $\bm{\hat{r}}_{n_{i_j}}$ with the assignment $\bm{r} = 0$. Meanwhile, for any $\bm{x}_i$ such that $|\bm{x}_i^T\bm{\hat{r}}| = 0$, replacing $\bm{\hat{r}}_{n_{i_j}}$ with $\bm{r} = 0$ would also decrease the value of the optimization objective in \eqref{eqn:ridge_objective}. Therefore, for large $j$, $\bm{r} = 0$ results in a smaller objective function value than $\bm{\hat{r}}_{n_{i_j}}$. This is a contradiction, and so the elements of the sequence $\bm{\hat{r}}_{n_{i_j}}$ cannot have arbitrarily-large magnitudes. Thus, the elements of the original sequence $\bm{\hat{r}}_i$ also cannot become arbitrarily large, and $||\bm{\hat{r}}_i|| \le b$ for some $b < \infty$.
\end{proof}

To finish proving convergence of the transition dynamics Bayesian model, we show that every state-action pair is visited infinitely-often.

\begin{lemma}\label{lemma:all_states_actions_visited}
Under \algo, every state-action pair is visited infinitely-often.
\end{lemma}

\begin{proof}
The proof proceeds by assuming that there exists a state-action pair that is visited only finitely-many times.  This assumption will lead to a contradiction\footnote{Note that in finite-horizon MDPs, the concept of visiting a state finitely-many times is not the same as that of a transient state in an infinite Markov chain, because: 1)  due to a finite horizon, the state is resampled from the initial state distribution $p_0(s)$ every $h$ timesteps, and 2) the policy---which determines which state-action pairs can be reached in an episode---is also resampled every $h$ timesteps.}: once this state-action pair is no longer visited, the reward model posterior is no longer updated with respect to it. Then, \algo~is guaranteed to eventually sample a high enough reward for this state-action that the resultant policy will prioritize visiting it.   

First we note that \algo~is guaranteed to reach at least one state-action pair infinitely often: given our problem's finite state and action spaces, at least one state-action pair must be visited infinitely-often during DPS execution. If all state-actions are \textit{not} visited infinitely-often, there must exist a state-action pair $(s, a)$ such that $s$ is visited infinitely-often, while $(s, a)$ is not. Otherwise, if all actions are selected infinitely-often in all infinitely-visited states, the finitely-visited states are unreachable (in which case these states are irrelevant to the learning process and regret minimization, and can be ignored). Without loss of generality, we label this state-action pair $(s, a)$ as $\tilde{s}_1$. To reach a contradiction, it suffices to show that $\tilde{s}_1$ is visited infinitely-often.

Let $\bm{r}_1$ be the reward vector with a reward of $1$ in state-action pair $\tilde{s}_1$ and rewards of zero elsewhere. From Definition \ref{defn:value_iteration_policy}, $\pi_{vi}(\bm{\tilde{p}}, \bm{r}_1)$ is the policy that maximizes the expected number of visits to $\tilde{s}_1$ under dynamics $\bm{\tilde{p}}$ and reward vector $\bm{r}_1$:
\begin{equation*}
    \pi_{vi}(\bm{\tilde{p}}, \bm{r}_1) = \text{argmax}_\pi V(\bm{\tilde{p}}, \bm{r}_1, \pi),
\end{equation*}
where $V(\bm{\tilde{p}}, \bm{r}_1, \pi)$ is the expected total reward of a length-$h$ trajectory under $\bm{\tilde{p}}, \bm{r}_1$, and $\pi$, or equivalently (by definition of $\bm{r}_1$), the expected number of visits to state-action $\tilde{s}_1$.

We next show that there exists a $\rho > 0$ such that $P(\pi = \pi_{vi}(\bm{\tilde{p}}, \bm{r}_1)) > \rho$ for all possible values of $\bm{\tilde{p}}$. That is, for any sampled parameters $\bm{\tilde{p}}$, the probability of selecting policy $\pi_{vi}(\bm{\tilde{p}}, \bm{r}_1)$ is uniformly lower-bounded, implying that \algo ~must eventually select $\pi_{vi}(\bm{\tilde{p}}, \bm{r}_1)$.

Let $\tilde{r}_j$ be the sampled reward associated with state-action pair $\tilde{s}_j$ in a particular \algo~episode, for each state-action $j \in \{1, \ldots, d\}$, with $d=SA$. We show that conditioned on $\bm{\tilde{p}}$, there exists $v > 0$ such that if $\tilde{r}_1$ exceeds $\text{max}\{v \tilde{r}_2, v \tilde{r}_3, \ldots, v \tilde{r}_d\}$, then value iteration returns the policy $\pi_{vi}(\bm{\tilde{p}}, \bm{r}_1)$, which is the policy maximizing the expected amount of time spent in $\tilde{s}_1$. This can be seen by setting $v := \frac{h}{\rho_1}$, where $h$ is the time horizon and $\rho_1$ is the expected number of visits to $\tilde{s}_1$ under $\pi_{vi}(\bm{\tilde{p}}, \bm{r}_1)$. Under this definition of $v$, the event $\left\{\tilde{r}_{1} \ge \text{max}\{v \tilde{r}_2, v \tilde{r}_3, \ldots, v \tilde{r}_d\}\right\}$ is equivalent to $\{\tilde{r}_{1}\rho_1 \ge h \, \text{max}\{\tilde{r}_2, \tilde{r}_3, \ldots, \tilde{r}_d\}\}$; the latter inequality implies that given $\bm{\tilde{p}}$ and $\bm{\tilde{r}}$, the expected reward accumulated solely in state-action $\tilde{s}_1$ exceeds the reward gained by repeatedly (during all $h$ time-steps) visiting the state-action pair in the set $\{\tilde{s}_2, \ldots, \tilde{s}_d\}$ having the highest sampled reward. Clearly, in this situation, value iteration results in the policy $\pi_{vi}(\bm{\tilde{p}}, \bm{r}_1)$.

Next we show that $v = \frac{h}{\rho_1}$ is continuous in the sampled dynamics $\bm{\tilde{p}}$ by showing that $\rho_1$ is continuous in $\bm{\tilde{p}}$. Recall that $\rho_1$ is defined as expected number of visits to $\tilde{s}_1$ under $\pi_{vi}(\bm{\tilde{p}}, \bm{r}_1)$. This is equivalent to the expected reward for following $\pi_{vi}(\bm{\tilde{p}}, \bm{r}_1)$ under dynamics $\bm{\tilde{p}}$ and rewards $\bm{r}_1$:
\begin{flalign}\label{eqn:defn_p_1}
\rho_1 = V(\bm{\tilde{p}}, \bm{r}_1, \pi_{vi}(\bm{\tilde{p}}, \bm{r}_1)) = \max_\pi V(\bm{\tilde{p}}, \bm{r}_1, \pi).
\end{flalign}
The value of any policy $\pi$ is continuous in the transition dynamics parameters, and so $V(\bm{\tilde{p}}, \bm{r}_1, \pi)$ is continuous in $\bm{\tilde{p}}$. The maximum in \eqref{eqn:defn_p_1} is taken over the finite set of deterministic policies; because a maximum over a finite number of continuous functions is also continuous, $\rho_1$ is continuous in $\bm{\tilde{p}}$.

Next, recall that a continuous function on a compact set achieves its maximum and minimum values on that set. The set of all possible dynamics parameters $\bm{\tilde{p}}$ is such that for each state-action pair $j$, $\sum_{k = 1}^S p_{jk} = 1$ and $p_{jk} \ge 0 \, \forall \, k$; the set of all possible vectors $\bm{\tilde{p}}$ is clearly closed and bounded, and hence compact. Therefore, $v$ achieves its maximum and minimum values on this set, and so for any $\bm{\tilde{p}}$, $v \in [v_{\text{min}}, v_{\text{max}}]$, where $v_{\text{min}} > 0$ ($v$ is nonnegative by definition, and $v = 0$ is impossible, as it would imply that $\tilde{s}_1$ is unreachable).

Then, $P(\pi = \pi_{vi}(\bm{\tilde{p}}, \bm{r}_1))$ can then be expressed in terms of $v$ and the parameters of the reward posterior. Firstly,
\begin{flalign*}
P(\pi = \pi_{vi}(\bm{\tilde{p}}, \bm{r}_1)) \ge P(\tilde{r}_1 > \text{max}\{v \tilde{r}_2, v \tilde{r}_3, \ldots, v \tilde{r}_d\}) \ge \prod_{j = 2}^d P(\tilde{r}_1 > v \tilde{r}_j) = \prod_{j = 2}^d [1 - P(\tilde{r}_1 - v \tilde{r}_j \le 0)].
\end{flalign*}

In the $n^{th}$ \algo~iteration, the sampled rewards are drawn from a jointly Gaussian posterior: $\bm{\tilde{r}} \sim \mathcal{N}(\bm{\mu}^{(n)}, \Sigma^{(n)})$ for some $\bm{\mu}^{(n)}$ and $\Sigma^{(n)}$, where $[\bm{\mu}^{(n)}]_j = \mu_j^{(n)}$ and $[\Sigma^{(n)}]_{jk} = \Sigma_{jk}^{(n)}$. Then, $(\tilde{r}_1 - v \tilde{r}_j) \sim \mathcal{N}(\mu_1^{(n)} - v \mu_j^{(n)}, \, \Sigma_{11}^{(n)} + v^2\Sigma_{jj}^{(n)} - 2v\Sigma_{1j}^{(n)})$, so that:
\begin{equation}\label{eqn:prob_in_terms_of_v}
P(\pi_{n1} = \pi_{vi}(\bm{\tilde{p}}, \bm{r}_1)) \ge \prod_{j = 2}^d   \left[1 - \Phi\left(\frac{-\mu_1^{(n)} +  v\mu_j^{(n)}}{\sqrt{\Sigma_{11}^{(n)} + v^2 \Sigma_{jj}^{(n)} - 2 v  \Sigma_{1j}^{(n)}}} \right)\right] 
= \prod_{j = 2}^d   \Phi\left(\frac{\mu_1^{(n)} - v\mu_j^{(n)}}{\sqrt{\Sigma_{11}^{(n)} + v^2 
 \Sigma_{jj}^{(n)} - 2 v \Sigma_{1j}^{(n)}}} \right),
\end{equation}
where $\Phi$ is the standard Gaussian cumulative distribution function. For the right-hand expression in \eqref{eqn:prob_in_terms_of_v} to have a lower bound greater than zero, the argument of $\Phi(\cdot)$ must be lower-bounded. It suffices to upper-bound the numerator's magnitude and to lower-bound the denominator above zero for each product factor $j$ and over all iterations $n$.

The numerator can be upper-bounded using Lemma \ref{lemma:MAP_bound}. Since $\bm{\mu}^{(n)}$ is equal to the MAP reward estimate at iteration $n$, $||\bm{\mu}^{(n)}||_2 \le b$, and so $|\mu_1^{(n)}|, |\mu_j^{(n)}| \le b$. Because $0 < v \le v_{\text{max}}$, $|\mu_1 - v\mu_j| \le |\mu_1^{(n)}| + v|\mu_j^{(n)}| \le (1 + v_{\text{max}})b$.

To lower-bound the denominator, first note that the reward model's posterior covariance is equal to $\beta_n(\delta)^2 M_n^{-1}$, with $M_n$ and $\beta_n(\delta)$ as defined in Equations \eqref{eqn:M} and \eqref{eqn:defn_beta}, respectively; however, because $\beta_n(\delta)^2$ is non-decreasing in $n$, it suffices to prove the statement while ignoring the $\beta_n(\delta)^2$ factor. Thus, to prove this lemma, we can set $\Sigma^{(n)} := M_n^{-1}$.

Let $\bm{w}_j \in \mathbb{R}^d$ be a vector with $1$ in the first position, $-v$ in the $j$\textsuperscript{th} position for some $j \in \{2, \ldots, d\}$, and zero elsewhere:
\begin{equation}\label{eqn:defn_w}
    \bm{w}_j = [1, 0, \ldots, 0, -v, 0, \ldots, 0]^T.
\end{equation}
The denominator in \eqref{eqn:prob_in_terms_of_v} can be expressed in terms of $\bm{w}_j$: $\Sigma_{11}^{(n)} + v^2\Sigma_{jj}^{(n)} - 2v\Sigma_{1j}^{(n)} = \bm{w}_j^T \Sigma^{(n)} \bm{w}_j$.  Recall from Definition \ref{defn:evecs_evals} that the eigenvectors of $\Sigma^{(n)}$ are $\bm{v}_1^{(n)}, \ldots, \bm{v}_d^{(n)}$, with corresponding eigenvalues $\left(\lambda_1^{(n)}\right)^{-1}, \ldots, \left(\lambda_d^{(n)}\right)^{-1}$. We can write $\bm{w}_j$ in terms of the orthonormal basis formed by the eigenvectors $\{\bm{v}_k^{(n)}\}$:
\begin{equation}\label{eqn:w_in_terms_of_evecs}
    \bm{w}_j = \sum_{k = 1}^d \alpha_k^{(n)}\bm{v}_k^{(n)},
\end{equation}
for some coefficients $\alpha_k^{(n)} \in \mathbb{R}$. Using \eqref{eqn:w_in_terms_of_evecs}, the square of the denominator in \eqref{eqn:prob_in_terms_of_v} can now be written as:
\begin{flalign}\label{eqn:lower_bound_denom}
\Sigma_{11}^{(n)} + v^2\Sigma_{jj}^{(n)} - 2v\Sigma_{1j}^{(n)} &= \bm{w}_j^T \Sigma^{(n)} \bm{w}_j = \left(\sum_{k = 1}^d \alpha_k^{(n)} \bm{v}_k^{(n) T}\right)\left(\sum_{l = 1}^d \frac{1}{\lambda_l^{(n)}} \bm{v}_l^{(n)} \bm{v}_l^{(n) T}\right)\left(\sum_{m = 1}^d \alpha_m^{(n)} \bm{v}_m^{(n)}\right) \nonumber \\ &\overset{(a)}= \sum_{k = 1}^d \left(\alpha_k^{(n)}\right)^2 \frac{1}{\lambda_k^{(n)}} \overset{(b)}\ge \left(\alpha_{k_0}^{(n)}\right)^2 \frac{1}{\lambda_{k_0}^{(n)}},
\end{flalign}
where equality (a) follows by orthonormality of the eigenvector basis, and (b) holds for any $k_0 \in \{1, \ldots, d\}$ due to positivity of the eigenvalues $(\lambda_k)^{-1}$. Therefore, to show that the denominator is bounded away from zero, it suffices to show that for every $n$, there exists some $k_0$ such that $\left(\alpha_{k_0}^{(n)}\right)^2 \left(\lambda_{k_0}^{(n)}\right)^{-1}$ is bounded away from zero.

To prove the previous statement, note that by definition of $M_n$, the eigenvalues $(\lambda_k^{(n)})^{-1}$ are non-increasing in $n$. Below, we will show that for any eigenvalue $(\lambda_k^{(n)})^{-1}$ such that $\lim_{n \longrightarrow \infty} (\lambda_k^{(n)})^{-1} = 0$, the first element of its corresponding eigenvector, $\left[\bm{v}_k^{(n)} \right]_1$, also converges to zero. Since the first element of $\bm{w}_j$ equals $1$, \eqref{eqn:defn_w} implies that there must exist some $k_0$ such that $\left[\bm{v}_{k_0}^{(n)} \right]_1 \centernot\longrightarrow 0$ and $\alpha_{k_0}^{(n)}$ is bounded away from 0.  If these implications did not hold, then $\bm{w}_j$ would not have a value of 1 in its first element, contradicting its definition. These observations imply that for every $n$, there must be some $k_0$ such that as $n \longrightarrow \infty$, $(\lambda_{k_0}^{(n)})^{-1} \centernot\longrightarrow 0$ and $\alpha_{k_0}^{(n)}$ is bounded away from zero.

Let $X_n$ denote the observation matrix after $n - 1$ observations: $X_n := \begin{bmatrix}\bm{x}_1 & \ldots & \bm{x}_{n - 1}\end{bmatrix}^T$.  Then, $\Sigma^{(n)} = M_n^{-1} = (X_n^T X_n + \lambda I)^{-1}$. The matrices $M_n^{-1}$ and $X_n^T X_n$ have the same eigenvectors. Meanwhile, for each eigenvalue $(\lambda_i^{(n)})^{-1}$ of $M_n^{-1}$, $X_n^T X_n$ has an eigenvalue $\nu_i^{(n)} := \lambda_i^{(n)} - \lambda \ge 0$ corresponding to the same eigenvector.  We aim to characterize the eigenvectors of $M_n^{-1}$ whose eigenvalues approach zero.  Since these eigenvectors are identical to those of $X_n^T X_n$ whose eigenvalues approach infinity, we consider the latter instead.

We assume that all finitely-visited state-action pairs (including $\tilde{s}_1$) occur in the first $m < n - 1$ iterations. Without loss of generality, we index these finitely-visited state-action pairs from 1 to $r \ge 1$, so that the finitely-visited state-actions are: $\{\tilde{s}_1, \tilde{s}_2, \cdots, \tilde{s}_r\}$. Let $X_{1:m} \in \mathbb{R}^{m \times d}$ denote the matrix containing the first $m$ rows of $X_n$, while $X_{m+1:n} \in \mathbb{R}^{n - m \times d}$ denotes the remaining rows of $X_n$.
With this notation, 
\begin{equation*}
    X_n^T X_n = \sum_{i = 1}^{n - 1} \bm{x}_i \bm{x}_i^T = X_{1:m}^T X_{1:m} + X_{m+1:n}^T X_{m+1:n}.
\end{equation*}

Because the first $r$ state-action pairs, $\{ \tilde{s}_1, \tilde{s}_2, \cdots, \tilde{s}_r \}$, are unvisited after iteration $m$, the first $r$ elements of $\bm{x}_i$ are zero for all $i > m$. Therefore, $X_{m+1:n}^T X_{m+1:n}$ can be written in the following block matrix form:
\begin{flalign*}
    X_{m+1:n}^T X_{m+1:n} = \begin{bmatrix} O_{r \times r} & O_{r \times (d - r)} \\ O_{(d - r) \times r} & A_n \end{bmatrix},
\end{flalign*}
\noindent where $O_{a \times b}$ denotes the all-zero matrix with dimensions $a \times b$. The matrix $A_n$ includes elements that are unbounded as $n \longrightarrow \infty$. In particular, the diagonal elements of $A_n$ approach infinity as $n \longrightarrow \infty$. We can write $X_n^T X_n$ in the following block matrix form:
\begin{flalign*}
    X_n^T X_n = X_{1:m}^T X_{1:m} + X_{m+1:n}^T X_{m+1:n} = \begin{bmatrix} [X_{1:m}^T X_{1:m}]_{(1:r, 1:r)} & [X_{1:m}^T X_{1:m}]_{(1:r, r + 1:d)} \\ [X_{1:m}^T X_{1:m}]_{(r + 1:d, 1:r)} & [X_{1:m}^T X_{1:m}]_{(r + 1:d, r + 1:d)} + A_n \end{bmatrix} := \begin{bmatrix} B & C \\ C^T & D_n\end{bmatrix},
\end{flalign*}
\noindent where $M_{(a:b, c:d)}$ denotes the submatrix of $M$ obtained by extracting rows $a$ through $b$ and columns $c$ through $d$. Matrices $B$ and $C$ only depend upon $X_{1:m}$, and so are fixed as $n$ increases, while matrix $D_n$ contains values that grow towards infinity with increasing $n$. In particular, all elements along $D_n$'s diagonal are unbounded. Intuitively, in the limit, $B$ and $C$ are close to zero compared to $D_n$, and $X_n^T X_n$ (when normalized) increasingly resembles a matrix in which only the bottom-right block is nonzero. This intuitive notion is formalized next.

Consider an eigenpair $(\bm{v}_i^{(n)}, \nu_i^{(n)})$ of $X_n^T X_n$ such that $\lim_{n \longrightarrow \infty}\nu_i^{(n)} = \infty$. We show that the first element of $\bm{v}_i^{(n)}$ must approach 0. Let $\bm{v}_i^{(n)} = \begin{bmatrix}\bm{z}_i^{(n)T} & \bm{q}_i^{(n)T}\end{bmatrix}^T$, where $\bm{z}_i^{(n)} \in \mathbb{R}^m$ and $\bm{q}_i^{(n)} \in \mathbb{R}^{n - 1 - m}$. We see that:
\begin{flalign*}
    (X_n^T X_n) \bm{v}_i^{(n)} = X_n^T X_n \begin{bmatrix}\bm{z}_i^{(n)} \\ \bm{q}_i^{(n)}\end{bmatrix} = \begin{bmatrix} B & C \\ C^T & D_n\end{bmatrix}\begin{bmatrix}\bm{z}_i^{(n)} \\ \bm{q}_i^{(n)}\end{bmatrix} = \begin{bmatrix}B\bm{z}_i^{(n)} + C\bm{q}_i^{(n)} \\ C^T\bm{z}_i^{(n)} + D_n\bm{q}_i^{(n)}\end{bmatrix} = \lambda_i^{(n)} \begin{bmatrix}\bm{z}_i^{(n)} \\ \bm{q}_i^{(n)}\end{bmatrix}.
\end{flalign*}

Dividing both sides by $\nu_i^{(n)}$,
\begin{flalign*}
    \frac{1}{\nu_i^{(n)}}X_n^T X_n \begin{bmatrix}\bm{z}_i^{(n)} \\ \bm{q}_i^{(n)}\end{bmatrix} = \begin{bmatrix}\frac{1}{\nu_i^{(n)}}\left(B\bm{z}_i^{(n)} + C\bm{q}_i^{(n)}\right) \\ \frac{1}{\nu_i^{(n)}}\left(C^T\bm{z}_i^{(n)} + D_n\bm{q}_i^{(n)}\right)\end{bmatrix} = \begin{bmatrix}\bm{z}_i^{(n)} \\ \bm{q}_i^{(n)}\end{bmatrix}.
\end{flalign*}

In the upper matrix block: $\lim_{n \longrightarrow \infty} \nu_i^{(n)} = \infty$, $B$ and $C$ are fixed as $n$ increases, and $\bm{z}_i^{(n)}$ and $\bm{q}_i^{(n)}$ have upper-bounded elements because $\bm{v}_i^{(n)}$ is a unit vector. Thus, $\lim_{n\longrightarrow \infty}\bm{z}_i^{(n)} = \lim_{n\longrightarrow \infty}\frac{1}{\nu_i^{(n)}}\left(B\bm{z}_i^{(n)} + C\bm{q}_i^{(n)}\right) = \bm{0}$. In particular, the first element of $\bm{z}_i^{(n)}$ converges to zero, and so the same is true of $\bm{v}_i^{(n)}$.

As justified above, this result implies that for each iteration $n$, there exists an index $k_0 \in \{1, \ldots, d\}$ such that the right-hand side of \eqref{eqn:lower_bound_denom} has a lower bound above zero. This completes the proof that the denominator of the fraction in \eqref{eqn:prob_in_terms_of_v} does not decay to zero. As a result, there exists some $\rho > 0$ such that $P(\pi = \pi_{vi}(\bm{\tilde{p}}, \bm{r}_1)) \ge \rho > 0$.

In consequence, \algo~is guaranteed to infinitely-often sample pairs $(\bm{\tilde{p}}, \pi)$ such that $\pi = \pi_{vi}(\bm{\tilde{p}}, \bm{r}_1)$. As a result, \algo~infinitely-often samples policies that prioritize reaching $\tilde{s}_1$ as quickly as possible. Such a policy always takes action $a$ in state $s$. Furthermore, because $s$ is visited infinitely-often, either a) $p_0(s) > 0$ or b) the infinitely-visited state-action pairs include a path with a nonzero probability of reaching $s$. In case a), since the initial state distribution is fixed, the MDP will infinitely-often begin in state $s$ under the policy $\pi = \pi_{vi}(\bm{\tilde{p}}, \bm{r}_1)$, and so $\tilde{s}_1$ will be visited infinitely-often. In case b), due to Lemma \ref{lemma:dyn_concentration}, the transition dynamics parameters for state-actions along the path to $s$ converge to their true values (intuitively, the algorithm knows how to reach $s$). In episodes with the policy $\pi = \pi_{vi}(\bm{\tilde{p}}, \bm{r}_1)$, \algo~is thus guaranteed to reach $\tilde{s}_1$ infinitely-often. Since \algo~selects $\pi_{vi}(\bm{\tilde{p}}, \bm{r}_1)$ infinitely-often, it must reach $\tilde{s}_1$ infinitely-often. This presents a contradiction, and so every state-action pair must be visited infinitely-often.  
\end{proof}

Thus, by the direct combination of Lemmas \ref{lemma:dyn_concentration} and \ref{lemma:all_states_actions_visited}, we arrive at the following result.

\begin{proposition}\label{prop:dyn_consistency}
Under \algo, the sampled dynamics $\bm{\tilde{p}}_{i1}, \bm{\tilde{p}}_{i2}$ converge in distribution to the true dynamics: $\bm{\tilde{p}}_{i1}, \bm{\tilde{p}}_{i2} \overset{D}{\longrightarrow} \bm{\overline{p}}$, where $\overset{D}\longrightarrow$ denotes convergence in distribution.
\end{proposition}

Next, we show that the sampled rewards converge in distribution to their true values. Our analysis will use Theorem 2 from \cite{abbasi2011improved}, which is repeated below. Recall that Eq.s \eqref{eqn:reward_MAP}-\eqref{eqn:M} define the MAP reward estimate.

\begin{lemma}[Theorem 2 from \cite{abbasi2011improved}]\label{lemma:abbasi_yadkori}
Let $\{F_i\}_{i = 0}^\infty$ be a filtration. Let $\{\eta_i\}_{i = 1}^\infty$ be a real-valued stochastic process such that $\eta_i$ is $F_i$-measurable and $\eta_i$ is conditionally $R$-sub-Gaussian for some $R \ge 0$. Let $\{\bm{x}_i\}$ be an $\mathbb{R}^d$-valued stochastic process such that $\bm{x}_i$ is $F_{i - 1}$-measurable. Define $y_i := \bm{x}_i^T \bm{\overline{r}} + \eta_i$, and assume that $||\bm{\overline{r}}||_2 \le S_r$ and $||\bm{x}_i||_2 \le L$. Then, for any $\delta > 0$, with probability at least $1 - \delta$, for all $i > 0$, $||\bm{\hat{r}}_i - \bm{\overline{r}}||_{M_i} \le \beta_i(\delta)$, where:
\begin{equation*}
    \beta_i(\delta) = R \sqrt{2 \log\left(\frac{\det(M_i)^{1/2}\lambda^{-d/2}}{\delta}\right)} + \sqrt{\lambda}S_r \le R \sqrt{d \log\left(\frac{1 + \frac{L^2i}{d\lambda}}{\delta}\right)} + \sqrt{\lambda}S_r.
\end{equation*}
\end{lemma}

Note that in the present case, $L \le 2h$, since:
\begin{equation*}
    ||\bm{x}_i||_2 = ||\bm{x}_{i2} - \bm{x}_{i1}||_2 \overset{(a)}{\le} ||\bm{x}_{i2} - \bm{x}_{i1}||_1 \le ||\bm{x}_{i2}||_1 + ||\bm{x}_{i1}||_1 \overset{(b)}= 2h,
\end{equation*}
where (a) holds because $||\bm{x}||_2 \le ||\bm{x}||_1$ for any $\bm{x} \in \mathbb{R}^d$, and (b) holds because $\bm{x}_{i1}$ and $\bm{x}_{i2}$ each count an episode's visits to every state-action pair, and so their elements are non-negative integers summing to $h$.

\begin{proposition}\label{prop:reward_consistency}
With probability $1 - \delta$, where $\delta$ is a parameter of the Bayesian linear regression model, the sampled rewards $\bm{\tilde{r}}_{i1}, \bm{\tilde{r}}_{i2}$ converge in distribution to the true reward parameters, $\bm{\tilde{r}}_{i1}, \bm{\tilde{r}}_{i2} \overset{D}\longrightarrow \bm{\overline{r}}$, as $i \longrightarrow \infty$.
\end{proposition}
\begin{proof}
This is a direct implication of Lemmas \ref{lemma:reward_consistency_part1} and \ref{lemma:reward_consistency_part4}, both proven below.
\end{proof}

\begin{lemma}\label{lemma:reward_consistency_part1}
If $\frac{\beta_i(\delta)^2}{\lambda_d^{(i)}} \overset{D}\longrightarrow 0$ as $i \longrightarrow \infty$, where $\lambda_d^{(i)}$ is the minimum eigenvalue of $M_i$ and $\overset{D}\longrightarrow$ denotes convergence in distribution, then $\bm{\tilde{r}}_{i1}, \bm{\tilde{r}}_{i2} \overset{D}\longrightarrow \bm{\overline{r}}$ with probability $1 - \delta$.
\end{lemma}

\begin{proof}
From Lemma \ref{lemma:abbasi_yadkori}, with probability at least $1 - \delta$, $\bm{\hat{r}}_i$ belongs to a confidence ellipsoid centered at $\bm{\overline{r}}$: $||\bm{\hat{r}}_i - \bm{\overline{r}}||_{M_i} \le \beta_i(\delta)$. We show that under this high-probability event, $\bm{\tilde{r}}_{i1}, \bm{\tilde{r}}_{i2} \overset{D}\longrightarrow \bm{\overline{r}}$. Similarly to the high-probability confidence ellipsoid from Lemma \ref{lemma:abbasi_yadkori}, the Thompson sampling covariance matrix is also defined by $\beta_i(\delta)$ and $M_i$:
\begin{equation}\label{eqn:TS_dist}
    \bm{\tilde{r}}_{i1}, \bm{\tilde{r}}_{i2} \sim \mathcal{N}(\bm{\hat{r}}_i, \beta_i(\delta)^2M_i^{-1}).
\end{equation}

Letting $\bm{z}_i \sim \mathcal{N}(\bm{0}, I)$ be independent for each $i$, we can equivalently express $\bm{\tilde{r}}_{i1}$ (and similarly, $\bm{\tilde{r}}_{i2}$) as:
\begin{equation}\label{eqn:reward_with_z}
    \bm{\tilde{r}}_{i1} = \bm{\hat{r}}_i + \beta_i(\delta)M_i^{-\frac{1}{2}}\bm{z}_i,
\end{equation}
since the random variable in \eqref{eqn:reward_with_z} has the same distribution  as \eqref{eqn:TS_dist}. The quantity $||\bm{\tilde{r}}_{i1} - \bm{\hat{r}}_i||_{M_i}$ can be rewritten as:
\begin{flalign*}
||\bm{\tilde{r}}_{i1} - \bm{\hat{r}}_i||_{M_i} = \left|\left|\beta_i(\delta)M_i^{-\frac{1}{2}}\bm{z}_i\right|\right|_{M_i} = \beta_i(\delta) \sqrt{\bm{z}_i^T M_i^{-\frac{1}{2}}M_i M_i^{-\frac{1}{2} }\bm{z}_i} = \beta_i(\delta)||\bm{z}_i||_2.
\end{flalign*}

Because the probability distribution of $||\bm{z}_i||_2$ is fixed, there exists some fixed $a > 0$ such that with probability at least $1 - \delta$, $||\bm{z}_i||_2 \le a$. So, for each $i$, with probability at least $1 - \delta$,
\begin{equation*}
    ||\bm{\tilde{r}}_{i1} - \bm{\hat{r}}_i||_{M_i} = \beta_i(\delta)||\bm{z}_i||_2 \le \beta_i(\delta)a.
\end{equation*}

Assuming that the high-probability event in Lemma \ref{lemma:abbasi_yadkori} occurs, we combine the previous inequality with $||\bm{\hat{r}}_i - \bm{\overline{r}}||_{M_i} \le \beta_i(\delta)$ to obtain that for each $i$, with probability at least $1 - \delta$,
\begin{equation*}
    ||\bm{\tilde{r}}_{i1} - \bm{\overline{r}}||_{M_i} \le ||\bm{\tilde{r}}_{i1} - \bm{\hat{r}}_i||_{M_i} + ||\bm{\hat{r}}_i - \bm{\overline{r}}||_{M_i} \le (a + 1)\beta_i(\delta).
\end{equation*}

Squaring both sides and dividing by $\beta_i(\delta)$ yields that for each $i$, with probability at least $1 - \delta$,
\begin{equation*}
    \frac{1}{\beta_i(\delta)^2}(\bm{\tilde{r}}_{i1} - \bm{\overline{r}})^T M_i (\bm{\tilde{r}}_{i1} - \bm{\overline{r}}) \le (a + 1)^2.
\end{equation*}

By assumption, $\frac{\lambda_d^{(i)}}{\beta_i(\delta)^2} \overset{D}\longrightarrow \infty$ as $i \longrightarrow \infty$. Recall from Definition \ref{defn:evecs_evals} that $\bm{v}_j^{(i)}, j \in \{1, \ldots, d\}$, represent the eigenvectors of $M_i$ corresponding to the eigenvalues $\lambda_j^{(i)}$. Then, with probability at least $1 - \delta$ for each $i$:
\begin{flalign}
    \frac{1}{\beta_i(\delta)^2}(\bm{\tilde{r}}_{i1} - \bm{\overline{r}})^T M_i (\bm{\tilde{r}}_{i1} - \bm{\overline{r}}) &= \frac{1}{\beta_i(\delta)^2}(\bm{\tilde{r}}_{i1} - \bm{\overline{r}})^T \left(\sum_{j = 1}^d \lambda_j^{(i)} \bm{v}_j^{(i)}\bm{v}_j^{(i)T} \right) (\bm{\tilde{r}}_{i1} - \bm{\overline{r}}) \nonumber \\ &= \frac{1}{\beta_i(\delta)^2}\sum_{j = 1}^d \lambda_j^{(i)} \left( (\bm{\tilde{r}}_{i1} - \bm{\overline{r}})^T \bm{v}_j^{(i)}\right)^2 \le (a + 1)^2. \label{eqn:reward_high_prob_bound}
\end{flalign}

Since for each $j$, $\frac{\lambda_j^{(i)}}{\beta_i(\delta)^2} \longrightarrow \infty$ as $i \longrightarrow \infty$, and $\bm{v}_j^{(i)}$ is an orthonormal basis, the constant bound of $(a + 1)^2$ in \eqref{eqn:reward_high_prob_bound} is violated if we do not have $\bm{\tilde{r}}_{i1} - \bm{\overline{r}} \overset{D}\longrightarrow \bm{0}$. Equation \eqref{eqn:reward_high_prob_bound} must hold with probability at least $1 - \delta$ independently for each iteration $i$, with the $(1 - \delta)$-probability due entirely to randomness in the Thompson sampling distribution, \eqref{eqn:TS_dist}. Therefore, it follows that $\bm{\tilde{r}}_{i1} \overset{D}\longrightarrow \bm{\overline{r}}$. The proof that $\bm{\tilde{r}}_{i2} \overset{D}\longrightarrow \bm{\overline{r}}$ with high probability is identical.

\end{proof}

The next result enables us to leverage convergence in distribution of the dynamics samples, $\bm{\tilde{p}}_{i1}, \bm{\tilde{p}}_{i2} \overset{D}\longrightarrow \bm{\overline{p}}$ (as guaranteed by Proposition \ref{prop:dyn_consistency}), in characterizing the impact of sampled policies upon convergence of the reward model.

\begin{lemma}\label{lemma:f_cont}
Let $f: \mathbb{R}^{S^2 A} \times \mathbb{R}^{SA} \longrightarrow \mathbb{R}$ be a function of transition dynamics $\bm{p} \in \mathbb{R}^{S^2 A}$ and reward vector $\bm{r} \in \mathbb{R}^{SA}$, $f(\bm{p}, \bm{r})$, where $f$ is continuous in $\bm{p}$ and uniformly-continuous in $\bm{r}$. Assume that Proposition \ref{prop:dyn_consistency} holds: $\bm{\tilde{p}}_{i1}, \bm{\tilde{p}}_{i2} \overset{D}\longrightarrow \bm{\overline{p}}$. Then, for any $\delta, \varepsilon > 0$, there exists $i^\prime$ such that for $i > i^\prime$, $|f(\bm{\overline{p}}, \bm{r}) - f(\bm{\tilde{p}}_{ij}, \bm{r})| < \varepsilon\,$ for any \emph{unit vector} $\bm{r}$ and $j \in \{1, 2\}$ with probability at least $1 - \delta$.
\end{lemma}
\begin{proof}
The proof is identical for $j = 1$ and $j = 2$, and so without loss of generality, we set $j = 1$. Applying Proposition \ref{prop:dyn_consistency}, $\bm{\tilde{p}}_{i1} \overset{D}\longrightarrow \bm{\overline{p}}$. By continuity of $f$, we can apply Fact \ref{fact:conv_dist_1} from Appendix \ref{sec:appendix_prob_facts} to obtain that $f(\bm{\tilde{p}}_{i1}, \bm{r}) \overset{D}\longrightarrow f(\bm{\overline{p}}, \bm{r})$ for any $\bm{r}$. Further applying Fact \ref{fact:conv_dist_2} from Appendix \ref{sec:appendix_prob_facts}, $f(\bm{\tilde{p}}_{i1}, \bm{r}) \overset{P}\longrightarrow f(\bm{\overline{p}}, \bm{r})$ for any $\bm{r}$. By definition of convergence in probability, given $\delta$, there exists $i_{\bm{r}}$ such that for $i \ge i_{\bm{r}}$:
\begin{equation}\label{eqn:cont_f_1}
    |f(\bm{\tilde{p}}_{i1}, \bm{r}) - f(\bm{\overline{p}}, \bm{r})| < \frac{1}{3}\varepsilon \text{ with probability at least } 1 - \delta.
\end{equation}

To obtain a high-probability bound that applies over all unit vectors $\bm{r}$, we use compactness of the set of unit vectors. Any infinite cover of a compact set has a finite subcover; in particular, for any $\delta^{\prime} > 0$, the set of unit vectors in $\mathbb{R}^d$ has a finite cover of the form $\{\mathcal{B}(\bm{r}_1, \delta^\prime), \ldots, \mathcal{B}(\bm{r}_K, \delta^\prime)\}$, where $\{\bm{r}_1, \ldots, \bm{r}_K\}$ are unit vectors, and $\mathcal{B}(\bm{r}, \delta^\prime) := \{\bm{r}^\prime \in \mathbb{R}^d \,|\, ||\bm{r}^\prime - \bm{r}||_2 < \delta^\prime\}$ is the $d$-dimensional sphere of radius $\delta^\prime$ centered at $\bm{r}$. Thus, there exists a finite set of unit vectors $\mathcal{U} = \{\bm{r}_1, \ldots, \bm{r}_K\}$ such that for any unit vector $\bm{r}^\prime$, $||\bm{r}_i - \bm{r}^\prime||_2 < \delta^{\prime}$ for some $i \in \{1, \ldots, K\}$. Because $f$ is uniformly-continuous in $\bm{r}$, for any transition dynamics $\bm{p}$, there exists $\delta_{\bm{p}} > 0$ such that for any two unit vectors $\bm{r}, \bm{r}^\prime$ such that $||\bm{r} - \bm{r}^\prime||_2 < \delta_{\bm{p}}$:
\begin{equation}\label{eqn:cont_f_2}
    |f(\bm{p}, \bm{r}) - f(\bm{p}, \bm{r}^\prime)| < \frac{1}{3}\varepsilon.
\end{equation}

Without loss of generality, for each $\bm{p}$, define $\delta_{\bm{p}} := \sup x$ such that $||\bm{r} - \bm{r}^\prime||_2 < x$ implies $|f(\bm{p}, \bm{r}) - f(\bm{p}, \bm{r}^\prime)| \le \frac{1}{6}\varepsilon$. Then, because $f$ is continuous in $\bm{p}$, $\delta_{\bm{p}}$ is also continuous in $\bm{p}$. Because the set of all possible transition probability vectors $\bm{p}$ is compact, and a continuous function over a compact set achieves its minimum value, there exists $\delta_{\text{min}} > 0$ such that $\delta_{\bm{p}} \ge \delta_{\text{min}} > 0$ over all $\bm{p}$. We can define $\mathcal{U}$ such that $\delta^{\prime} \le \delta_{\text{min}}$; then, for any unit vector $\bm{r}^\prime$, there exists $\bm{r} \in \mathcal{U}$ such that $||\bm{r} - \bm{r}^\prime||_2 < \delta_{\text{min}}$, and thus \eqref{eqn:cont_f_2} holds for any $\bm{p}$.

By \eqref{eqn:cont_f_1}, for each $\bm{r}_j \in \mathcal{U}$, there exists there exists $i_{\bm{r}_j}$ such that for $i \ge i_{\bm{r}_j}$: $||f(\bm{\tilde{p}}_{i1}, \bm{r}) - f(\bm{\overline{p}}, \bm{r})||_2 < \frac{1}{3}\varepsilon$ with probability at least $1 - \delta$. Because $\mathcal{U}$ is a finite set, there exists $i^\prime > \max\{i_{\bm{r}_1}, \ldots, i_{\bm{r}_K}\}$ such that for $\bm{r} \in \mathcal{U}$ and $i > i^\prime$:
\begin{equation}\label{eqn:cont_f_3}
    |f(\bm{\tilde{p}}_{i1}, \bm{r}) - f(\bm{\overline{p}}, \bm{r})| < \frac{1}{3}\varepsilon \text{ for each } \bm{r} \in \mathcal{U} \text{ with probability at least } 1 - \delta.
\end{equation}

Therefore, for any unit vector $\bm{r}^\prime$, there exists $\bm{r} \in \mathcal{U}$ such that $||\bm{r} - \bm{r}^\prime||_2 < \delta^\prime \le \delta_{\text{min}}$, and with probability at least $1 - \delta$ for $i > i^\prime$:
\begin{flalign}
|f(\bm{\overline{p}}, \bm{r}^\prime) - f(\bm{\tilde{p}}_{i1}, \bm{r}^\prime)| &= |f(\bm{\overline{p}}, \bm{r}^\prime) - f(\bm{\overline{p}}, \bm{r}) + f(\bm{\overline{p}}, \bm{r}) - f(\bm{\tilde{p}}_{i1}, \bm{r}) + f(\bm{\tilde{p}}_{i1}, \bm{r}) - f(\bm{\tilde{p}}_{i1}, \bm{r}^\prime)| \\ &\overset{(a)}\le |f(\bm{\overline{p}}, \bm{r}^\prime) - f(\bm{\overline{p}}, \bm{r})| + |f(\bm{\overline{p}}, \bm{r}) - f(\bm{\tilde{p}}_{i1}, \bm{r})| + |f(\bm{\tilde{p}}_{i1}, \bm{r}) - f(\bm{\tilde{p}}_{i1}, \bm{r}^\prime)| \\ &\overset{(b)}\le \frac{1}{3}\varepsilon + \frac{1}{3}\varepsilon + \frac{1}{3}\varepsilon = \varepsilon,
\end{flalign}
where (a) holds due to the triangle inequality, and (b) holds via \eqref{eqn:cont_f_2} and \eqref{eqn:cont_f_3}, where we showed that there exists $\delta_{\text{min}}$ such that $0 < \delta_{\text{min}} \le \delta_{\bm{p}}$ for all possible transition dynamics parameters $\bm{p}$.
\end{proof}

Applying Lemma \ref{lemma:f_cont} to the two functions $V(\bm{p}, \bm{r}, \pi_{vi}(\bm{p}, \bm{r})) = \max_{\pi^\prime}V(\bm{p}, \bm{r}, \pi^\prime)$ and $V(\bm{p}, \bm{r}, \pi)$, for any fixed policy $\pi$, yields the following result.

\begin{lemma}\label{lemma:V_conv_in_p}
For any $\varepsilon, \delta > 0$, any policy $\pi$, and any unit reward vector $\bm{r}$, both of the following hold with probability at least $1 - \delta$ for sufficiently-large $i$ and $j \in \{1, 2\}$:
\begin{flalign*}
&|V(\bm{\overline{p}}, \bm{r}, \pi) - V(\bm{\tilde{p}}_{ij}, \bm{r}, \pi)| < \varepsilon \\
&|V(\bm{\overline{p}}, \bm{r}, \pi_{vi}(\bm{\overline{p}}, \bm{r})) - V(\bm{\tilde{p}}_{ij}, \bm{r}, \pi_{vi}(\bm{\tilde{p}}_{ij}, \bm{r}))| < \varepsilon.
\end{flalign*}
\end{lemma}
\begin{proof}
Both statements follow by applying Lemma \ref{lemma:f_cont}. First, consider the function $f_1(\bm{p}, \bm{r}) := V(\bm{p}, \bm{r}, \pi)$ for a fixed policy $\pi$. The value function $V(\bm{p}, \bm{r}, \pi)$ is continuous in both $\bm{p}$ and $\bm{r}$. Furthermore, it is linear in $\bm{r}$ and therefore uniformly-continuous in $\bm{r}$: for a linear function $g(\bm{z}) = \bm{a}^T\bm{z}$ and for any $\varepsilon^\prime > 0$, if $\delta^\prime := \frac{\varepsilon^\prime}{||\bm{a}||}$, then for any $\bm{z}_1, \bm{z}_2$ such that $||\bm{z}_1 - \bm{z}_2|| < \delta^\prime$:
\begin{equation*}
    |g(\bm{z}_1) - g(\bm{z}_2)| = |\bm{a}^T(\bm{z}_1 - \bm{z}_2)| \le ||\bm{a}||_2 ||\bm{z}_1 - \bm{z}_2||_2 < ||\bm{a}||_2\delta^\prime = \varepsilon^\prime.
\end{equation*}

Thus, $f_1$ satisfies the conditions of Lemma \ref{lemma:f_cont} for any fixed $\pi$, and so for $i > i_{\pi}$, $|V(\bm{\overline{p}}, \bm{r}, \pi) - V(\bm{\tilde{p}}_{ij}, \bm{r}, \pi)| < \varepsilon$ with probability at least $1 - \delta$. Because there are finitely-many deterministic policies $\pi$, we can set $i > \max_{\pi} i_{\pi}$, so that the statement holds jointly over all $\pi$.

Next, let $f_2(\bm{p}, \bm{r}) = \max_{\pi}V(\bm{p}, \bm{r}, \pi) = V(\bm{p}, \bm{r}, \pi_{vi}(\bm{p}, \bm{r}))$. A maximum over finitely-many continuous functions is continuous, and a maximum over finitely-many uniformly-continuous functions is uniformly-continuous. Therefore, $f_2$ also satisfies the conditions of Lemma \ref{lemma:f_cont}.
\end{proof}

We will show convergence in distribution of the reward samples, $\bm{\tilde{r}}_{i1}, \bm{\tilde{r}}_{i2} \overset{D}\longrightarrow \bm{\overline{r}}$, by applying Lemma \ref{lemma:reward_consistency_part1} and demonstrating that $\frac{1}{\beta_i(\delta)^2} \lambda_d^{(i)} \longrightarrow \infty$ as $i \longrightarrow \infty$. This result is proven by contradiction: intuitively, if $\frac{1}{\beta_i(\delta)^2} \lambda_d^{(i)}$ is upper-bounded, then \algo~has a lower-bounded probability of selecting policies that increase $\lambda_d^{(i)}$. Importantly, abbreviating $\lambda_d^{(i)}
$'s eigenvector $\bm{v}_d^{(i)}$ as $\bm{v}$, this proof is contingent upon there existing a pair of policies $\pi_1, \pi_2$ such that:
\begin{equation*}
    \left|\mathbb{E}[\bm{x}_i^T \bm{v} \,|\, \pi_{i1} = \pi_1, \pi_{i2} = \pi_2]\right| = \left|\mathbb{E}[(\bm{x}_{i2} - \bm{x}_{i1})^T \bm{v} \,|\, \pi_{i1} = \pi_1, \pi_{i2} = \pi_2]\right| \overset{(a)}= \left|V(\bm{\overline{p}}, \bm{v}, \pi_1) - V(\bm{\overline{p}}, \bm{v}, \pi_2)\right| > 0,
\end{equation*}
where (a) holds because the value function $V(\bm{\overline{p}}, \bm{v}, \pi)$ gives the expected total reward of $\pi$ under the reward vector $\bm{v}$. In other words, the proof will require,
\begin{equation}\label{eqn:relevance_condition}
    \max_{\pi_1, \pi_2}\left|\mathbb{E}[\bm{x}_i^T \bm{v} \,|\, \pi_{i1} = \pi_1, \pi_{i2} = \pi_2]\right| = \max_{\pi_1, \pi_2}\left|V(\bm{\overline{p}}, \bm{v}, \pi_1) - V(\bm{\overline{p}}, \bm{v}, \pi_2)\right| > 0.
\end{equation}
If this does not hold, then it is impossible to select a pair of policies under which the observation $\bm{x}_i$ is not expected to be orthogonal to the eigenvector $\bm{v}$. 

We argue that without loss of generality, \eqref{eqn:relevance_condition} can be assumed to hold for all eigenvectors of $M_i$. Note that if \eqref{eqn:relevance_condition} does not hold, then $\mathbb{E}[\bm{x}_{i1}^T\bm{v} \,|\, \pi_{i1} = \pi] = V(\bm{\overline{p}}, \bm{v}, \pi)$ is fixed for all $\pi$. Given $\bm{\overline{p}}$, by linearity of the value function $V$ in the rewards, any $\bm{v}$-directed component of $\bm{r}$ does not affect policy selection: $V(\bm{\overline{p}}, \bm{r}, \pi) = V(\bm{\overline{p}}, \bm{r}^{\bm{v}} + \bm{r}^{\bm{v}\perp}, \pi) = V(\bm{\overline{p}}, \bm{r}^{\bm{v}}, \pi) + V(\bm{\overline{p}}, \bm{r}^{\bm{v}\perp}, \pi)$, where $\bm{r}^{\bm{v}}$ is the projection of $\bm{r}$ onto the $\bm{v}$-direction and $\bm{r}^{\bm{v}\perp}$ is its orthogonal complement in $\mathbb{R}^d$. Because $V(\bm{\overline{p}}, \bm{r}^{\bm{v}}, \pi)$ does not depend on $\pi$, $\pi_{vi}(\bm{\overline{p}}, \bm{r}) = \text{argmax}_{\pi} V(\bm{\overline{p}}, \bm{r}, \pi) = \text{argmax}_{\pi} V(\bm{\overline{p}}, \bm{r}^{\bm{v}\perp}, \pi)$.

We call any vector $\bm{v}$ which does \textit{not} satisfy \eqref{eqn:relevance_condition} an \textit{irrelevant dimension} of the rewards: given $\bm{\overline{p}}$, removing the $\bm{v}$-directed component of $\bm{r}$ does not influence policy selection. The following lemma demonstrates that such vectors remain irrelevant towards policy selection when $\bm{\overline{p}}$ is unknown, but once $\bm{\tilde{p}}_{i1}, \bm{\tilde{p}}_{i2}$ have sufficiently converged to $\bm{\overline{p}}$ in distribution. 

\begin{lemma}\label{lemma:relevance}
For any reward vector $\bm{r} \in \mathbb{R}^d$, let $\bm{r}^{\mathrm{rel}}$ be the projection of $\bm{r}$ onto the relevant subspace (for which \eqref{eqn:relevance_condition} holds), and $\bm{r}^\perp$ be its orthogonal complement in $\mathbb{R}^d$, such that $\bm{r} = \bm{r}^{\mathrm{rel}} + \bm{r}^\perp$, and $\bm{r}^\perp$ belongs to the subspace of irrelevant dimensions (where \eqref{eqn:relevance_condition} does \emph{not} hold). The reward samples on iteration $i$ are $\bm{\tilde{r}}_{ij}$, $j \in \{1, 2\}$. Then, for any $\varepsilon, \delta > 0$, there exists $i_0$ such that for $i > i_0$, with probability at least $1 - \delta$:
\begin{equation*}
    |V(\bm{\overline{p}}, \bm{\tilde{r}}_{ij}, \pi_{ij}) - V(\bm{\overline{p}}, \bm{\tilde{r}}_{ij}, \pi_{vi}(\bm{\tilde{p}}_{ij}, \bm{\tilde{r}}_{ij}^{\mathrm{rel}}))|  < \varepsilon.
\end{equation*}
In other words, with respect to the sampled rewards $\bm{\tilde{r}}_{ij}$, the expected reward of the selected policy $\pi_{ij} = \pi_{vi}(\bm{\tilde{p}}_{ij}, \bm{\tilde{r}}_{ij})$ is close to the expected reward of the policy that \textit{would have been} selected were $\bm{\tilde{r}}_{ij}$ replaced by $\bm{\tilde{r}}_{ij}^{\mathrm{rel}}$.
\end{lemma}
\begin{proof}
We prove the result for $j = 1$ (the proof is identical for $j = 2$). Because $V(\bm{\overline{p}}, \bm{\tilde{r}}_{i1}^{\perp}, \pi)$ is constant for all $\pi$, we define $w := V(\bm{\overline{p}}, \bm{\tilde{r}}_{i1}^{\perp}, \pi)$ for convenience. First, we show that under the true transition dynamics $\bm{\overline{p}}$, the irrelevant dimensions of $\bm{\tilde{r}}_{i1}$ do not affect policy selection. For any $\pi$,
\begin{flalign}
    V(\bm{\overline{p}}, \bm{\tilde{r}}_{i1}, \pi_{vi}(\bm{\overline{p}}, \bm{\tilde{r}}_{i1})) &= \max_\pi V(\bm{\overline{p}}, \bm{\tilde{r}}_{i1}, \pi) \overset{(a)}= \max_\pi \left[V(\bm{\overline{p}}, \bm{\tilde{r}}_{i1}^{\text{rel}}, \pi) + V(\bm{\overline{p}}, \bm{\tilde{r}}_{i1}^{\perp}, \pi)\right] \nonumber \\ &\overset{(b)}= V(\bm{\overline{p}}, \bm{\tilde{r}}_{i1}^{\text{rel}}, \pi_{vi}(\bm{\overline{p}}, \bm{\tilde{r}}_{i1}^{\text{rel}})) + V(\bm{\overline{p}}, \bm{\tilde{r}}_{i1}^{\perp}, \pi_{vi}(\bm{\overline{p}}, \bm{\tilde{r}}_{i1}^{\text{rel}})) \overset{(c)}= V(\bm{\overline{p}}, \bm{\tilde{r}}_{i1}, \pi_{vi}(\bm{\overline{p}}, \label{eqn:relevance} \bm{\tilde{r}}_{i1}^{\text{rel}})),
\end{flalign}
where (a) and (c) hold because $\bm{\tilde{r}}_{i1} = \bm{\tilde{r}}_{i1}^{\text{rel}} + \bm{\tilde{r}}_{i1}^{\perp}$ and the value function is linear in the rewards, and (b) holds because $V(\bm{\overline{p}}, \bm{\tilde{r}}_{i1}^{\perp}, \pi) = w$ is constant across all policies $\pi$.

To upper-bound $|V(\bm{\overline{p}}, \bm{\tilde{r}}_{i1}, \pi_{i1}) - V(\bm{\overline{p}}, \bm{\tilde{r}}_{i1}, \pi_{vi}(\bm{\tilde{p}}_{i1}, \bm{\tilde{r}}_{i1}^{\mathrm{rel}}))|$, we write:
\begin{flalign}
    |V(\bm{\overline{p}}, &\bm{\tilde{r}}_{i1}, \pi_{i1}) - V(\bm{\overline{p}}, \bm{\tilde{r}}_{i1}, \pi_{vi}(\bm{\tilde{p}}_{i1}, \bm{\tilde{r}}_{i1}^{\mathrm{rel}}))| \overset{(a)}= |V(\bm{\overline{p}}, \bm{\tilde{r}}_{i1}, \pi_{vi}(\bm{\tilde{p}}_{i1}, \bm{\tilde{r}}_{i1})) - (V(\bm{\overline{p}}, \bm{\tilde{r}}_{i1}^{\mathrm{rel}}, \pi_{vi}(\bm{\tilde{p}}_{i1}, \bm{\tilde{r}}_{i1}^{\mathrm{rel}})) + w)| \nonumber \\ &= |V(\bm{\overline{p}}, \bm{\tilde{r}}_{i1}, \pi_{vi}(\bm{\tilde{p}}_{i1}, \bm{\tilde{r}}_{i1})) - (V(\bm{\overline{p}}, \bm{\tilde{r}}_{i1}^{\mathrm{rel}}, \pi_{vi}(\bm{\tilde{p}}_{i1}, \bm{\tilde{r}}_{i1}^{\mathrm{rel}})) + w) \nonumber \\ &\hspace{5mm}- V(\bm{\tilde{p}}_{i1}, \bm{\tilde{r}}_{i1}, \pi_{vi}(\bm{\tilde{p}}_{i1}, \bm{\tilde{r}}_{i1})) + V(\bm{\tilde{p}}_{i1}, \bm{\tilde{r}}_{i1}, \pi_{vi}(\bm{\tilde{p}}_{i1}, \bm{\tilde{r}}_{i1}))| - V(\bm{\overline{p}}, \bm{\tilde{r}}_{i1}, \pi_{vi}(\bm{\overline{p}}, \bm{\tilde{r}}_{i1})) + V(\bm{\overline{p}}, \bm{\tilde{r}}_{i1}, \pi_{vi}(\bm{\overline{p}}, \bm{\tilde{r}}_{i1}))| \nonumber \\ &\hspace{5mm}- (V(\bm{\tilde{p}}_{i1}, \bm{\tilde{r}}_{i1}^{\text{rel}}, \pi_{vi}(\bm{\tilde{p}}_{i1}, \bm{\tilde{r}}_{i1}^{\text{rel}})) + w) + (V(\bm{\tilde{p}}_{i1}, \bm{\tilde{r}}_{i1}^{\text{rel}}, \pi_{vi}(\bm{\tilde{p}}_{i1}, \bm{\tilde{r}}_{i1}^{\text{rel}})) + w)| \nonumber \\ &\overset{(b)}\le |V(\bm{\overline{p}}, \bm{\tilde{r}}_{i1}, \pi_{vi}(\bm{\tilde{p}}_{i1}, \bm{\tilde{r}}_{i1})) - V(\bm{\tilde{p}}_{i1}, \bm{\tilde{r}}_{i1}, \pi_{vi}(\bm{\tilde{p}}_{i1}, \bm{\tilde{r}}_{i1})| + |V(\bm{\tilde{p}}_{i1}, \bm{\tilde{r}}_{i1}, \pi_{vi}(\bm{\tilde{p}}_{i1}, \bm{\tilde{r}}_{i1})) - V(\bm{\overline{p}}, \bm{\tilde{r}}_{i1}, \pi_{vi}(\bm{\overline{p}}, \bm{\tilde{r}}_{i1})|  \nonumber \\ &\hspace{5mm}+ |V(\bm{\overline{p}}, \bm{\tilde{r}}_{i1}, \pi_{vi}(\bm{\overline{p}}, \bm{\tilde{r}}_{i1}) - (V(\bm{\tilde{p}}_{i1}, \bm{\tilde{r}}_{i1}^{\text{rel}}, \pi_{vi}(\bm{\tilde{p}}_{i1}, \bm{\tilde{r}}_{i1}^{\text{rel}})) + w)|\label{eqn:line_relevance} \\ &\hspace{5mm}+ |(V(\bm{\tilde{p}}_{i1}, \bm{\tilde{r}}_{i1}^{\text{rel}}, \pi_{vi}(\bm{\tilde{p}}_{i1}, \bm{\tilde{r}}_{i1}^{\text{rel}})) + w) - (V(\bm{\overline{p}}, \bm{\tilde{r}}_{i1}^{\mathrm{rel}}, \pi_{vi}(\bm{\tilde{p}}_{i1}, \bm{\tilde{r}}_{i1}^{\mathrm{rel}})) + w)| \nonumber \\ &\overset{(c)}\le |V(\bm{\overline{p}}, \bm{\tilde{r}}_{i1}, \pi_{vi}(\bm{\tilde{p}}_{i1}, \bm{\tilde{r}}_{i1})) - V(\bm{\tilde{p}}_{i1}, \bm{\tilde{r}}_{i1}, \pi_{vi}(\bm{\tilde{p}}_{i1}, \bm{\tilde{r}}_{i1})|  + |V(\bm{\tilde{p}}_{i1}, \bm{\tilde{r}}_{i1}, \pi_{vi}(\bm{\tilde{p}}_{i1}, \bm{\tilde{r}}_{i1})) - V(\bm{\overline{p}}, \bm{\tilde{r}}_{i1}, \pi_{vi}(\bm{\overline{p}}, \bm{\tilde{r}}_{i1})|  \nonumber \\ &\hspace{5mm}+ |V(\bm{\overline{p}}, \bm{\tilde{r}}_{i1}^{\text{rel}}, \pi_{vi}(\bm{\overline{p}}, \bm{\tilde{r}}_{i1}^{\text{rel}})) - V(\bm{\tilde{p}}_{i1}, \bm{\tilde{r}}_{i1}^{\text{rel}}, \pi_{vi}(\bm{\tilde{p}}_{i1}, \bm{\tilde{r}}_{i1}^{\text{rel}}))| \nonumber \\ &\hspace{5mm}+ |V(\bm{\tilde{p}}_{i1}, \bm{\tilde{r}}_{i1}^{\text{rel}}, \pi_{vi}(\bm{\tilde{p}}_{i1}, \bm{\tilde{r}}_{i1}^{\text{rel}})) - V(\bm{\overline{p}}, \bm{\tilde{r}}_{i1}^{\mathrm{rel}}, \pi_{vi}(\bm{\tilde{p}}_{i1}, \bm{\tilde{r}}_{i1}^{\mathrm{rel}}))|, \label{eqn:line_relevance_2}
\end{flalign}
where (a) applies $\bm{\tilde{r}}_{i1} = \bm{\tilde{r}}_{i1}^{\text{rel}} + \bm{\tilde{r}}_{i1}^{\perp}$, linearity of the value function in the rewards, and the definition of $w$; (b) rearranges terms and uses the triangle inequality; and (c) applies \eqref{eqn:relevance} to line \eqref{eqn:line_relevance}, that is,  $V(\bm{\overline{p}}, \bm{\tilde{r}}_{i1}, \pi_{vi}(\bm{\overline{p}}, \bm{\tilde{r}}_{i1})) = V(\bm{\overline{p}}, \bm{\tilde{r}}_{i1}^{\text{rel}}, \pi_{vi}(\bm{\overline{p}}, \bm{\tilde{r}}_{i1}^{\text{rel}})) + w$.

Each of the four terms in \eqref{eqn:line_relevance_2} can be upper-bounded with high probability using Lemma \ref{lemma:V_conv_in_p}. In particular, for large enough $i$, each term is less than $\frac{1}{4}\varepsilon$ with probability at least $1 - \frac{1}{4}\delta$. Therefore, the desired result holds.
\end{proof}

\begin{remark}\label{rmk:relevant_evecs}
The reward dimensions could be irrelevant due to a number of reasons. For instance, because the elements of $\bm{x}_i := \bm{x}_{i2} - \bm{x}_{i1}$ must sum to zero, the vector $[1, 1, \ldots, 1]^T$ must always be orthogonal to every observation $\bm{x}_i$. Alternatively, the MDP's transition dynamics could constrain the expected number of visits to a particular state to be constant regardless of the policy.

Such constraints result in a subspace of $\mathbb{R}^d$ that is irrelevant to learning the optimal policy once the transition dynamics model has converged sufficiently. Therefore, we only need Lemma \ref{lemma:reward_consistency_part1} to be satisfied for eigenvalues of $M_i$ along relevant dimensions in order to asymptotically select the optimal policy. Thus, we can assume that sampled reward vectors $\bm{\tilde{r}}_{i1}, \bm{\tilde{r}}_{i2}$ have been projected onto the relevant subspace of $\mathbb{R}^d$. As a result, in proving that $\frac{\beta_i(\delta)^2}{\lambda_d^{(i)}} \overset{D}\longrightarrow 0$ as $i \longrightarrow \infty$, we can assume that all eigenvectors of $M_i$ belong to the relevant subspace without loss of generality. More formally, we assume without loss of generality that all eigenvectors $\{\bm{v}_j^{(i)}\}$ of $M_i$ satisfy \eqref{eqn:relevance_condition}.
\end{remark}

In Lemma \ref{lemma:reward_consistency_part4}, we will show that as $i \longrightarrow \infty$, $\frac{\beta_i(\delta)^2}{\lambda_d^{(i)}} \overset{D}\longrightarrow 0$. Combined with Lemma \ref{lemma:reward_consistency_part1}, this proves Proposition \ref{prop:reward_consistency}, that the reward samples are convergent in distribution to $\bm{\overline{r}}$.  Lemma \ref{lemma:reward_consistency_part4} proves this result via contradiction, by first assuming that there exists an $i_0$ such that for all $i \ge i_0$, $\frac{\beta_i(\delta)^2}{\lambda_d^{(i)}} \ge \alpha$. The following two lemmas, Lemmas \ref{lemma:reward_consistency_part2}-\ref{lemma:reward_consistency_part3}, are intermediate results leading to Lemma \ref{lemma:reward_consistency_part4}, which utilize this contradiction hypothesis as a premise. In particular, Lemma \ref{lemma:reward_consistency_part2} demonstrates that under the contradiction hypothesis, there is a non-decaying probability of sampling rewards $\bm{\tilde{r}}_{i1}, \bm{\tilde{r}}_{i2}$ that are highly-aligned with the eigenvector $\bm{v}_d^{(i)}$ of $M_i^{-1}$ corresponding to its largest eigenvalue, $(\lambda_d^{(i)})^{-1}$.

\begin{lemma}\label{lemma:reward_consistency_part2}
Assume that for a given iteration $i$, $\beta_i(\delta)^2 \left(\lambda_d^{(i)}\right)^{-1} \ge \alpha$. Then, the reward samples $\bm{\tilde{r}}_{i1}, \bm{\tilde{r}}_{i2}$ satisfy:
\begin{flalign}
P(\bm{\tilde{r}}_{i1}^T\bm{v}_d^{(i)} &\ge a\max_{j < d}|\bm{\tilde{r}}_{i1}^T\bm{v}_j^{(i)}|) \ge c(a) > 0, \label{eqn:reward_alignment_1}\\
P(\bm{\tilde{r}}_{i2}^T\bm{v}_d^{(i)} &\le -a\max_{j < d}|\bm{\tilde{r}}_{i2}^T\bm{v}_j^{(i)}|) \ge c(a) > 0, \label{eqn:reward_alignment_2}
\end{flalign}
where $c: \mathbb{R}_+ \longrightarrow \mathbb{R}_+$ is a continuous, monotonically-decreasing function.
\end{lemma}
\begin{proof}
Recall that the reward samples $\bm{\tilde{r}}_{i1}, \bm{\tilde{r}}_{i2}$ are drawn according to \eqref{eqn:TS_dist}. We will demonstrate that the reward samples can equivalently be expressed as:
\begin{equation}\label{eqn:reward_samples_z_v}
    \bm{\tilde{r}}_{i1} = \bm{\hat{r}}_i + \beta_i(\delta)\sum_{j = 1}^d \left(\lambda_j^{(i)}\right)^{-\frac{1}{2}} z_{ij} \bm{v}_j^{(i)}, z_{ij} \sim \mathcal{N}(0, 1) \text{ i.i.d.},
\end{equation}
and similarly for $\bm{\tilde{r}}_{i2}$. Similarly to \eqref{eqn:TS_dist}, the expression in \eqref{eqn:reward_samples_z_v} has a multivariate Gaussian distribution. We take the expectation and covariance of \eqref{eqn:reward_samples_z_v} with respect to the variables $\{z_{ij}\}$ to show that they match the expressions in \eqref{eqn:TS_dist}: 
\begin{flalign*}
\mathbb{E}\left[\bm{\hat{r}}_i + \beta_i(\delta)\sum_{j = 1}^d \left(\lambda_j^{(i)}\right)^{-\frac{1}{2}} z_{ij} \bm{v}_j^{(i)} \right] &= \bm{\hat{r}}_i + \beta_i(\delta)\sum_{j = 1}^d \left(\lambda_j^{(i)}\right)^{-\frac{1}{2}}\mathbb{E}[z_{ij}]\bm{v}_j^{(i)} = \bm{\hat{r}}_i, \\ \text{Cov}\left[\bm{\hat{r}}_i + \beta_i(\delta)\sum_{j = 1}^d \left(\lambda_j^{(i)}\right)^{-\frac{1}{2}} z_{ij} \bm{v}_j^{(i)} \right] &\overset{(a)}= \mathbb{E}\left[\left(\beta_i(\delta)\sum_{j = 1}^d \left(\lambda_j^{(i)}\right)^{-\frac{1}{2}} z_{ij} \bm{v}_j^{(i)}\right)\left(\beta_i(\delta)\sum_{k = 1}^d \left(\lambda_k^{(i)}\right)^{-\frac{1}{2}} z_{ik} \bm{v}_k^{(i)}\right)^T \right] \\ &= \beta_i(\delta)^2 \sum_{j = 1}^d \sum_{k = 1}^d \left(\lambda_j^{(i)}\right)^{-\frac{1}{2}}\left(\lambda_k^{(i)}\right)^{-\frac{1}{2}}\bm{v}_j^{(i)}\bm{v}_k^{(i) T} \mathbb{E}[z_{ij}z_{ik}] \\ &\overset{(b)}= \beta_i(\delta)^2 \sum_{j = 1}^d \left(\lambda_j^{(i)}\right)^{-1}\bm{v}_j^{(i)}\bm{v}_j^{(i) T} = \beta_i(\delta)^2 M_i^{-1},
\end{flalign*}
which match the expectation and covariance in \eqref{eqn:TS_dist}. In the above, (a) applies the definition $\text{Cov}[\bm{x}] = \mathbb{E}[(\bm{x} - \mathbb{E}[\bm{x}])(\bm{x} - \mathbb{E}[\bm{x}])^T]$, and (b) holds because $\mathbb{E}[z_{ij}z_{ik}] = \text{Cov}[z_{ij}z_{ik}] = \delta_{jk}$, where $\delta_{jk}$ is the Kronecker delta function.

Next, we show that the probability that $\bm{\tilde{r}}_{i1}$ is arbitrarily-aligned with $\bm{v}_d^{(i)}$ is lower-bounded above zero: that is, there exists $c: \mathbb{R}_+ \longrightarrow \mathbb{R}_+$ such that for any $a > 0$, $P(\bm{\tilde{r}}_{i1}^T\bm{v}_d^{(i)} \ge a\max_{j < d}|\bm{\tilde{r}}_{i1}^T\bm{v}_j^{(i)}|) \ge c(a) > 0$. This can be shown by bounding the terms $|\bm{\tilde{r}}_{i1}^T\bm{v}_j^{(i)}|, j < d$, and $\bm{\tilde{r}}_{i1}^T \bm{v}_d^{(i)}$. Firstly, the term $|\bm{\tilde{r}}_{i1}^T\bm{v}_j^{(i)}|$, $j < d$, can be upper-bounded:
\begin{flalign*}
|\bm{\tilde{r}}_{i1}^T\bm{v}_j^{(i)}| &\overset{(a)}= \left|\bm{\hat{r}}_i^T\bm{v}_j^{(i)} + \beta_i(\delta)\sum_{k = 1}^d \left(\lambda_k^{(i)}\right)^{-\frac{1}{2}} z_{ik} \bm{v}_k^{(i) T}\bm{v}_j^{(i)}\right| \overset{(b)}= \left|\bm{\hat{r}}_i^T\bm{v}_j^{(i)} + \beta_i(\delta) \left(\lambda_j^{(i)}\right)^{-\frac{1}{2}} z_{ij}\right| \\&\le |\bm{\hat{r}}_i^T\bm{v}_j^{(i)}| + \beta_i(\delta) \left(\lambda_j^{(i)}\right)^{-\frac{1}{2}} |z_{ij}| \overset{(c)}\le ||\bm{\hat{r}}_i||_2 ||\bm{v}_j^{(i)}||_2 + \beta_i(\delta) \left(\lambda_j^{(i)}\right)^{-\frac{1}{2}} |z_{ij}| \\&\overset{(d)}\le b + \beta_i(\delta) \left(\lambda_j^{(i)}\right)^{-\frac{1}{2}} |z_{ij}|,
\end{flalign*}
where (a) applies \eqref{eqn:reward_samples_z_v}, (b) follows from orthonormality of the eigenbasis, (c) follows from the Cauchy-Schwarz inequality, and (d) uses that $||\bm{\hat{r}}_i||_2 \le b$ (Lemma \ref{lemma:MAP_bound}). Similarly, $\bm{\tilde{r}}_{i1}^T \bm{v}_d^{(i)}$ can be lower-bounded:
\begin{flalign*}
\bm{\tilde{r}}_{i1}^T \bm{v}_d^{(i)} &\overset{(a)}= \bm{\hat{r}}_i^T \bm{v}_d^{(i)} + \beta_i(\delta)\sum_{j = 1}^d \left(\lambda_j^{(i)}\right)^{-\frac{1}{2}} z_{ij} \bm{v}_j^{(i) T}\bm{v}_d^{(i)} \overset{(b)}= \bm{\hat{r}}_i^T \bm{v}_d^{(i)} + \beta_i(\delta)\left(\lambda_d^{(i)}\right)^{-\frac{1}{2}} z_{i1} \\ &\ge -|\bm{\hat{r}}_i^T \bm{v}_d^{(i)}| + \beta_i(\delta) \left(\lambda_d^{(i)}\right)^{-\frac{1}{2}} z_{i1} \overset{(c)}\ge -||\bm{\hat{r}}_i||_2 ||\bm{v}_d^{(i)}||_2 + \beta_i(\delta) \left(\lambda_d^{(i)}\right)^{-\frac{1}{2}} z_{i1} \\ &\overset{(d)}\ge -b + \beta_i(\delta) \left(\lambda_d^{(i)}\right)^{-\frac{1}{2}} z_{i1},
\end{flalign*}
where as before, (a) applies \eqref{eqn:reward_samples_z_v}, (b) follows from orthonormality of the eigenbasis, (c) follows from the Cauchy-Schwarz inequality, and (d) holds via Lemma \ref{lemma:MAP_bound}. Given these upper and lower bounds, the probability $P\left(\bm{\tilde{r}}_{i1}^T\bm{v}_d^{(i)} \ge a\max_{j < d}|\bm{\tilde{r}}_{i1}^T\bm{v}_j^{(i)}|\right)$ can be lower-bounded:
\begin{flalign*}
P\left(\bm{\tilde{r}}_{i1}^T\bm{v}_d^{(i)} \ge a\max_{j < d}|\bm{\tilde{r}}_{i1}^T\bm{v}_j^{(i)}|\right) &\overset{(a)}\ge P\left(-b + \beta_i(\delta) \left(\lambda_d^{(i)}\right)^{-\frac{1}{2}} z_{i1} \ge a\max_{j < d}\left[b + \beta_i(\delta) \left(\lambda_j^{(i)}\right)^{-\frac{1}{2}} |z_{ij}|\right]\right) \\&= P\left(z_{i1} \ge \frac{b \sqrt{\lambda_d^{(i)}}}{\beta_i(\delta)} + a\max_{j < d}\left[\frac{b \sqrt{\lambda_d^{(i)}}}{\beta_i(\delta)} + \sqrt{\frac{\lambda_d^{(i)}}{\lambda_j^{(i)}}} |z_{ij}|\right]\right) \\&\overset{(b)}\ge P\left(z_{i1} \ge \frac{b}{\sqrt{\alpha}} + a\max_{j < d}\left[\frac{b}{\sqrt{\alpha}} + |z_{ij}|\right]\right) \\ &= P\left(z_{i1} \ge \frac{b(1 + a)}{\sqrt{\alpha}} + a\max_{j < d}|z_{ij}| \right) := c(a) > 0,
\end{flalign*}
where (a) results from the upper and lower bounds derived above, and (b) follows because $\frac{\lambda_d^{(i)}}{\lambda_j^{(i)}} \le 1$ and $\beta_i(\delta) \left(\lambda_d^{(i)}\right)^{-\frac{1}{2}} \ge \sqrt{\alpha}$ by assumption. The function $c(a) > 0$ is continuous and decreasing in $a$.

By identical arguments, $P(\bm{\tilde{r}}_{i2}^T\bm{v}_d^{(i)} \le -a\max_{j < d}|\bm{\tilde{r}}_{i2}^T\bm{v}_j^{(i)}|) \ge c(a)$. Thus, for any $a > 0$ and set of eigenvectors $\bm{v}_j^{(i)}$:
\begin{flalign*}
P(\bm{\tilde{r}}_{i1}^T\bm{v}_d^{(i)} &\ge a\max_{j < d}|\bm{\tilde{r}}_{i1}^T\bm{v}_j^{(i)}|) \ge c(a) > 0, \\
P(\bm{\tilde{r}}_{i2}^T\bm{v}_d^{(i)} &\le -a\max_{j < d}|\bm{\tilde{r}}_{i2}^T\bm{v}_j^{(i)}|) \ge c(a) > 0.
\end{flalign*}
\end{proof}

Next, we show that given sampled rewards $\bm{\tilde{r}}_{i1}, \bm{\tilde{r}}_{i2}$ that are highly-aligned with the eigenvector $\bm{v}_d^{(i)}$ of $M_i$ as in Lemma \ref{lemma:reward_consistency_part2}, there is a lower-bounded probability of sampling trajectories that have non-zero projection onto this eigenvector.

\begin{lemma}\label{lemma:reward_consistency_part3}
Assume that there exists $i_0$ such that for $i > i_0$, $\beta_i(\delta)^2 \left(\lambda_d^{(i)}\right)^{-1} \ge \alpha$. Then, there exists $i^\prime \ge i_0$ and a constant $c^\prime > 0$ such that for $i > i^\prime$:
\begin{flalign}\label{eqn:exp_bound}
\mathbb{E}\left[\left|\bm{x}_i^T \bm{v}_d^{(i)} \right| \right] \ge c^\prime > 0,
\end{flalign}
where $c^\prime > 0$ depends only on the MDP parameters $\bm{\overline{p}}$ and $\bm{\overline{r}}$, so that in particular, \eqref{eqn:exp_bound} holds for any eigenvector $\bm{v}_d^{(i)}$.
\end{lemma}
\begin{proof}
By Lemma \ref{lemma:reward_consistency_part2}, \eqref{eqn:reward_alignment_1} and \eqref{eqn:reward_alignment_2} both hold. We will refer to the events in \eqref{eqn:reward_alignment_1} and \eqref{eqn:reward_alignment_2}, $\{\bm{\tilde{r}}_{i1}^T\bm{v}_d^{(i)} \ge a\max_{j < d}|\bm{\tilde{r}}_{i1}^T\bm{v}_j^{(i)}|\}$ and $\{\bm{\tilde{r}}_{i2}^T\bm{v}_d^{(i)} \le -a\max_{j < d}|\bm{\tilde{r}}_{i2}^T\bm{v}_j^{(i)}|\}$, as events $A(a)$ and $B(a)$, respectively. From Lemma \ref{lemma:reward_consistency_part2}, $A(a)$ and $B(a)$ have positive probability for any $a$.

We will show that by setting $a$ to a large-enough value, under events $A(a)$ and $B(a)$, value iteration samples policies $\pi_{i1}$, $\pi_{i2}$ such that $\mathbb{E}\left[\left|\bm{x}_i^T \bm{v}_d^{(i)} \right| \right] \ge c^\prime > 0$ for some $c^\prime > 0$, for sufficiently-high $i$ and for any unit vector $\bm{v}_d^{(i)}$.

First, note that under events $A(a)$ and $B(a)$, as $a \longrightarrow \infty$, $\frac{\bm{\tilde{r}}_{i1}}{||\bm{\tilde{r}}_{i1}||_2} \longrightarrow \bm{v}_d^{(i)}$ and $\frac{\bm{\tilde{r}}_{i2}}{||\bm{\tilde{r}}_{i2}||_2} \longrightarrow -\bm{v}_d^{(i)}$. Let $\varepsilon > 0$. Under event $A(a)$ for sufficiently-large $a$, $\left|\left|\frac{\bm{\tilde{r}}_{i1}}{||\bm{\tilde{r}}_{i1}||_2} - \bm{v}_d^{(i)}\right|\right|_2 < \varepsilon$. Define $a_{\text{min}, 1}(\varepsilon, \bm{v}_1^{(i)}, \ldots, \bm{v}_d^{(i)})$ as the minimum value of $a$ such that $A(a)$ implies $\left|\left|\frac{\bm{\tilde{r}}_{i1}}{||\bm{\tilde{r}}_{i1}||_2} - \bm{v}_d^{(i)}\right|\right|_2 \le \frac{\varepsilon}{2}$, given the eigenbasis $\{\bm{v}_1^{(i)}, \ldots, \bm{v}_d^{(i)}\}$. Because the inequality defining $A(a)$ is continuous in $a$, $\bm{\tilde{r}}_{i1}$, and the eigenbasis $\{\bm{v}_1^{(i)}, \ldots, \bm{v}_d^{(i)}\}$, the function $a_{\text{min}, 1}(\varepsilon, \bm{v}_1^{(i)}, \ldots, \bm{v}_d^{(i)})$ is also continuous in the eigenbasis $\{\bm{v}_1^{(i)}, \ldots, \bm{v}_d^{(i)}\}$. Because $a_{\text{min}, 1}(\varepsilon, \bm{v}_1^{(i)}, \ldots, \bm{v}_d^{(i)})$ is positive for all $\{\bm{v}_1^{(i)}, \ldots, \bm{v}_d^{(i)}\}$, and the set of all eigenbases $\{\bm{v}_1^{(i)}, \ldots, \bm{v}_d^{(i)}\}$ is compact, there exists $a_{\text{min}, 1}(\varepsilon)$ such that for any eigenbasis, if $A(a)$ holds for $a \ge a_{\text{min}, 1}(\varepsilon)$, then $\left|\left|\frac{\bm{\tilde{r}}_{i1}}{||\bm{\tilde{r}}_{i1}||_2} - \bm{v}_d^{(i)}\right|\right|_2 < \varepsilon$.

By the same arguments, there exists $a_{\text{min}, 2}(\varepsilon)$ such that for any eigenbasis, if $B(a)$ holds for $a \ge a_{\text{min}, 2}(\varepsilon)$, then $\left|\left|\frac{\bm{\tilde{r}}_{i2}}{||\bm{\tilde{r}}_{i2}||_2} - (-\bm{v}_d^{(i)})\right|\right|_2 < \varepsilon$. Taking $a_{\text{min}}(\varepsilon) := \max\{a_{\text{min}, 1}(\varepsilon), a_{\text{min}, 2}(\varepsilon)\}$, then for any $a \ge a_{\text{min}}(\varepsilon)$, under events $A(a)$ and $B(a)$, both $\left|\left|\frac{\bm{\tilde{r}}_{i1}}{||\bm{\tilde{r}}_{i1}||_2} - \bm{v}_d^{(i)}\right|\right|_2 < \varepsilon$ and $\left|\left|\frac{\bm{\tilde{r}}_{i2}}{||\bm{\tilde{r}}_{i2}||_2} - (-\bm{v}_d^{(i)})\right|\right|_2 < \varepsilon$ hold.

Next, we will show that by setting $\varepsilon$ small enough, the inequality $\left|\left|\frac{\bm{\tilde{r}}_{i1}}{||\bm{\tilde{r}}_{i1}||_2} - \bm{v}_d^{(i)}\right|\right|_2 < \varepsilon$ implies that the expected reward accrued by $\pi_{i1}$ with respect to $\bm{v}_d^{(i)}$, that is, $V(\bm{\overline{p}}, \bm{v}_d^{(i)}, \pi_{i1})$, is close to the maximum possible expected reward with respect to $\bm{v}_d^{(i)}$, $\max_{\pi} V(\bm{\overline{p}}, \bm{v}_d^{(i)}, \pi) = V(\bm{\overline{p}}, \bm{v}_d^{(i)}, \pi_{vi}(\bm{\overline{p}}, \bm{v}_d^{(i)}))$. (The same approach yields an equivalent result for $\bm{\tilde{r}}_{i2}$.)

Assume that $\left|\left|\frac{\bm{\tilde{r}}_{i1}}{||\bm{\tilde{r}}_{i1}||_2} - \bm{v}_d^{(i)}\right|\right|_2 < \varepsilon$, and let $\varepsilon^\prime > 0$. We will show that $|V(\bm{\overline{p}}, \bm{v}_d^{(i)}, \pi_{vi}(\bm{\overline{p}}, \bm{v}_d^{(i)})) - V(\bm{\overline{p}}, \bm{v}_d^{(i)}, \pi_{i1})| < \varepsilon^\prime$ for small-enough $\varepsilon$ and when $\bm{\tilde{p}}_{i1}$ has sufficiently converged in distribution to $\bm{\overline{p}}$ (as is guaranteed to occur by Proposition \ref{prop:dyn_consistency}):

\begingroup 
\allowdisplaybreaks
\begin{flalign}
\Bigr|V(\bm{\overline{p}}, \bm{v}_d^{(i)}, \pi_{vi}(\bm{\overline{p}}, &\bm{v}_d^{(i)})) - V(\bm{\overline{p}}, \bm{v}_d^{(i)}, \pi_{i1})\Bigr| \overset{(a)}= \left|V\left(\bm{\overline{p}}, \bm{v}_d^{(i)}, \pi_{vi}(\bm{\overline{p}}, \bm{v}_d^{(i)})\right) - V\left(\bm{\overline{p}}, \bm{v}_d^{(i)}, \pi_{vi}\left(\bm{\tilde{p}}_{i1}, \frac{\bm{\tilde{r}}_{i1}}{||\bm{\tilde{r}}_{i1}||_2}\right)\right)\right| \nonumber\\ &= \Bigg|V\left(\bm{\overline{p}}, \bm{v}_d^{(i)}, \pi_{vi}(\bm{\overline{p}}, \bm{v}_d^{(i)})\right) - V\left(\bm{\overline{p}}, \frac{\bm{\tilde{r}}_{i1}}{||\bm{\tilde{r}}_{i1}||_2}, \pi_{vi}\left(\bm{\overline{p}}, \frac{\bm{\tilde{r}}_{i1}}{||\bm{\tilde{r}}_{i1}||_2}\right)\right) \nonumber\\ &\hspace{5mm}+ V\left(\bm{\overline{p}}, \frac{\bm{\tilde{r}}_{i1}}{||\bm{\tilde{r}}_{i1}||_2}, \pi_{vi}\left(\bm{\overline{p}}, \frac{\bm{\tilde{r}}_{i1}}{||\bm{\tilde{r}}_{i1}||_2}\right)\right) - V\left(\bm{\tilde{p}}_{i1}, \frac{\bm{\tilde{r}}_{i1}}{||\bm{\tilde{r}}_{i1}||_2}, \pi_{vi}\left(\bm{\tilde{p}}_{i1}, \frac{\bm{\tilde{r}}_{i1}}{||\bm{\tilde{r}}_{i1}||_2}\right)\right) \nonumber\\ &\hspace{5mm}+ V\left(\bm{\tilde{p}}_{i1}, \frac{\bm{\tilde{r}}_{i1}}{||\bm{\tilde{r}}_{i1}||_2}, \pi_{vi}\left(\bm{\tilde{p}}_{i1}, \frac{\bm{\tilde{r}}_{i1}}{||\bm{\tilde{r}}_{i1}||_2}\right)\right) - V\left(\bm{\overline{p}}, \frac{\bm{\tilde{r}}_{i1}}{||\bm{\tilde{r}}_{i1}||_2}, \pi_{vi}\left(\bm{\tilde{p}}_{i1}, \frac{\bm{\tilde{r}}_{i1}}{||\bm{\tilde{r}}_{i1}||_2}\right)\right) \nonumber\\ &\hspace{5mm}+ V\left(\bm{\overline{p}}, \frac{\bm{\tilde{r}}_{i1}}{||\bm{\tilde{r}}_{i1}||_2}, \pi_{vi}\left(\bm{\tilde{p}}_{i1}, \frac{\bm{\tilde{r}}_{i1}}{||\bm{\tilde{r}}_{i1}||_2}\right)\right) - V\left(\bm{\overline{p}}, \bm{v}_d^{(i)}, \pi_{vi}\left(\bm{\tilde{p}}_{i1}, \frac{\bm{\tilde{r}}_{i1}}{||\bm{\tilde{r}}_{i1}||_2}\right)\right)\Bigg| \nonumber\\ &\overset{(b)}\le \left|V\left(\bm{\overline{p}}, \frac{\bm{\tilde{r}}_{i1}}{||\bm{\tilde{r}}_{i1}||_2}, \pi_{vi}\left(\bm{\tilde{p}}_{i1}, \frac{\bm{\tilde{r}}_{i1}}{||\bm{\tilde{r}}_{i1}||_2}\right)\right) - V\left(\bm{\overline{p}}, \bm{v}_d^{(i)}, \pi_{vi}\left(\bm{\tilde{p}}_{i1}, \frac{\bm{\tilde{r}}_{i1}}{||\bm{\tilde{r}}_{i1}||_2}\right)\right)\right|\label{eqn:diff1} \\ &\hspace{5mm}+\left|V\left(\bm{\overline{p}}, \bm{v}_d^{(i)}, \pi_{vi}(\bm{\overline{p}}, \bm{v}_d^{(i)})\right) - V\left(\bm{\overline{p}}, \frac{\bm{\tilde{r}}_{i1}}{||\bm{\tilde{r}}_{i1}||_2}, \pi_{vi}\left(\bm{\overline{p}}, \frac{\bm{\tilde{r}}_{i1}}{||\bm{\tilde{r}}_{i1}||_2}\right)\right)\right| \label{eqn:diff2}\\ &\hspace{5mm}+ \left|V\left(\bm{\overline{p}}, \frac{\bm{\tilde{r}}_{i1}}{||\bm{\tilde{r}}_{i1}||_2}, \pi_{vi}\left(\bm{\overline{p}}, \frac{\bm{\tilde{r}}_{i1}}{||\bm{\tilde{r}}_{i1}||_2}\right)\right) - V\left(\bm{\tilde{p}}_{i1}, \frac{\bm{\tilde{r}}_{i1}}{||\bm{\tilde{r}}_{i1}||_2}, \pi_{vi}\left(\bm{\tilde{p}}_{i1}, \frac{\bm{\tilde{r}}_{i1}}{||\bm{\tilde{r}}_{i1}||_2}\right)\right)\right| \label{eqn:diff3}\\ &\hspace{5mm}+ \left|V\left(\bm{\tilde{p}}_{i1}, \frac{\bm{\tilde{r}}_{i1}}{||\bm{\tilde{r}}_{i1}||_2}, \pi_{vi}\left(\bm{\tilde{p}}_{i1}, \frac{\bm{\tilde{r}}_{i1}}{||\bm{\tilde{r}}_{i1}||_2}\right)\right) - V\left(\bm{\overline{p}}, \frac{\bm{\tilde{r}}_{i1}}{||\bm{\tilde{r}}_{i1}||_2}, \pi_{vi}\left(\bm{\tilde{p}}_{i1}, \frac{\bm{\tilde{r}}_{i1}}{||\bm{\tilde{r}}_{i1}||_2}\right)\right)\right| \label{eqn:diff4},
\end{flalign}
\endgroup
where (a) uses that $\pi_{i1} = \pi_{vi}(\bm{\tilde{p}}_{i1}, \bm{\tilde{r}}_{i1})$ by definition, and also that positive scaling of the reward argument of $\pi_{vi}(\bm{p}, \bm{r})$ does not affect its output; and (b) applies the triangle inequality. Next, we will show that each of \eqref{eqn:diff1}-\eqref{eqn:diff4} can be upper-bounded by $\frac{1}{4}\varepsilon^\prime$ (for \eqref{eqn:diff3} and \eqref{eqn:diff4} with high probability) by appropriate choice of $\varepsilon$ and by utilizing that $\bm{\tilde{p}}_{i1} \overset{D}\longrightarrow \bm{\overline{p}}$ (Proposition \ref{prop:dyn_consistency}).

Beginning with \eqref{eqn:diff1}, because $V(\bm{p}, \bm{r}, \pi)$ is linear in $\bm{r}$, it is uniformly continuous in $\bm{r}$ for fixed transition dynamics $\bm{p}$ and policy $\pi$. So, for fixed dynamics $\bm{\overline{p}}$ and policy $\pi$ and for any reward vector $\bm{r}$, there exists $\varepsilon_{\pi}$ such that if $||\bm{r} - \bm{r}^\prime|| < \varepsilon_{\pi}$, then $|V(\bm{p}, \bm{r}, \pi) - V(\bm{p}, \bm{r}^\prime, \pi)| < \frac{1}{4}\varepsilon^\prime$. Because there are finitely-many deterministic policies, there exists $\varepsilon_1 > 0$ such that $\varepsilon_1 \le \varepsilon_{\pi}$ for all $\pi$. Therefore, for any policy $\pi$, if $||\bm{r} - \bm{r}^\prime||_2 < \varepsilon_1$, then $|V(\bm{\overline{p}}, \bm{r}, \pi) - V(\bm{\overline{p}}, \bm{r}^\prime, \pi)| < \frac{1}{4}\varepsilon^\prime$. The expression in \eqref{eqn:diff1} is thus upper-bounded by $\frac{1}{4}\varepsilon^\prime$ if $\varepsilon < \varepsilon_1$.

To upper-bound \eqref{eqn:diff2}, observe that $V(\bm{p}, \bm{r}, \pi_{vi}(\bm{p}, \bm{r})) = \max_\pi V(\bm{p}, \bm{r}, \pi)$ is also uniformly continuous in $\bm{r}$ for fixed transition dynamics $\bm{p}$: the maximum over finitely-many uniformly continuous functions is also uniformly continuous. Thus, there exists $\varepsilon_2 > 0$ such that if $||\bm{r} - \bm{r}^\prime||_2 < \varepsilon_2$, then $|V(\bm{\overline{p}}, \bm{r}, \pi_{vi}(\bm{\overline{p}}, \bm{r})) - V(\bm{\overline{p}}, \bm{r}^\prime, \pi_{vi}(\bm{\overline{p}}, \bm{r}^\prime))| < \frac{1}{4}\varepsilon^\prime$. The expression in \eqref{eqn:diff2} is thus upper-bounded by $\frac{1}{4}\varepsilon^\prime$ if $\varepsilon < \varepsilon_2$.

To obtain high-probability upper bounds for \eqref{eqn:diff3} and \eqref{eqn:diff4}, we apply Lemma \ref{lemma:V_conv_in_p}. For sufficiently-high $i$, each of the following holds with probability at least $1 - \frac{1}{2}\delta^\prime$:
\begin{flalign*}
    &\left|V\left(\bm{\overline{p}}, \frac{\bm{\tilde{r}}_{i1}}{||\bm{\tilde{r}}_{i1}||_2}, \pi_{vi}\left(\bm{\overline{p}}, \frac{\bm{\tilde{r}}_{i1}}{||\bm{\tilde{r}}_{i1}||_2}\right)\right) - V\left(\bm{\tilde{p}}_{i1}, \frac{\bm{\tilde{r}}_{i1}}{||\bm{\tilde{r}}_{i1}||_2}, \pi_{vi}\left(\bm{\tilde{p}}_{i1}, \frac{\bm{\tilde{r}}_{i1}}{||\bm{\tilde{r}}_{i1}||_2}\right)\right)\right| < \frac{1}{4}\varepsilon^\prime, \\
    &\left|V\left(\bm{\tilde{p}}_{i1}, \frac{\bm{\tilde{r}}_{i1}}{||\bm{\tilde{r}}_{i1}||_2}, \pi\right) - V\left(\bm{\overline{p}}, \frac{\bm{\tilde{r}}_{i1}}{||\bm{\tilde{r}}_{i1}||_2}, \pi\right)\right| < \frac{1}{4}\varepsilon^\prime,
\end{flalign*}
where the second statement holds for any policy $\pi$, and in particular for $\pi = \pi_{vi}\left(\bm{\tilde{p}}_{i1}, \frac{\bm{\tilde{r}}_{i1}}{||\bm{\tilde{r}}_{i1}||_2}\right)$.

We combine the bounds for \eqref{eqn:diff1}-\eqref{eqn:diff4}, setting $\varepsilon < \min\{\varepsilon_1, \varepsilon_2\}$ and $i > i^\prime$. Thus, for any $\varepsilon^\prime, \delta^\prime > 0$, we have shown that by setting $\varepsilon$ small enough and taking $i > i^\prime$:
\begin{equation*}
    \Bigr|V(\bm{\overline{p}}, \bm{v}_d^{(i)}, \pi_{vi}(\bm{\overline{p}}, \bm{v}_d^{(i)})) - V(\bm{\overline{p}}, \bm{v}_d^{(i)}, \pi_{i1})\Bigr| < \varepsilon^\prime \text{ with probability at least } 1 - \delta^\prime.
\end{equation*}

Combining with the analogous result for $\pi_{i2}$ and $-\bm{v}_d^{(i)}$ yields that for any $\varepsilon^\prime, \delta^\prime > 0$, there exists sufficiently-small $\varepsilon$ and large-enough $i^\prime$ such that for $i > i^\prime$:
\begin{flalign}\label{eqn:V_proximity}
    &\Bigr|V(\bm{\overline{p}}, \bm{v}_d^{(i)}, \pi_{vi}(\bm{\overline{p}}, \bm{v}_d^{(i)})) - V(\bm{\overline{p}}, \bm{v}_d^{(i)}, \pi_{i1})\Bigr| < \varepsilon^\prime \text{ with probability at least } 1 - \delta^\prime, \text{ and} \\
    &\Bigr|V(\bm{\overline{p}}, -\bm{v}_d^{(i)}, \pi_{vi}(\bm{\overline{p}}, -\bm{v}_d^{(i)})) - V(\bm{\overline{p}}, -\bm{v}_d^{(i)}, \pi_{i2})\Bigr| < \varepsilon^\prime \text{ with probability at least } 1 - \delta^\prime.  \nonumber   
\end{flalign}

Next, we will set $\varepsilon^\prime$ to a small enough number to achieve $\left|\mathbb{E}[\bm{x}_i^T \bm{v}_d^{(i)}]\right| > \varepsilon^\prime > 0$. Firstly, note that  that $|\mathbb{E}[\bm{x}_i^T \bm{v}_d^{(i)}]|$ is maximized when setting $\pi_{i1} = \pi_{vi}(\bm{\overline{p}}, \bm{v}_d^{(i)})$ and $\pi_{i2} = \pi_{vi}(\bm{\overline{p}}, -\bm{v}_d^{(i)})$:
\begin{flalign*}
    \max_{\pi_1, \pi_2}\Bigr|\mathbb{E}[\bm{x}_i^T \bm{v}_d^{(i)}] \,|\, \pi_{i1} &= \pi_1, \pi_{i2} = \pi_2 \Bigr| = \max_{\pi_1, \pi_2}\left|\mathbb{E}[\bm{x}_{i1}^T \bm{v}_d^{(i)} - \bm{x}_{i2}^T \bm{v}_d^{(i)}]  \,|\, \pi_{i1} = \pi_1, \pi_{i2} = \pi_2 \right| \\ &= \left|\mathbb{E}[\bm{x}_{i1}^T \bm{v}_d^{(i)} - \bm{x}_{i2}^T \bm{v}_d^{(i)} \,|\, \pi_{i1} = \text{argmax}_{\pi}\mathbb{E}[\bm{x}_{i1}^T \bm{v}_d^{(i)}], \pi_{i2} = \text{argmin}_{\pi}\mathbb{E}[\bm{x}_{i2}^T \bm{v}_d^{(i)}]]\right| \\ &= \left|\mathbb{E}[\bm{x}_{i1}^T \bm{v}_d^{(i)} - \bm{x}_{i2}^T \bm{v}_d^{(i)} \,|\, \pi_{i1} = \text{argmax}_{\pi}\mathbb{E}[\bm{x}_{i1}^T \bm{v}_d^{(i)}], \pi_{i2} = \text{argmax}_{\pi}\mathbb{E}[\bm{x}_{i2}^T (-\bm{v}_d^{(i)})]]\right| \\ &= \left|\mathbb{E}[\bm{x}_{i1}^T \bm{v}_d^{(i)} - \bm{x}_{i2}^T \bm{v}_d^{(i)} \,|\, \pi_{i1} = \text{argmax}_{\pi}V(\bm{\overline{p}}, \bm{v}_d^{(i)}, \pi), \pi_{i2} = \text{argmax}_{\pi}V(\bm{\overline{p}}, -\bm{v}_d^{(i)}, \pi)]\right| \\ &= \left|\mathbb{E}[\bm{x}_{i1}^T \bm{v}_d^{(i)} - \bm{x}_{i2}^T \bm{v}_d^{(i)} \,|\, \pi_{i1} = \pi_{vi}(\bm{\overline{p}}, \bm{v}_d^{(i)}), \pi_{i2} = \pi_{vi}(\bm{\overline{p}}, -\bm{v}_d^{(i)})]\right|.
\end{flalign*}

From Lemma \ref{lemma:relevance} and Remark \ref{rmk:relevant_evecs}, we can assume without loss of generality that for all $\bm{v}_d^{(i)}$, $\max_{\pi_1, \pi_2} \left|\mathbb{E}[\bm{x}_i^T \bm{v}_d^{(i)} \,|\, \pi_{i1} = \pi_1, \pi_{i2} = \pi_2]\right| > 0$. Because $\left|\mathbb{E}[\bm{x}_i^T \bm{v}_d^{(i)} \,|\, \pi_{i1} = \pi_1, \pi_{i2} = \pi_2]\right|$ is continuous in $\bm{v}_d^{(i)}$ for fixed $\pi_1, \pi_2$, and a maximum over finitely-many continuous functions is continuous, $\max_{\pi_1, \pi_2} \left|\mathbb{E}[\bm{x}_i^T \bm{v}_d^{(i)} \,|\, \pi_{i1} = \pi_1, \pi_{i2} = \pi_2]\right|$ is also continuous in $\bm{v}_d^{(i)}$. Because $\bm{v}_d^{(i)}$ belongs to the compact set of unit vectors, the expression achieves a minimum positive value on the set of possible $\bm{v}_d^{(i)}$, and thus, there exists $\eta > 0$ such that $\max_{\pi_1, \pi_2} \left|\mathbb{E}[\bm{x}_i^T \bm{v}_d^{(i)} \,|\, \pi_{i1} = \pi_1, \pi_{i2} = \pi_2]\right| \ge \eta > 0$. Setting $\varepsilon^\prime := \frac{\eta}{3}$: 
\begin{flalign*}
0 < 3\varepsilon^\prime &= \eta \le \max_{\pi_1, \pi_2} \left|\mathbb{E}[\bm{x}_i^T \bm{v}_d^{(i)} \,|\, \pi_{i1} = \pi_1, \pi_{i2} = \pi_2]\right| = \left|\max_{\pi_1} \mathbb{E}[\bm{x}_{i1}^T \bm{v}_d^{(i)} \,|\, \pi_{i1} = \pi_1] - \min_{\pi_2} \mathbb{E}[\bm{x}_{i2}^T \bm{v}_d^{(i)} \,|\, \pi_{i2} = \pi_2] \right| \\ &= \left|\max_{\pi_1} V(\bm{\overline{p}}, \bm{v}_d^{(i)}, \pi_1) - \min_{\pi_2} V(\bm{\overline{p}}, \bm{v}_d^{(i)}, \pi_2) \right| = \left|\max_{\pi_1} V(\bm{\overline{p}}, \bm{v}_d^{(i)}, \pi_1) + \max_{\pi_2} [-V(\bm{\overline{p}}, \bm{v}_d^{(i)}, \pi_2)] \right| \\ &= \left|\max_{\pi_1} V(\bm{\overline{p}}, \bm{v}_d^{(i)}, \pi_1) + \max_{\pi_2} V(\bm{\overline{p}}, -\bm{v}_d^{(i)}, \pi_2) \right| = \left|V(\bm{\overline{p}}, \bm{v}_d^{(i)}, \pi_{vi}(\bm{\overline{p}}, \bm{v}_d^{(i)})) + V(\bm{\overline{p}}, -\bm{v}_d^{(i)}, \pi_{vi}(\bm{\overline{p}}, -\bm{v}_d^{(i)})) \right| \\ &\overset{(a)}\le \left|V(\bm{\overline{p}}, \bm{v}_d^{(i)}, \pi_{i1}) + V(\bm{\overline{p}}, -\bm{v}_d^{(i)}, \pi_{i2}) \right| + 2\varepsilon^\prime,
\end{flalign*}
where (a) holds with probability at least $1 - 2\delta^\prime$ by \eqref{eqn:V_proximity}. Rearranging terms, with probability at least $1 - 2\delta^\prime$,
\begin{flalign*}
\varepsilon^\prime < \left|V(\bm{\overline{p}}, \bm{v}_d^{(i)}, \pi_{i1}) + V(\bm{\overline{p}}, -\bm{v}_d^{(i)}, \pi_{i2}) \right| = \left|V(\bm{\overline{p}}, \bm{v}_d^{(i)}, \pi_{i1}) - V(\bm{\overline{p}}, \bm{v}_d^{(i)}, \pi_{i2}) \right| = \left|\mathbb{E}[\bm{x}_i^T\bm{v}_d^{(i)}] \right|.
\end{flalign*}

This implies that:
\begin{flalign*}
\mathbb{E}\left[\left|\bm{x}_i^T \bm{v}_d^{(i)} \right| \right] \overset{(a)}\ge \left|\mathbb{E}\left[\bm{x}_i^T \bm{v}_d^{(i)}  \right]\right| \ge c^\prime > 0 \text{ for some positive } c^\prime \text{ and for all } i > i^\prime \text{ and } \bm{v}_d^{(i)},
\end{flalign*}
where (a) holds via Jensen's inequality and $c^\prime := \varepsilon^\prime$.
\end{proof}

We are now equipped to complete the proof of asymptotic consistency of the reward model.

\begin{lemma}\label{lemma:reward_consistency_part4}
As $i \longrightarrow \infty$, $\frac{\beta_i(\delta)^2}{\lambda_d^{(i)}} \overset{D}\longrightarrow 0$, where $\lambda_d^{(i)}$ is the minimum eigenvalue of $M_i$.
\end{lemma}
\begin{proof}
We will first show that $\underset{i \longrightarrow \infty}\liminf \, \beta_i(\delta)^2 \left(\lambda_d^{(i)}\right)^{-1} = 0$ via a proof by contradiction. Assume that:
\begin{equation}\label{eqn:contradiction_hypothesis}
    \underset{i \longrightarrow \infty}\liminf \, \beta_i(\delta)^2 \left(\lambda_d^{(i)}\right)^{-1} = 2\alpha > 0.
\end{equation}
If \eqref{eqn:contradiction_hypothesis} is true, then there exists $i_0$ such that for all $i \ge i_0$, $\beta_i(\delta)^2 \left(\lambda_d^{(i)}\right)^{-1} \ge \alpha$. In the following, we assume that $i \ge i_0$.
Since $\beta_i(\delta)$ increases at most logarithmically in $i$, it suffices to show that $\lambda_d^{(i)}$ increases at least linearly on average to achieve a contradiction with \eqref{eqn:contradiction_hypothesis}.

Under the contradiction hypothesis, Lemmas \ref{lemma:reward_consistency_part2} and \ref{lemma:reward_consistency_part3} both hold. Due to Lemma \ref{lemma:reward_consistency_part3}, \algo~will infinitely-often, and at a non-decaying rate, sample trajectory pairs such that $|\bm{x}_i^T \bm{v}_d^{(i)}|$ is lower-bounded away from zero. At iteration $n$, we analyze the effect of this guarantee upon $\lambda_d^{(n)}$. Note that $\lambda_d^{(n)}$ corresponds to the eigenvector $\bm{v}_d^{(n)}$ of $M_n$, and so:
\begin{flalign}\label{eqn:lambda_d}
\lambda_d^{(n)} = \bm{v}_d^{(n) T}M_n \bm{v}_d^{(n)} \overset{(a)}= \bm{v}_d^{(n) T}\left(\lambda I + \sum_{i = 1}^{n - 1}\bm{x}_i\bm{x}_i^T \right) \bm{v}_d^{(n)} = \lambda + \sum_{i = 1}^{n - 1}\left(\bm{x}_i^T\bm{v}_d^{(n)}\right)^2,
\end{flalign}
where (a) follows from the definition of $M_n$. Note that while the right-hand side expression of \eqref{eqn:lambda_d} depends upon $\bm{x}_i^T \bm{v}_d^{(n)}$ for $i < n$, \eqref{eqn:exp_bound} depends upon $\bm{x}_i^T \bm{v}_d^{(i)}$. Clearly, if $\bm{v}_d^{(i)}$ remained constant in $i$, then the combination of \eqref{eqn:exp_bound} and \eqref{eqn:lambda_d} would suffice to prove that $\lambda_d^{(n)} \longrightarrow \infty$ with at least a linear on-average rate; however, $\bm{v}_d^{(i)}$ can vary with $i$ over the entire space of unit vectors in $\mathbb{R}^d$.

We leverage that $\bm{v}_d^{(i)}$ is a unit vector, that the set of unit vectors in $\mathbb{R}^d$ is compact, and that any infinite cover of a compact set has a finite subcover. In particular, for any $\varepsilon > 0$, there exist sets $S_1, \ldots, S_K \subset \mathbb{R}^d$, $K < \infty$, such that:
\begin{enumerate}
    \item For $\bm{v} \in \mathbb{R}^d$ such that $||\bm{v}||_2 = 1$, $\bm{v} \in S_k$ for some $k \in \{1, \ldots, K\}$, and 
    \item If $\bm{v}_1, \bm{v}_2 \in S_k$, then $||\bm{v}_1 - \bm{v}_2|| < \varepsilon$.
\end{enumerate}

We will show that there exists a sequence $(n_i) \in \mathbb{N}$ such that $\bm{v}_d^{(n_i)} \in S_k$ for fixed $k \in\{1, \ldots, K\}$, with the events $\bm{v}_d^{(n_i)} \in S_k$ corresponding to the indices $(n_i)$ occurring at some non-decaying frequency. Then, by appropriately choosing $\varepsilon$, we will use \eqref{eqn:lambda_d} and the mutual proximity of the vectors $\bm{v}_d^{(n_i)}$ to show that $\lambda_d^{(n)}$ increases with an at-least linear rate.

Observe that for any number of total iterations $N$, there exists an integer $k \in \{1, \ldots, K\}$ such that $\bm{v}_d^{(i)} \in S_k$ during at least $\frac{N}{K}$ iterations. Because $K$ is a constant and $\frac{N}{K}$ is linear in $N$, the number of iterations in which $\bm{v}_d^{(i)} \in S_k$ is at least linear in $N$ for some $k$. The right-hand sum in \eqref{eqn:lambda_d} can then be divided according to the indices $(n_i)$ and the remaining indices:
\begin{flalign}\label{eqn:lambda_2}
\lambda_d^{(n_j)} = \lambda + \sum_{i = 1}^{n_j - 1}\left(\bm{x}_i^T\bm{v}_d^{(n_i)}\right)^2 = \lambda + \sum_{i = 1}^{j - 1}\left(\bm{x}_{n_i}^T\bm{v}_d^{(n_i)}\right)^2 + \sum_{j = 1; j\notin \{n_1, n_2, \ldots, n_{j - 1}\}}^{n_j - 1}\left(\bm{x}_j^T\bm{v}_d^{(n_i)}\right)^2.
\end{flalign}

The latter sum in \eqref{eqn:lambda_2} is non-decreasing in $n_j$, as all of its terms are non-negative. In the former sum, $||\bm{v}_d^{(n_j)} - \bm{v}_d^{(n_i)}|| < \varepsilon$ for each $i \in \{1, \ldots, j - 1\}$. Defining $\bm{\delta} := \bm{v}_d^{(n_j)} - \bm{v}_d^{(n_i)}$, so that $||\bm{\delta}||_2 \le \varepsilon$:
\begin{flalign}\label{eqn_v_tn_bound}
\left(\bm{x}_{n_i}^T\bm{v}_d^{(n_i)}\right)^2 =
\left(\bm{x}_{n_i}^T\left(\bm{v}_d^{(n_i)} + \bm{\delta} \right)\right)^2 = \left(\bm{x}_{n_i}^T \bm{v}_d^{(n_i)} + \bm{x}_{n_i}^T\bm{\delta} \right)^2 \ge \left(\left|\bm{x}_{n_i}^T \bm{v}_d^{(n_i)}\right| - \left|\bm{x}_{n_i}^T\bm{\delta} \right|\right)^2.
\end{flalign}

By the Cauchy-Schwarz inequality, $\left|\bm{x}_{n_i}^T\bm{\delta} \right| \le ||\bm{\delta}||_2*||\bm{x}_{n_i}||_2 \le 2\varepsilon h$, where $h$ is the trajectory horizon. Because \eqref{eqn:exp_bound} requires that $\mathbb{E}\left[\left|\bm{x}_{n_i}^T \bm{v}_d^{(n_i)} \right| \right] \ge c^\prime > 0$, one can choose $\varepsilon$ small enough that $\mathbb{E}\left[\left|\bm{x}_{n_i}^T \bm{v}_d^{(n_i)}\right| - \left|\bm{x}_{n_i}^T\bm{\delta} \right|\right] \ge c^\prime - 2\varepsilon h \ge c^{\prime\prime} > 0$, and:
\begin{flalign*}
\mathbb{E}\left[\left(\bm{x}_{n_i}^T\bm{v}_d^{(n_i)}\right)^2\right] \overset{(a)}\ge \mathbb{E}\left[\left(\left|\bm{x}_{n_i}^T \bm{v}_d^{(n_i)}\right| - \left|\bm{x}_{n_i}^T\bm{\delta} \right|\right)^2\right] \overset{(b)}\ge \mathbb{E}\left[\left|\bm{x}_{n_i}^T \bm{v}_d^{(n_i)}\right| - \left|\bm{x}_{n_i}^T\bm{\delta} \right|\right]^2 \ge (c^{\prime\prime})^2 > 0,
\end{flalign*}
where (a) takes expectations of both sides of \eqref{eqn_v_tn_bound}, and (b) follows from Jensen's inequality. Merging this result with \eqref{eqn:lambda_2} implies that $\lambda_d^{(n_j)}$ is expected to increase at least linearly on average, according to the positive constant $c^{\prime\prime}$, over the indices $(n_i)$. Recall that there always exists an $S_k$ such that the number of times when $\bm{v}_d^{(i)} \in S_k$ is at least linear in the total number of iterations $N$. Thus, the rate at which indices $(n_i)$ occur is always (at least) linear in $N$ on average, and $\lambda_d^{(n_j)}$ increases at least linearly in $N$ in expectation.

We demonstrate that $\underset{i \longrightarrow \infty}{\liminf} \frac{\beta_i(\delta)^2} {\lambda_d^{(i)}} = 0$ holds: the numerator of $\frac{\beta_i(\delta)^2} {\lambda_d^{(i)}}$ is the square of a quantity that increases at most logarithmically in $i$, while the denominator increases at least linearly in $i$ on average. This contradicts the assumption in \eqref{eqn:contradiction_hypothesis}, and so $\underset{i \longrightarrow \infty}\liminf \, \beta_i(\delta)^2 \left(\lambda_d^{(i)}\right)^{-1} = 0$ must hold.

Finally, we leverage that $\underset{i \longrightarrow \infty}\liminf \, \beta_i(\delta)^2 \left(\lambda_d^{(i)}\right)^{-1} = 0$ to show that $\underset{i \longrightarrow \infty}\lim \, \beta_i(\delta)^2 \left(\lambda_d^{(i)}\right)^{-1} = 0$. Consider the following two possible cases: 1) $\beta_i(\delta)^2 \left(\lambda_d^{(i)}\right)^{-1}$ converges to zero in probability, and 2) $\beta_i(\delta)^2 \left(\lambda_d^{(i)}\right)^{-1}$ \textit{does not} converge to zero in probability. In case 1), because convergence in probability implies convergence in distribution, the desired result holds.

In case 2), there exists some $\varepsilon > 0$ such that $P\left(\beta_i(\delta)^2 \left(\lambda_d^{(i)}\right)^{-1} \ge \varepsilon\right) \centernot\longrightarrow 0$. In this case, one can apply the same arguments used to show that $\underset{i \longrightarrow \infty}\liminf \, \beta_i(\delta)^2 \left(\lambda_d^{(i)}\right)^{-1} = 0$, but specifically over time indices where $\beta_i(\delta)^2\left(\lambda_d^{(i)}\right)^{-1} \ge \varepsilon$. Due to the non-convergence in probability, these time indices must occur at some non-decaying rate, and so the same analysis applies. Thus, $\lambda_d^{(i)}$ increases in $i$ with at least a minimum linear average rate, while $\beta_i(\delta)$ increases at most logarithmically in $i$. This violates the non-convergence assumption of case 2), resulting in a contradiction. Therefore, only case 1) can hold.

\end{proof}

From the asymptotic consistency of the dynamics and reward samples, one can show that the sampled policies converge in distribution to the optimal policy:
\begin{theorem}\label{thm:policy_consistency}
With probability $1 - \delta$, the sampled policies $\pi_{i1}, \pi_{i2}$ converge in distribution to the optimal policy, $\pi^*$, as $i \longrightarrow \infty$. That is, $P(\pi_{i1} = \pi^*) \longrightarrow 1$ and $P(\pi_{i2} = \pi^*) \longrightarrow 1$ as $i \longrightarrow \infty$.
\end{theorem}
\begin{proof}
It suffices to show that $P(\pi_{i1} = \pi^*) \longrightarrow 1$  as $i \longrightarrow \infty$, as the proof is identical for $\pi_{i2}$. From Propositions \ref{prop:dyn_consistency} and \ref{prop:reward_consistency}, respectively, we have that $\bm{\tilde{p}}_{i1} \overset{D}\longrightarrow \bm{\overline{p}}$ and that $\bm{\tilde{r}}_{i1} \overset{D}\longrightarrow \bm{\overline{r}}$ with probability $1 - \delta$. We proceed under the assumption that $\bm{\tilde{r}}_{i1} \overset{D}\longrightarrow \bm{\overline{r}}$, i.e., that the probability-$(1 - \delta)$ event occurs.

By Fact \ref{fact:conv_dist_1} in Appendix \ref{sec:appendix_prob_facts}, for each fixed $\pi$, $V(\bm{\tilde{p}}_{i1}, \bm{\tilde{r}}_{i1}, \pi) \overset{D}\longrightarrow V(\bm{\overline{p}}, \bm{\overline{r}}, \pi)$, as value functions are continuous in the dynamics and reward parameters. Applying Fact \ref{fact:conv_dist_2} in Appendix \ref{sec:appendix_prob_facts}, for each fixed $\pi$ and $\varepsilon > 0$:
\begin{equation}\label{eqn:value_conv_prob}
    P(|V(\bm{\tilde{p}}_{i1}, \bm{\tilde{r}}_{i1}, \pi) - V(\bm{\overline{p}}, \bm{\overline{r}}, \pi)| > \varepsilon) \longrightarrow 0 \text{ as } i \longrightarrow \infty.
\end{equation}

Next, we set the value of $\varepsilon$ to be less than half of the smallest gap between the value of the optimal policy and the value of any suboptimal policy:
\begin{equation*}
    \varepsilon < \frac{1}{2} \left[\max_\pi V(\bm{\overline{p}}, \bm{\overline{r}}, \pi) - \underset{\pi \text{ s.t. } V(\bm{\overline{p}}, \bm{\overline{r}}, \pi) < \max_{\pi^\prime}V(\bm{\overline{p}}, \bm{\overline{r}}, \pi^\prime)}{\max} \, V(\bm{\overline{p}}, \bm{\overline{r}}, \pi)\right].
\end{equation*}

Then, the probability of selecting a non-optimal policy can be upper-bounded by a quantity that decays with $i$:
\begin{flalign*}
P(\pi_{i1} \neq \pi^*) &\overset{(a)}\le P\left(\underset{\pi}{\bigcup} \{|V(\bm{\tilde{p}}_{i1}, \bm{\tilde{r}}_{i1}, \pi) - V(\bm{\overline{p}}, \bm{\overline{r}}, \pi)| > \varepsilon \} \right) \\ &\overset{(b)}\le \sum_{\pi} P\left(|V(\bm{\tilde{p}}_{i1}, \bm{\tilde{r}}_{i1}, \pi) - V(\bm{\overline{p}}, \bm{\overline{r}}, \pi)| > \varepsilon \right) \overset{(c)}{\longrightarrow} 0 \text{ as } i \longrightarrow \infty,
\end{flalign*}
where (a) follows from the definition of $\varepsilon$, (b) follows from the union bound, and (c) holds due to \eqref{eqn:value_conv_prob}.

\end{proof}

\subsection{BOUNDING THE ONE-SIDED REGRET WHEN ONE POLICY IS DRAWN FROM A FIXED DISTRIBUTION}\label{sec:bound_one_sided_fixed}

In this portion of the analysis, we assume that in each iteration $i$, policy $\pi_{i1}$ is drawn from a fixed distribution over $\Pi$, the set of deterministic policies. In this setting, we only consider the one-sided regret incurred by the policies $\{\pi_{i2}\}, i \ge 1$. Let $N$ be the total number of iterations of \algo, so that the total number of actions taken by $\{\pi_{i2}\}, i \ge 1$, is $T = Nh$. Let $\bm{x}_i^*$ be the trajectory obtained in iteration $i$ if $\pi_{i2} = \pi^*$. Then, the expected one-sided regret of the policies $\{\pi_{i2}\}$ is the expected loss in total utility due to selecting suboptimal policies:
\begin{flalign*}
\mathbb{E}\left[\text{\reg}_2(T)\right] = \mathbb{E}\left[\text{\reg}_2(Nh)\right] := \mathbb{E}\left[\sum_{i = 1}^N \bm{\overline{r}}^T(\bm{x}^*_i - \bm{x}_{i2})\right].
\end{flalign*}

Recall that the preference outcome on iteration $i$, $y_i \in \left\{-\frac{1}{2}, \frac{1}{2}\right\}$, can be written as:
\begin{equation*}
    y_i = \bm{\overline{r}}^T (\bm{x}_{i2} - \bm{x}_{i1}) + \eta_i,
\end{equation*}
for zero-mean noise $\eta_i$. Define the outcome of selecting policy $\pi_{i2} = \pi^*$ in iteration $i$ as $y_i^* \in \left\{-\frac{1}{2}, \frac{1}{2}\right\}$:
\begin{equation*}
    y^*_i = \bm{\overline{r}}^T (\bm{x}^*_i - \bm{x}_{i1}) + \eta_i^*,
\end{equation*}
for zero-mean noise $\eta_i^*$. Importantly, the difference $y^*_i - y_i$ is equal to the instantaneous one-sided regret in expectation:
\begin{flalign}
\mathbb{E}\left[\text{\reg}_2(T)\right] :&= \mathbb{E}\left[\sum_{i = 1}^N \bm{\overline{r}}^T(\bm{x}^*_i - \bm{x}_{i2})\right] = \mathbb{E}\left[\sum_{i = 1}^N \bigr(\bm{\overline{r}}^T(\bm{x}^*_i - \bm{x}_{i1}) - \bm{\overline{r}}^T(\bm{x}_{i2} - \bm{x}_{i1})\bigr)\right] \label{eqn:xi1_cancel} \\ &\overset{(a)}= \mathbb{E}\left[\sum_{i = 1}^N (y_i^* - y_i) \right], \nonumber
\end{flalign}
where equality (a) holds because the noise terms $\eta_i$ and $\eta_i^*$ are zero-mean.

Define the one-sided algorithm's history at iteration $i$ as $\mathcal{H}_{i}^{(2)} = \{Z_1, Z_2, \ldots, Z_i\}$, where $Z_i = (\pi_{i2}, \tau_{i1}, \tau_{i2}, \bm{x}_{i2} - \bm{x}_{i1}, y_i)$. Analogously to \cite{russo2016information}, we establish notation for probabilities and information-theoretic quantities while conditioning on the history $\mathcal{H}_{i - 1}^{(2)}$. In particular, $P_i(\cdot) := P(\cdot \,|\, \mathcal{H}_{i - 1}^{(2)})$ and $\mathbb{E}_i[\cdot] := \mathbb{E}[\cdot \,|\, \mathcal{H}_{i - 1}^{(2)}]$. With respect to the history, the entropy of a random variable $X$ is $H_i(X) := -\sum_{x} P_i(X = x)\log P_i(X = x)$, while two random variables $X$ and $Y$ have mutual information $I_i(X; Y) := H_i(X) - H_i(X | Y)$. Lastly, $D(P || Q)$ is the Kullback-Leibler divergence between two discrete probability distributions $P$ and $Q$.

The \textit{information ratio} is then defined as:
\begin{equation}\label{eqn:iratio}
    \Gamma_i := \frac{\mathbb{E}_i[y_i^* - y_i]^2}{I_i(\pi^*; (\pi_{i2}, \tau_{i1}, \tau_{i2}, \bm{x}_{i2} - \bm{x}_{i1}, y_i))}.
\end{equation}
This definition is analogous to the information ratio defined in \cite{russo2016information}, but while \cite{russo2016information} study the bandit setting with absolute feedback, our definition is adapted to the preference-based RL setting. Note that the numerator is the square of the expected instantaneous one-sided regret, while the denominator is the information gained about the optimal policy in iteration $i$.

In \cite{russo2016information}, the authors express the expected Bayesian regret of Thompson sampling in the bandit setting in terms of the information ratio (Proposition 1, \cite{russo2016information}). The following lemma adapts this result to the PBRL setting:

\begin{lemma}\label{lemma:regret_inf_ratio}
If there exists an upper bound $\overline{\Gamma}$ such that $\Gamma_i \le \overline{\Gamma}$ almost surely for each $i \in \{1, \ldots, N\}$, where $N$ is the number of \algo~iterations (i.e., pairs of trajectories), then the policies $\{\pi_{i2}\}$ take $T = Nh$ total actions, and:
\begin{equation*}
    \mathbb{E}[\reg_2(T)] = \mathbb{E}[\reg_2(Nh)] \le \sqrt{\overline{\Gamma} H(\pi^*) N},
\end{equation*}
where $H(\pi^*)$ is the entropy of the optimal policy $\pi^*$. Note that since there are at most $A^{Sh}$ possible deterministic policies, $H(\pi^*) \le \log |A^{Sh}| = Sh \log A$. Substituting this,
\begin{equation*}
    \mathbb{E}[\reg_2(T)] \le \sqrt{\overline{\Gamma}Sh N \log A} = \sqrt{\overline{\Gamma}ST \log A}.
\end{equation*}
\end{lemma}

\begin{proof}
\begin{flalign}
\mathbb{E}[\reg_2(T)] &= \sum_{i = 1}^N \mathbb{E}[y_i^* - y_i] \overset{(a)}{=} \sum_{i = 1}^N \mathbb{E}_{\mathcal{H}_{i - 1}^{(2)}}\mathbb{E}_i\left[y_i^* - y_i \right] = \sum_{i = 1}^N \mathbb{E}_{\mathcal{H}_{i - 1}^{(2)}}\sqrt{\Gamma_i I_i(\pi^*; (\pi_{i2}, \tau_{i1}, \tau_{i2}, \bm{x}_{i2} - \bm{x}_{i1}, y_i))} \nonumber \\ &\le \sqrt{\overline{\Gamma}} \sum_{i = 1}^N \mathbb{E}_{\mathcal{H}_{i - 1}^{(2)}}\sqrt{I_i(\pi^*; (\pi_{i2}, \tau_{i1}, \tau_{i2}, \bm{x}_{i2} - \bm{x}_{i1}, y_i))} \nonumber \\ &\overset{(b)}{\le} \sqrt{\overline{\Gamma}N \sum_{i = 1}^N \left[\mathbb{E}_{\mathcal{H}_{i - 1}^{(2)}}\sqrt{I_i(\pi^*; (\pi_{i2}, \tau_{i1}, \tau_{i2}, \bm{x}_{i2} - \bm{x}_{i1}, y_i))}\right]^2} \nonumber \\ &\overset{(c)}{\le} \sqrt{\overline{\Gamma}N \sum_{i = 1}^N \mathbb{E}_{\mathcal{H}_{i - 1}^{(2)}}I_i(\pi^*; (\pi_{i2}, \tau_{i1}, \tau_{i2}, \bm{x}_{i2} - \bm{x}_{i1}, y_i))}, \label{eqn:lemma1}
\end{flalign}
where (a) results from the tower property of conditional expectation, (b) follows from the Cauchy-Schwarz inequality, and (c) follows from Jensen's inequality. It remains to upper-bound the summation in \eqref{eqn:lemma1} by $H(\pi^*)$. Defining $Z_i := (\pi_{i2}, \tau_{i1}, \tau_{i2}, \bm{x}_{i2} - \bm{x}_{i1}, y_i)$, the summation terms are equal to $\mathbb{E}_{\mathcal{H}_{i - 1}^{(2)}}I_i(\pi^*; Z_i) = I(\pi^*; Z_i | Z_1, \ldots, Z_{i - 1})$, where the last equality comes from applying the definitions of $I_i$ and of conditional mutual information. Therefore,
\begin{flalign*}
\sum_{i = 1}^N \mathbb{E}_{\mathcal{H}_{i - 1}^{(2)}}I_i(\pi^*; (\pi_{i2}, \tau_{i1}, \tau_{i2}, \bm{x}_{i2} - \bm{x}_{i1}, y_i)) &= \sum_{i = 1}^N I(\pi^*; Z_i | Z_1, \ldots, Z_{i - 1}) \overset{(a)}{=} I(\pi^*; Z_1, \ldots, Z_i) \\ &\overset{(b)}{=} H(\pi^*) - H(\pi^* | Z_1, \ldots, Z_i) \overset{(c)}{\le} H(\pi^*),
\end{flalign*}
where (a) results from the chain rule for mutual information, (b) is a standard identity resulting from the definitions of entropy and mutual information \citep{cover2012elements}, and (c) follows from the non-negativity of entropy.
\end{proof}

\begin{coro}\label{coro:one_sided_bound_in_terms_of_Gamma}
If $\Gamma_i \le \overline{\Gamma}$ almost surely for each $i \in \{i_0, \ldots, N\}$ for some index $i_0 \ge 1$, then \algo~achieves an asymptotic one-sided regret rate of $\sqrt{\overline{\Gamma}Sh N \log A} = \sqrt{\overline{\Gamma}ST \log A}$.
\end{coro}

\begin{proof}
The one-sided regret can be decomposed into two sums:
\begin{flalign*}
\mathbb{E}[\reg_2(T)] = \sum_{i = 1}^N \mathbb{E}[y_i^* - y_i] = \sum_{i = 1}^{i_0 - 1} \mathbb{E}[y_i^* - y_i] + \sum_{i = i_0}^N \mathbb{E}[y_i^* - y_i].
\end{flalign*}
The second term can be upper-bounded via the same arguments used to prove Lemma \ref{lemma:regret_inf_ratio}. The first term is a constant for iteration $i \ge i_0$, as it does not depend on the performance of the algorithm for iterations $i \ge i_0$.
\end{proof}

Next, we turn to upper-bounding the information ratio. First, Lemma \ref{lemma:TS} demonstrates in Thompson sampling, the probability of selecting a policy is equal to that policy's posterior probability of being optimal.

\begin{lemma}\label{lemma:TS}
When $\pi_{i2}$ is selected via Thompson sampling---i.e., by drawing respective samples $\bm{\tilde{p}}$ and $\bm{\tilde{r}}$ from the dynamics and reward posteriors, and performing value iteration to obtain $\pi_{i2} = \pi_{vi}(\bm{\tilde{p}}, \bm{\tilde{r}})$---then for each deterministic policy $\pi$, $P_i(\pi_{i2} = \pi) = P_i(\pi^* = \pi)$.
\end{lemma}
\begin{proof}
In Thompson sampling, the MDP parameters are sampled from the model posterior, that is, according to their posterior probability of being the true MDP parameters. Let $\bm{\overline{m}} = [\bm{\overline{r}}^T, \bm{\overline{p}}^T]^T \in \mathbb{R}^{S^2 A + SA}$ be the vector of true MDP parameters, where $\bm{\overline{r}} \in \mathbb{R}^{SA}$ is the vector of reward parameters and $\bm{\overline{p}} \in \mathbb{R}^{S^2 A}$ is the vector of dynamics parameters. Let $p_i(\bm{m})$ be the posterior probability density of $\bm{\overline{m}}$, given the history; then, Thompson sampling samples parameter vectors $\bm{m} \in \mathbb{R}^{S^2 A + SA}$ according to $p_i(\bm{m})$. Finally, for any deterministic policy $\pi$, let $m(\pi) \subset \mathbb{R}^{S^2 A + SA}$ give the set of all MDP parameters for which value iteration yields the policy $\pi$. Then:
\begin{flalign*}
P_i(\pi_{i2} = \pi) \overset{(a)}{=} \int_{m(\pi)}p_i(\bm{m}) d\bm{m} \overset{(b)}{=} P_i(\pi^* = \pi),
\end{flalign*}
where (a) holds because Thompson sampling selects policies by sampling $\bm{m} \sim p_i(\cdot)$ and then applying value iteration to $\bm{m}$, and (b) follows from integrating over the posterior probability of all MDP parameter vectors resulting in $\pi$.

\end{proof}

To upper-bound the information ratio, we next express its numerator and denominator using the optimal policy's distribution via the following lemma (analogous to Proposition 2 in \cite{russo2016information}):

\begin{lemma}\label{lemma:inf_ratio} Recall that $\Pi$ is the set of deterministic policies. The following two statements hold:
\begin{equation*}
    \mathbb{E}_i \left[y_i^* - y_i \right] = \sum_{\pi \in \Pi} P_i(\pi^* = \pi) \big\{\mathbb{E}_i[y_i \,|\, \pi^* = \pi_{i2} = \pi] - \mathbb{E}_i[y_i \,|\, \pi_{i2} = \pi]  \big\}
\end{equation*}
and
\begin{equation*}
    I_i(\pi^*; (\pi_{i2}, \tau_{i1}, \tau_{i2}, \bm{x}_{i2} - \bm{x}_{i1}, y_i)) \ge \sum_{\pi, \pi^\prime \in \Pi} P_i(\pi^* = \pi)P_i(\pi^* = \pi^\prime) D\big(P_i(y_i \,|\, \pi_{i2} = \pi, \pi^* = \pi^\prime) \,||\, P_i(y_i \,|\, \pi_{i2} = \pi)\big),
\end{equation*}
where $D(p || q)$ is the Kullback-Leibler divergence between discrete probability distributions $p$ and $q$.
\end{lemma}

\begin{proof}
The numerator of the information ratio can be written as:
\begin{flalign*}
\mathbb{E}_i \left[y_i^* - y_i \right] &= \sum_{\pi \in \Pi}P_i(\pi^* = \pi)\mathbb{E}_i [y_i^* \,|\, \pi^* = \pi] - \sum_{\pi \in \Pi} P_i(\pi_{i2} = \pi)\mathbb{E}_i [y_i \,|\, \pi_{i2} = \pi] \\ &\overset{(a)}{=} \sum_{\pi \in \Pi}P_i(\pi^* = \pi)\big(\mathbb{E}_i [y_i^* \,|\, \pi^* = \pi] - \mathbb{E}_i [y_i \,|\, \pi_{i2} = \pi]\big) \\ &= \sum_{\pi \in \Pi} P_i(\pi^* = \pi) \big\{\mathbb{E}_i[y_i \,|\, \pi^* = \pi_{i2} = \pi] - \mathbb{E}_i[y_i \,|\, \pi_{i2} = \pi]  \big\},
\end{flalign*}
where (a) follows from Lemma \ref{lemma:TS}. The denominator of the information ratio can meanwhile be lower-bounded:
\begin{flalign*}
I_i(\pi^*; (\pi_{i2}, \tau_{i1}, \tau_{i2}, \bm{x}_{i2} - \bm{x}_{i1}, y_i)) &\overset{(a)}{=} I_i(\pi^*; \pi_{i2}) + I_i(\pi^*; y_i \,|\, \pi_{i2}) + I_i(\pi^*; (\bm{x}_{i2} - \bm{x}_{i1}, \tau_{i1}, \tau_{i2}) \,|\,  \pi_{i2}, y_i) \\ &\overset{(b)}{\ge} I_i(\pi^*; y_i \,|\, \pi_{i2}) = \sum_{\pi \in \Pi} P_i(\pi_{i2} = \pi) I_i(\pi^*; y_i \,|\, \pi_{i2} = \pi) \\ &\overset{(c)}{=} \sum_{\pi \in \Pi} P_i(\pi^* = \pi) I_i(\pi^*; y_i \,|\, \pi_{i2} = \pi),
\end{flalign*}
where (a) follows from the chain rule for mutual information, (b) results from the non-negativity of mutual information, and (c) is a consequence of Lemma \ref{lemma:TS}. We next apply the following information-theoretic identity (Fact 6 in \cite{russo2016information}), which holds for discrete random variables $X$ and $Y$:
\begin{equation*}
    I(X; Y) = \sum_x P(X = x)D(P(Y \,|\, X = x) \,||\, P(Y)).
\end{equation*}
From this, one obtains:
\begin{flalign*}
    I_i(\pi^*; y_i \,|\, \pi_{i2} = \pi) &= \sum_{\pi^\prime \in \Pi} P_i(\pi^* = \pi^\prime  \,|\, \pi_{i2} = \pi)D(P_i(y_i \,|\,  \pi^* = \pi^\prime, \pi_{i2} = \pi) \,||\, P_i(y_i \,|\, \pi_{i2} = \pi)) \\ &\overset{(a)}= \sum_{\pi^\prime \in \Pi} P_i(\pi^* = \pi^\prime)D(P_i(y_i \,|\,  \pi^* = \pi^\prime, \pi_{i2} = \pi) \,||\, P_i(y_i \,|\, \pi_{i2} = \pi)),
\end{flalign*}
where (a) holds because by definition of Thompson sampling, $\pi^*$ and $\pi_{i2}$ are independent given the history. Therefore:
\begin{flalign*}
I_i(\pi^*; (\pi_{i2}, \tau_{i1}, \tau_{i2}, \bm{x}_{i2} - \bm{x}_{i1}, y_i)) &\ge \sum_{\pi \in \Pi} P_i(\pi^* = \pi) \sum_{\pi^\prime \in \Pi} P_i(\pi^* = \pi^\prime)D(P_i(y_i \,|\,  \pi^* = \pi^\prime, \pi_{i2} = \pi) \,||\, P_i(y_i \,|\, \pi_{i2} = \pi))  \\ &= \sum_{\pi, \pi^\prime \in \Pi} P_i(\pi^* = \pi)P_i(\pi^* = \pi^\prime) D\big(P_i(y_i \,|\, \pi_{i2} = \pi, \pi^* = \pi^\prime) \,||\, P_i(y_i \,|\, \pi_{i2} = \pi)\big),
\end{flalign*}
which is the desired result.

\end{proof}

The next lemma asymptotically upper-bounds the information ratio $\Gamma_i$ for the one-sided regret. It is inspired by the analysis in \cite{russo2016information} for the linear bandit setting (Proposition 5 in \cite{russo2016information}); however, extending this result to the PBRL setting requires accounting for the dynamics, which complicates the analysis.

\begin{lemma}\label{lemma:bounded_inf_ratio}
Consider the one-sided regret when $\pi_{i1}$ is drawn from a fixed distribution. The information ratio $\Gamma_i$ satisfies:
\begin{equation*}
    \underset{i \longrightarrow \infty}{\lim} \, \Gamma_i \le \frac{d}{2} = \frac{SA}{2}.
\end{equation*}
Thus, for any $\varepsilon > 0$ and a sufficiently-large iteration index $i$, $\Gamma_i \le \frac{SA}{2} + \varepsilon$.
\end{lemma}

\begin{proof}
Let $K := |\Pi|$. Index all the deterministic policies as $\pi_1, \ldots, \pi_K$, and define $B^{(i)} \in \mathbb{R}^{K \times K}$ with $jk$\textsuperscript{th} element:
\begin{equation}\label{eqn:B_matrix}
    B_{jk}^{(i)} = \sqrt{P_i(\pi^* = \pi_j)P_i(\pi^* = \pi_k)} \Big(\mathbb{E}_i[y_i \,|\, \pi_{i2} = \pi_j, \pi^* = \pi_k] -  \mathbb{E}_i[y_i \,|\, \pi_{i2} = \pi_j] \Big).
\end{equation}

The numerator and denominator of $\Gamma_i$ can both be expressed in terms of $B^{(i)}$. Applying Lemma \ref{lemma:inf_ratio} to the numerator, the instantaneous regret can be written as follows:
\begin{flalign*}
\mathbb{E}_i \left[y_i^* - y_i \right] &= \sum_{\pi \in \Pi} P_i(\pi^* = \pi) \big\{\mathbb{E}_i[y_i \,|\, \pi^* = \pi_{i2} = \pi] - \mathbb{E}_i[y_i \,|\, \pi_{i2} = \pi]  \big\} \\ &= \sum_{j = 1}^K P_i(\pi^* = \pi_j) \big\{\mathbb{E}_i[y_i \,|\, \pi^* = \pi_{i2} = \pi_j] - \mathbb{E}_i[y_i \,|\, \pi_{i2} = \pi_j]  \big\} \\ &= \sum_{j = 1}^K B_{jj}^{(i)} = \text{Tr}\left(B^{(i)}\right).
\end{flalign*}

Applying Lemma \ref{lemma:inf_ratio} to the denominator of $\Gamma_i$ in Equation \eqref{eqn:iratio} yields:
\begin{flalign*}
I_i(\pi^*; (\pi_{i2}, \tau_{i1}, \tau_{i2}, \bm{x}_{i2} - \bm{x}_{i1}, y_i)) &\ge \sum_{\pi, \pi^\prime \in \Pi} P_i(\pi^* = \pi)P_i(\pi^* = \pi^\prime) D\big(P_i(y_i \,|\, \pi_{i2} = \pi, \pi^* = \pi^\prime) \,||\, P_i(y_i \,|\, \pi_{i2} = \pi)\big) \\ &= \sum_{j, k = 1}^K P_i(\pi^* = \pi_j)P_i(\pi^* = \pi_k) D\big(P_i(y_i \,|\, \pi_{i2} = \pi_j, \pi^* = \pi_k) \,||\, P_i(y_i \,|\, \pi_{i2} = \pi_j)\big).
\end{flalign*}

We convert the Kullback-Leibler divergence to a difference of expectations by applying Fact 9 from \cite{russo2016information}, restated here: for any probability distributions $P$ and $Q$ such that $P$ is absolutely continuous with respect to $Q$, any random variable $X$ taking values on the set $\mathcal{X}$, and any $g: \mathcal{X} \longrightarrow \mathbb{R}$ such that $\sup g - \inf g \le 1$,
\begin{equation}\label{eqn:fact_9}
    D(P \,||\, Q) \ge 2 \bigr(\mathbb{E}_P[g(X)] - \mathbb{E}_Q[g(X)]\bigr)^2.
\end{equation}

Thus, we have that:
\begin{equation*}
    D\big(P_i(y_i \,|\, \pi_{i2} = \pi_j, \pi^* = \pi_k) \,||\, P_i(y_i \,|\, \pi_{i2} = \pi_j)\big) \ge 2 \Big(\mathbb{E}_i[y_i \,|\, \pi_{i2} = \pi_j, \pi^* = \pi_k] -  \mathbb{E}_i[y_i \,|\, \pi_{i2} = \pi_j] \Big)^2,
\end{equation*}
where we applied \eqref{eqn:fact_9} with $g(x) = x$; this definition of $g$ satisfies the requirement that $\sup g - \inf g \le 1$, since its argument is $y_i \in \left\{-\frac{1}{2}, \frac{1}{2}\right\}$. As a result:
\begin{flalign*}
    I_i(\pi^*; (\pi_{i2}, \tau_{i1}, \tau_{i2}, \bm{x}_{i2} - \bm{x}_{i1}, y_i)) &\ge 2\sum_{j, k = 1}^K P_i(\pi^* = \pi_j)P_i(\pi^* = \pi_k) \Big(\mathbb{E}_i[y_i \,|\, \pi_{i2} = \pi_j, \pi^* = \pi_k] -  \mathbb{E}_i[y_i \,|\, \pi_{i2} = \pi_j] \Big)^2 \\ &= 2\sum_{j, k = 1}^K \left(B_{jk}^{(i)}\right)^2 = 2||B^{(i)}||_F^2.
\end{flalign*}    
    
Combining these results gives that $\Gamma_i \le \frac{\text{Tr}\left(B^{(i)}\right)^2}{2||B^{(i)}||_F^2}$. As shown in \cite{russo2016information} (Fact 10), for any square matrix $B \in \mathbb{R}^{m \times m}$, $\text{Tr}(B) \le \sqrt{\text{Rank}(B)}||B||_F$. Thus:
\begin{equation*}
\Gamma_i \le \frac{\text{Tr}\left(B^{(i)}\right)^2}{2||B^{(i)}||_F^2} \le \frac{1}{2} \text{Rank}\left(B^{(i)}\right).    
\end{equation*}

The problem is therefore reduced to upper-bounding $\text{Rank}\left(B^{(i)}\right)$. First, we will show that under known transition dynamics, $\text{Rank}\left(B^{(i)}\right) \le d = SA$. Subsequently, we will demonstrate that as the sampled dynamics parameters converge in distribution to their true values (which occurs by Proposition \ref{prop:dyn_consistency}), $\underset{i \longrightarrow \infty}{\mathrm{lim}} \, \text{Rank}\left(B^{(i)}\right) \le d$.

Recall the definition of $B^{(i)}$ in \eqref{eqn:B_matrix}. We will first show that under known transition dynamics, we can define a set of vectors $\bm{u}_1, \ldots, \bm{u}_K, \bm{v}_1, \ldots, \bm{v}_K \in \mathbb{R}^d$, such that $B_{jk}^{(i)} = \bm{u}_k^T \bm{v}_j$. Recalling that $y_i = \bm{\overline{r}}^T (\bm{x}_{i2} - \bm{x}_{i1}) + \eta_i$, and assuming that the dynamics $\bm{\overline{p}}$ are known,
\begin{flalign}
\mathbb{E}_i[y_i \,|\, \pi_{i2} = \pi_j, \pi^* = \pi_k, \bm{\overline{p}}] &= \mathbb{E}_i[\bm{\overline{r}}^T (\bm{x}_{i2} - \bm{x}_{i1}) + \eta_i \,|\, \pi_{i2} = \pi_j, \pi^* = \pi_k, \bm{\overline{p}}] \nonumber \\ &\overset{(a)}{=} \mathbb{E}_i[\bm{\overline{r}}^T (\bm{x}_{i2} - \bm{x}_{i1}) \,|\, \pi_{i2} = \pi_j, \pi^* = \pi_k, \bm{\overline{p}}]  \nonumber \\ &\overset{(b)}{=} \mathbb{E}_i[\bm{\overline{r}} \,|\, \pi_{i2} = \pi_j, \pi^* = \pi_k, \bm{\overline{p}}]^T (\mathbb{E}_i[\bm{x}_{i2} \,|\, \pi_{i2} = \pi_j, \bm{\overline{p}}] - \mathbb{E}_i[\bm{x}_{i1}]) \label{eqn:assume_dynamics}\\ &\overset{(c)}{=} \mathbb{E}_i[\bm{\overline{r}} \,|\, \pi^* = \pi_k, \bm{\overline{p}}]^T (\mathbb{E}_i[\bm{x}_{i2} \,|\, \pi_{i2} = \pi_j, \bm{\overline{p}}] - \mathbb{E}_i[\bm{x}_{i1}]),  \nonumber 
\end{flalign}
where equality (a) holds because the noise $\eta_i$ is zero-mean.  Equality (b) holds because, by assumption, $\bm{x}_{i1}$ is drawn from a fixed distribution, independently of $\pi^*$, $\pi_{i2}$, or $\bm{\overline{r}}$, while given $\pi_{i2}$ and known dynamics, the distribution of $\bm{x}_{i2}$ is fully-determined and independent of $\pi^*$ and $\bm{\overline{r}}$.  Equality (c) holds because, conditioned upon $\mathcal{H}_{i - 1}^{(2)}$, $\bm{\rr}$ and $\pi_{i2}$ are independent (as the history fully determines the distribution of $\pi_{i2}$). By similar arguments,
\begin{flalign}
\mathbb{E}_i[y_i \,|\, \pi_{i2} = \pi_j, \bm{\overline{p}}] &= \mathbb{E}_i[\bm{\overline{r}}^T (\bm{x}_{i2} - \bm{x}_{i1}) + \eta_i \,|\, \pi_{i2} = \pi_j, \bm{\overline{p}}] = \mathbb{E}_i[\bm{\overline{r}}^T (\bm{x}_{i2} - \bm{x}_{i1}) \,|\, \pi_{i2} = \pi_j, \bm{\overline{p}}] \nonumber \\ &= \mathbb{E}_i[\bm{\overline{r}} \,|\, \pi_{i2} = \pi_j, \bm{\overline{p}}]^T (\mathbb{E}_i[\bm{x}_{i2} \,|\, \pi_{i2} = \pi_j, \bm{\overline{p}}] - \mathbb{E}_i[\bm{x}_{i1}]) \label{eqn:assume_dynamics_2} \\ &= \mathbb{E}_i[\bm{\overline{r}} \,|\, \bm{\overline{p}}]^T (\mathbb{E}_i[\bm{x}_{i2} \,|\, \pi_{i2} = \pi_j, \bm{\overline{p}}] - \mathbb{E}_i[\bm{x}_{i1}]). \nonumber
\end{flalign}
Subtracting the latter two quantities yields:
\begin{flalign}\label{eqn:exp_diff}
\mathbb{E}_i[y_i \,|\, \pi_{i2} = \pi_j, \pi^* = \pi_k, \bm{\overline{p}}] &-  \mathbb{E}_i[y_i \,|\, \pi_{i2} = \pi_j, \bm{\overline{p}}] \nonumber \\ &= (\mathbb{E}_i[\bm{\overline{r}} \,|\, \pi^* = \pi_k, \bm{\overline{p}}] - \mathbb{E}_i[\bm{\overline{r}} \,|\, \bm{\overline{p}}])^T (\mathbb{E}_i[\bm{x}_{i2} \,|\, \pi_{i2} = \pi_j, \bm{\overline{p}}] - \mathbb{E}_i[\bm{x}_{i1}]).
\end{flalign}

Applying this result, under the case of known dynamics, $B_{jk}^{(i)}$ can be written as:
\begin{flalign*}
B_{jk}^{(i)} &= \sqrt{P_i(\pi^* = \pi_j \,|\, \bm{\overline{p}})P_i(\pi^* = \pi_k  \,|\, \bm{\overline{p}})} \Big(\mathbb{E}_i[y_i \,|\, \pi_{i2} = \pi_j, \pi^* = \pi_k, \bm{\overline{p}}] -  \mathbb{E}_i[y_i \,|\, \pi_{i2} = \pi_j, \bm{\overline{p}}] \Big) \\ &= \sqrt{P_i(\pi^* = \pi_k \,|\, \bm{\overline{p}})} (\mathbb{E}_i[\bm{\overline{r}} \,|\, \pi^* = \pi_k, \bm{\overline{p}}] - \mathbb{E}_i[\bm{\overline{r}} \,|\, \bm{\overline{p}}])^T \sqrt{P_i(\pi^* = \pi_j \,|\, \bm{\overline{p}})}(\mathbb{E}_i[\bm{x}_{i2} \,|\, \pi_{i2} = \pi_j, \bm{\overline{p}}] - \mathbb{E}_i[\bm{x}_{i1}]) \\ &:= \bm{u}_k^T \bm{v}_j,
\end{flalign*}
where in the last equality, we define $\bm{u}_k := \sqrt{P_i(\pi^* = \pi_k \,|\, \bm{\overline{p}})} (\mathbb{E}_i[\bm{\overline{r}} \,|\, \pi^* = \pi_k, \bm{\overline{p}}] - \mathbb{E}_i[\bm{\overline{r}} \,|\, \bm{\overline{p}}])$ and $\bm{v}_j := \sqrt{P_i(\pi^* = \pi_j \,|\, \bm{\overline{p}})}$ $(\mathbb{E}_i[\bm{x}_{i2} \,|\, \pi_{i2} = \pi_j, \bm{\overline{p}}] - \mathbb{E}_i[\bm{x}_{i1}])$. Therefore:
\begin{equation}\label{eqn:matrix_prod}
B^{(i)} = \begin{bmatrix}
\bm{u}_1^T \bm{v}_1 & \ldots & \bm{u}_K^T \bm{v}_1 \\
\vdots & \ddots & \vdots \\
\bm{u}_1^T \bm{v}_K & \ldots & \bm{u}_K^T \bm{v}_K
\end{bmatrix} = \begin{bmatrix} \bm{v}_1^T \\ \vdots \\ \bm{v}_K^T\end{bmatrix}\begin{bmatrix} \bm{u}_1 & \ldots & \bm{u}_K\end{bmatrix}.
\end{equation}

Because $B^{(i)}$ can be written as the product of a $K \times d$ matrix and a $d \times K$ matrix, $B^{(i)}$ can have rank at most $d$. To reach this result, however, we assumed that the MDP transition dynamics are known. To complete the proof, we now show that the result still holds asymptotically as the sampled dynamics converge in distribution to the true dynamics (Proposition \ref{prop:dyn_consistency} guarantees that this convergence occurs). Note that we only used our assumed knowledge of the dynamics to arrive at the equalities in lines \eqref{eqn:assume_dynamics} and \eqref{eqn:assume_dynamics_2}. In both cases, knowledge of $\bm{\overline{p}}$ is used to treat $\bm{x}_{i2}$ and $\bm{\rr}$ as conditionally independent given $\pi_{i2}$. In particular, \eqref{eqn:assume_dynamics} uses that:
\begin{equation*}
    \mathbb{E}_i[\bm{\overline{r}}^T\bm{x}_{i2} \,|\, \pi_{i2} = \pi_j, \pi^* = \pi_k, \bm{\overline{p}}] = \mathbb{E}_i[\bm{\overline{r}} \,|\, \pi_{i2} = \pi_j, \pi^* = \pi_k, \bm{\overline{p}}]^T \mathbb{E}_i[\bm{x}_{i2} \,|\, \pi_{i2} = \pi_j, \bm{\overline{p}}].
\end{equation*}
We show that in general, when the transition dynamics $\bm{\overline{p}}$ are not known but learned, that:
\begin{equation}\label{eqn:exp_conv_goal}
    \mathbb{E}_i[\bm{\overline{r}}^T\bm{x}_{i2} \,|\, \pi_{i2} = \pi_j, \pi^* = \pi_k] \,\overset{i \longrightarrow \infty}{\longrightarrow}\, \mathbb{E}_i[\bm{\overline{r}} \,|\, \pi_{i2} = \pi_j, \pi^* = \pi_k]^T \mathbb{E}_i[\bm{x}_{i2} \,|\, \pi_{i2} = \pi_j, \bm{\overline{p}}].
\end{equation}

Let $p_i(\cdot)$ be the posterior probability density of the transition dynamics parameters; then, $p_i(\cdot)$ is also the density with which the dynamics $\bm{\tilde{p}}$ are sampled at iteration $i$. Because each dynamics parameter converges in distribution to its true value---and there are finitely-many dynamics parameters---the distribution $p_i(\cdot)$ converges uniformly to $\delta(\bm{\overline{p}} = \cdot)$, where $\delta(\cdot)$ denotes the Dirac-delta distribution. Therefore,
\begin{flalign}\label{eqn:prob_conv}
P_i(\bm{x}_{i2} = \cdot \,|\, \pi_{i2} = \pi_j, \pi^* = \pi_k, \bm{\overline{r}}) &\overset{(a)}{=} \int_{\bm{\tilde{p}}} P_i(\bm{x}_{i2} = \cdot \,|\, \pi_{i2} = \pi_j, \pi^* = \pi_k, \, \bm{\overline{r}}, \, \bm{\overline{p}} = \bm{\tilde{p}})p_i(\bm{\tilde{p}})d\bm{\tilde{p}} \\ &\overset{i \longrightarrow \infty}{\longrightarrow} \, \int_{\bm{\tilde{p}}} P_i(\bm{x}_{i2} = \cdot \,|\, \pi_{i2} = \pi_j, \pi^* = \pi_k, \, \bm{\overline{r}}, \, \bm{\overline{p}} = \bm{\tilde{p}})\delta(\bm{\tilde{p}} = \bm{\overline{p}})d\bm{\tilde{p}} \nonumber \\ &\overset{(b)}{=} P_i(\bm{x}_{i2} = \cdot \,|\, \pi_{i2} = \pi_j, \pi^* = \pi_k, \, \bm{\overline{r}}, \, \bm{\overline{p}}) \overset{(c)}{=} P_i(\bm{x}_{i2} = \cdot \,|\, \pi_{i2} = \pi_j, \, \bm{\overline{p}}), \nonumber
\end{flalign}
where (a) integrates over the posterior probability density of each possible dynamics parameter vector $\bm{\tilde{p}}$, (b) utilizes the sifting property of the Dirac-delta function, and (c) follows because the distribution of $\bm{x}_{i2}$ is fully determined given the dynamics and policy $\pi_{i2}$. For any discrete random variables $X$ and $X_n, n \ge 1$, defined over a finite set $\mathcal{X}$, convergence in distribution $X_n \overset{D}\longrightarrow X$ is equivalent to $\lim_{n \longrightarrow \infty} P(X_n = x) = P(X = x)$ for each $x \in \mathcal{X}$. Since $\mathcal{X}$ is a finite set, one also has convergence in expectation:
\begin{equation}\label{eqn:exp_conv_general}
    \mathbb{E}[X_n] = \sum_{x \in \mathcal{X}}x P(X_n = x) \, \overset{n \longrightarrow \infty}{\longrightarrow} \, \sum_{x \in \mathcal{X}} xP(X = x) = \mathbb{E}[X].
\end{equation}

Combining \eqref{eqn:prob_conv} and \eqref{eqn:exp_conv_general} yields the result,
\begin{equation}\label{eqn:exp_conv}
    \mathbb{E}_i[\bm{x}_{i2} \,|\, \pi_{i2} = \pi_j, \pi^* = \pi_k, \bm{\overline{r}}] \, \overset{i \longrightarrow \infty}{\longrightarrow} \, \mathbb{E}_i[\bm{x}_{i2} \,|\, \pi_{i2} = \pi_j, \bm{\overline{p}}].
\end{equation}

In consequence,
\begin{flalign}
\mathbb{E}_i[\bm{\overline{r}}^T\bm{x}_{i2} \,|\, \pi_{i2} = \pi_j, \pi^* = \pi_k] &\overset{(a)}{=} \mathbb{E}_{i, \bm{\overline{r}}}[\mathbb{E}_i[\bm{\overline{r}}^T\bm{x}_{i2} \,|\, \pi_{i2} = \pi_j, \pi^* = \pi_k, \bm{\overline{r}}] \,|\, \pi_{i2} = \pi_j, \pi^* = \pi_k] \nonumber \\ &= \mathbb{E}_{i, \bm{\overline{r}}}[\bm{\overline{r}}^T\mathbb{E}_i[\bm{x}_{i2} \,|\, \pi_{i2} = \pi_j, \pi^* = \pi_k, \bm{\overline{r}}] \,|\, \pi_{i2} = \pi_j, \pi^* = \pi_k] \nonumber \\ &\overset{(b)}{\longrightarrow} \mathbb{E}_i[\bm{\overline{r}} \,|\, \pi_{i2} = \pi_j, \pi^* = \pi_k]^T \mathbb{E}_i[\bm{x}_{i2} \,|\, \pi_{i2} = \pi_j, \bm{\overline{p}}] \text{ as } i \longrightarrow \infty, \label{eqn:exp_conv_2}
\end{flalign}
where (a) follows from the tower property of expectation, and (b) is an application of \eqref{eqn:exp_conv}. This proves the desired statement, \eqref{eqn:exp_conv_goal}. Repeating the same analysis as in \eqref{eqn:exp_conv_goal}-\eqref{eqn:exp_conv_2}, but removing the conditioning on $\pi^* = \pi_k$ yields:
\begin{flalign}
\mathbb{E}_i[\bm{\overline{r}}^T\bm{x}_{i2} \,|\, \pi_{i2} = \pi_j] \longrightarrow \mathbb{E}_i[\bm{\overline{r}} \,|\, \pi_{i2} = \pi_j]^T \mathbb{E}_i[\bm{x}_{i2} \,|\, \pi_{i2} = \pi_j, \bm{\overline{p}}] \text{ as } i \longrightarrow \infty. \label{eqn:exp_conv_3}
\end{flalign}

The analysis in \eqref{eqn:assume_dynamics}-\eqref{eqn:exp_diff} can be repeated, replacing the knowledge of $\bm{\overline{p}}$ with the asymptotic relations in \eqref{eqn:exp_conv_2} and \eqref{eqn:exp_conv_3}:
\begin{flalign*}
\mathbb{E}_i[y_i \,|\, \pi_{i2} = \pi_j, \pi^* = \pi_k] &= \mathbb{E}_i[\bm{\overline{r}}^T (\bm{x}_{i2} - \bm{x}_{i1}) + \eta_i \,|\, \pi_{i2} = \pi_j, \pi^* = \pi_k] \\ &= \mathbb{E}_i[\bm{\overline{r}}^T (\bm{x}_{i2} - \bm{x}_{i1}) \,|\, \pi_{i2} = \pi_j, \pi^* = \pi_k]\\ &\overset{i \longrightarrow \infty}\longrightarrow \, \mathbb{E}_i[\bm{\overline{r}} \,|\, \pi_{i2} = \pi_j, \pi^* = \pi_k]^T (\mathbb{E}_i[\bm{x}_{i2} \,|\, \pi_{i2} = \pi_j, \bm{\overline{p}}] - \mathbb{E}_i[\bm{x}_{i1}]) \\ &= \mathbb{E}_i[\bm{\overline{r}} \,|\, \pi^* = \pi_k]^T (\mathbb{E}_i[\bm{x}_{i2} \,|\, \pi_{i2} = \pi_j, \bm{\overline{p}}] - \mathbb{E}_i[\bm{x}_{i1}]), \text{ and similarly,} \\
\mathbb{E}_i[y_i \,|\, \pi_{i2} = \pi_j] &= \mathbb{E}_i[\bm{\overline{r}}^T (\bm{x}_{i2} - \bm{x}_{i1}) + \eta_i \,|\, \pi_{i2} = \pi_j] = \mathbb{E}_i[\bm{\overline{r}}^T (\bm{x}_{i2} - \bm{x}_{i1}) \,|\, \pi_{i2} = \pi_j]\\ &\overset{i \longrightarrow \infty}\longrightarrow \, \mathbb{E}_i[\bm{\overline{r}} \,|\, \pi_{i2} = \pi_j]^T (\mathbb{E}_i[\bm{x}_{i2} \,|\, \pi_{i2} = \pi_j, \bm{\overline{p}}] - \mathbb{E}_i[\bm{x}_{i1}]) \\ &= \mathbb{E}_i[\bm{\overline{r}}]^T (\mathbb{E}_i[\bm{x}_{i2} \,|\, \pi_{i2} = \pi_j, \bm{\overline{p}}] - \mathbb{E}_i[\bm{x}_{i1}]).
\end{flalign*}

Then, the difference between these two quantities asymptotically becomes:
\begin{flalign*}
\mathbb{E}_i[y_i \,|\, \pi_{i2} = \pi_j, \pi^* = \pi_k] -  \mathbb{E}_i[y_i \,|\, \pi_{i2} = \pi_j] \,\overset{i \longrightarrow \infty}{\longrightarrow}\, (\mathbb{E}_i[\bm{\overline{r}} \,|\, \pi^* = \pi_k] - \mathbb{E}_i[\bm{\overline{r}}])^T (\mathbb{E}_i[\bm{x}_{i2} \,|\, \pi_{i2} = \pi_j, \bm{\overline{p}}] - \mathbb{E}_i[\bm{x}_{i1}]).
\end{flalign*}

Asymptotically, $B_{jk}^{(i)}$ can then be expressed as:
\begin{flalign*}
B_{jk}^{(i)} &= \sqrt{P_i(\pi^* = \pi_j)P_i(\pi^* = \pi_k)} \Big(\mathbb{E}_i[y_i \,|\, \pi_{i2} = \pi_j, \pi^* = \pi_k] -  \mathbb{E}_i[y_i \,|\, \pi_{i2} = \pi_j] \Big) \\ &\,\overset{i \longrightarrow \infty}{\longrightarrow}\, \sqrt{P_i(\pi^* = \pi_k)} (\mathbb{E}_i[\bm{\overline{r}} \,|\, \pi^* = \pi_k] - \mathbb{E}_i[\bm{\overline{r}}])^T \sqrt{P_i(\pi^* = \pi_j)}(\mathbb{E}_i[\bm{x}_{i2} \,|\, \pi_{i2} = \pi_j, \bm{\overline{p}}] - \mathbb{E}_i[\bm{x}_{i1}]) \\ &:= \bm{u}_k^{\prime T} \bm{v^\prime}_j,
\end{flalign*}
where in the last equality, we define $\bm{u}_k^\prime := \sqrt{P_i(\pi^* = \pi_k)} (\mathbb{E}_i[\bm{\overline{r}} \,|\, \pi^* = \pi_k] - \mathbb{E}_i[\bm{\overline{r}}])$ and $\bm{v^\prime}_j := \sqrt{P_i(\pi^* = \pi_j)}$ $(\mathbb{E}_i[\bm{x}_{i2} \,|\, \pi_{i2} = \pi_j, \bm{\overline{p}}] - \mathbb{E}_i[\bm{x}_{i1}])$.

Thus, one can write $B^{(i)} = B_a^{(i)} + B_b^{(i)}$, where $B_a^{(i)}$ has $jk$\textsuperscript{th} element $\bm{u}_k^{\prime T} \bm{v^\prime}_j$. As in \eqref{eqn:matrix_prod}, $B_a^{(i)}$ can be written as a product of a $K \times d$ matrix with a $d \times K $ matrix, and so its rank is at most $d$. Meanwhile, the elements of $B_b^{(i)}$ decay to zero as $i \longrightarrow \infty$. For any $\varepsilon > 0$ and sufficiently-high $i$, the information ratio becomes upper-bounded by:
\begin{equation*}
    \Gamma_i \le \frac{\text{Tr}\left(B^{(i)}\right)^2}{2||B^{(i)}||_F^2} = \frac{\text{Tr}\left(B_a^{(i)} + B_b^{(i)}\right)^2}{2||B_a^{(i)} + B_b^{(i)}||_F^2} \overset{(a)}{\le} \frac{\text{Tr}\left(B_a^{(i)}\right)^2}{2||B_a^{(i)}||_F^2} + \varepsilon \overset{(b)}{\le} \frac{1}{2}\text{Rank}\left(B_a^{(i)}\right) + \varepsilon \overset{(c)}{\le} \frac{d}{2} + \varepsilon = \frac{SA}{2} + \varepsilon,
\end{equation*}
where (a) follows because $\text{Tr}(B)$ and $||B||_F$ are both continuous in the elements of $B$, and the elements of $B_b^{(i)}$ approach zero; (b) follows from Fact 10 in \cite{russo2016information}, as described earlier for the case with known dynamics; and (c) holds by definition of $B_a^{(i)}$.

\end{proof}

Combining Corollary \ref{coro:one_sided_bound_in_terms_of_Gamma} and Lemma \ref{lemma:bounded_inf_ratio} yields the asymptotic regret rate for a fixed $\pi_{i1}$-distribution:
\begin{theorem}\label{thm:one_sided_bound}
If the policy $\pi_{i1}$ is drawn from a fixed distribution for each iteration $i$, then for the competing policy $\pi_{i2}$, \algo~achieves a one-sided asymptotic Bayesian regret rate of:
\begin{equation*}
   S\sqrt{\frac{A T \log A}{2}}.
\end{equation*}
\end{theorem}
\begin{proof}
This result is a direct consequence of Corollary \ref{coro:one_sided_bound_in_terms_of_Gamma} and Lemma \ref{lemma:bounded_inf_ratio}; the asymptotic bound $\overline{\Gamma} \le \frac{SA}{2} + \varepsilon$ for any $\varepsilon$ is substituted into the expression in Corollary \ref{coro:one_sided_bound_in_terms_of_Gamma}.
\end{proof}

\subsection{BOUNDING THE ONE-SIDED REGRET WHEN ONE POLICY IS DRAWN FROM A DRIFTING AND CONVERGING DISTRIBUTION}\label{sec:bound_one_sided_converging}

We now assume that the distribution of $\pi_{i1}$ is no longer fixed, but rather, that the sampled policies $\pi_{i1}$ converge in distribution toward some fixed probability distribution over $\Pi$, the set of deterministic policies. We will asymptotically bound the one-sided regret incurred by $\pi_{i2}$ in this case. To do so, we will leverage that when two discrete random variables converge in distribution, their mutual information also converges:

\begin{lemma}\label{lemma:MI_cts}
Let $X_n$ and $Y_n$, $n \in \mathbb{N}$, be two sequences of discrete random variables that take values on the finite sets $\mathcal{X}$ and $\mathcal{Y}$, respectively. If $X_n \overset{D}{\longrightarrow} X$ and $Y_n \overset{D}{\longrightarrow} Y$, then $\lim_{n \longrightarrow \infty} I(X_n; Y_n) = I(X; Y)$.
\end{lemma}
\begin{proof}
Firstly, note that $X_n \overset{D}{\longrightarrow} X$ and $Y_n \overset{D}{\longrightarrow} Y$ imply that jointly, $(X_n, Y_n) \overset{D}{\longrightarrow} (X, Y)$. Let $P_n(x)$ and $P(x)$ be the probability distributions of $X_n$ and $X$, respectively. Because the variables are discrete, $X_n \overset{D}{\longrightarrow} X$ implies that $P_n(x) \longrightarrow P(x)$ for each $x \in \mathcal{X}$, and similarly for $Y$ and $(X, Y)$.

We express the mutual information as a sum of entropies: $I(X; Y) = H(X) + H(Y) - H(X, Y)$ \citep{cover2012elements}. It suffices to show that if $X_n \overset{D}{\longrightarrow} X$, then $\lim_{n \longrightarrow \infty}H(X_n) = H(X)$: because $Y$ and $(X, Y)$ are also discrete random variables, this would imply that similarly, $\lim_{n \longrightarrow \infty}H(Y_n) = H(Y)$ and $\lim_{n \longrightarrow \infty}H(X_n, Y_n) = H(X, Y)$.

By definition, the entropies of $X$ and $X_n$ are:
\begin{flalign*}
    H(X) &= -\sum_{x \in \mathcal{X}} P(x) \log P(x) \text{  and} \\
    H(X_n) &= -\sum_{x \in \mathcal{X}} P_n(x) \log P_n(x).
\end{flalign*}

Because $P_n(x) \longrightarrow P(x)$ for each $x \in \mathcal{X}$, for any $x \in \mathcal{X}$ and $\delta > 0$, there exists $N_x \in \mathbb{N}$ such that for all $n \ge N_x$, $|P_n(x) - P(x)| < \delta$. Let $N = \text{max}_{x \in \mathcal{X}} \, N_x$. Then, for all $n > N$ and all $x \in \mathcal{X}$, $|P_n(x) - P(x)| < \delta$.

Choose any $\varepsilon^\prime > 0$. Since $f(z) = z \log z$ is continuous for $z \ge 0$ (with $0 \log 0 := 0$), there exists $\delta > 0$ such that if $|P_n(x) - P(x)| < \delta$, then $|P_n(x)\log P_n(x) - P(x) \log P(x)| < \varepsilon^\prime$. We choose a $\delta$ that satisfies this condition.

Then, there exists $N$ such that for all $n > N$ and for all $x \in \mathcal{X}$, $|P_n(x)\log P_n(x) - P(x) \log P(x)| < \varepsilon^\prime$. Finally, choose $\varepsilon > 0$ and set $\varepsilon^\prime \le \frac{\varepsilon}{|\mathcal{X}|}$, so that for all $n > N$:

\begin{flalign*}
\Bigr| H(X) - H(X_n) \Bigr| &= \left| \sum_{x \in \mathcal{X}} P(x) \log P(x) - \sum_{x \in \mathcal{X}} P_n(x) \log P_n(x)\right| \le \sum_{x \in  \mathcal{X}} \Big|P(x) \log P(x) - P_n(x) \log P_n(x) \Big| \\ &\le \sum_{x \in  \mathcal{X}} \varepsilon^\prime = \varepsilon^\prime |\mathcal{X}| \le \frac{\varepsilon}{|\mathcal{X}|} |\mathcal{X}| = \varepsilon.
\end{flalign*}

So for any $\varepsilon > 0$, there exists $N$ such that for all $n > N$, $|H(X) - H(X_n)| \le \varepsilon$. This proves that $\lim_{n \longrightarrow \infty} H(X_n) = H(X)$, and therefore that $\lim_{n \longrightarrow \infty} I(X_n; Y_n) = I(X; Y)$.
\end{proof}

Armed with this convergence of mutual information, the one-sided regret for $\pi_{i2}$ can be bounded as follows:
\begin{lemma}\label{lemma:drift_converge}
Assume that $\pi_{i1}$ is drawn from a distribution that is drifting and converging, that is, $\pi_{i1}$ converges in distribution to some fixed probability distribution. From Lemma 10, if $\pi_{i1}$ is drawn from a fixed distribution, then asymptotically, its information ratio is bounded by $\overline{\Gamma}_{\pi_{i1\, \mathrm{fixed}}} \le \frac{SA}{2}$. In the case of a drifting/converging $\pi_{i1}$ distribution, the information ratio $\Gamma_i$ for $\pi_{i2}$'s one-sided regret satisfies $\lim_{i \longrightarrow \infty} \, \Gamma_i \le \overline{\Gamma}_{\pi_{i1\, \mathrm{fixed}}} \le \frac{SA}{2}$.
\end{lemma}
\begin{proof}
By Lemma \ref{lemma:bounded_inf_ratio}, under a fixed $\pi_{i1}$ distribution, the information ratio corresponding to the one-sided regret for $\pi_{i2}$ is asymptotically upper-bounded: $\lim_{i \longrightarrow \infty} \, \Gamma_{i, \pi_{i1 \, \mathrm{fixed}}} \le \overline{\Gamma}_{\pi_{i1 \, \mathrm{fixed}}}$, where $\Gamma_{i, \pi_{i1 \, \mathrm{fixed}}}$ is the information ratio at iteration $i$ when the distribution of $\pi_{i1}$ is fixed.

The denominator of the information ratio, $I_i(\pi^*; (\pi_{i2}, \tau_{i1}, \tau_{i2}, \bm{x}_{i2} - \bm{x}_{i1}, y_i))$, is a mutual information between discrete random variables, as there are finitely-many possible policies, trajectories, and preference outcomes; therefore, by Lemma \ref{lemma:MI_cts}, it converges to the values that it would have under the fixed distribution to which $\pi_{i1}$ converges.

The numerator of the information ratio is the square of the expected instantaneous one-sided regret, $\mathbb{E}\left[\bm{\overline{r}}^T (\bm{x}_i^* - \bm{x}_{i2}) \, | \, \mathcal{H}_{i - 1}^{(2)}\right]^2$. Conditioned upon the history $\mathcal{H}_{i - 1}^{(2)}$, this does not depend upon the action $\bm{x}_{i1}$, and thus is unaffected by the distribution of $\bm{x}_{i1}$ (recall that the $\bm{x}_{i1}$-dependency cancels in the regret formulation in \eqref{eqn:xi1_cancel}).

Thus, the information ratio can be asymptotically upper-bounded:
\begin{equation*}
    \lim_{i \longrightarrow \infty} \, \Gamma_i \le \lim_{i \longrightarrow \infty} \, \Gamma_{i, \, \pi_{i1 \, \text{fixed}}} \le \overline{\Gamma}_{\pi_{i1\, \mathrm{fixed}}}.
\end{equation*}
This means that for all $\varepsilon > 0$, there exists $i_0$ such that for all $i > i_0$, $\Gamma_i \le \overline{\Gamma}_{\pi_{i1\, \mathrm{fixed}}} + \varepsilon \le \frac{SA}{2} + \varepsilon$.
\end{proof}

\subsection{OBTAINING THE ASYMPTOTIC REGRET RATE}\label{sec:asy_regret_rate}

By combining Lemma \ref{lemma:drift_converge} with previous results, we obtain the final asymptotic Bayesian regret rate.
\begin{theorem}\label{thm:regret_rate}
With probability $1 - \delta$, where $\delta$ is a parameter of the Bayesian linear regression model, the Bayesian regret $\mathbb{E}[\reg(T)]$ of \algo~achieves an asymptotic rate of $S\sqrt{2A T \log A}$.
\end{theorem}
\begin{proof}
Combining Theorem \ref{thm:policy_consistency}, Theorem \ref{thm:one_sided_bound}, and Lemma \ref{lemma:drift_converge}, the Bayesian one-sided regrets $\mathbb{E}[\reg_1(T)]$ and $\mathbb{E}[\reg_2(T)]$ of policies $\pi_{i1}$ and $\pi_{i2}$ respectively each achieve asymptotic rates of $S\sqrt{\frac{A T \log A}{2}}$. The total regret is the sum of the regret contributions from policies $\pi_{i1}$ and $\pi_{i2}$, and so asymptotically, $\mathbb{E}[\reg(T)]$ increases at a rate of at most $S\sqrt{2A T \log A}$.
\end{proof}

\subsection{FACTS ABOUT CONVERGENCE IN DISTRIBUTION}\label{sec:appendix_prob_facts}

Recall that for a random variable $X$ and a sequence of random variables $(X_n)$, $n \in \mathbb{N}$, $X_n \overset{D}\longrightarrow X$ denotes that $X_n$ converges to $X$ in distribution, while $X_n \overset{P}\longrightarrow X$ denotes that $X_n$ converges to $X$ in probability. We apply the following two facts about convergence in distribution:

\begin{fact}[\cite{billingsley1968convergence}]\label{fact:conv_dist_1}
For random variables $\bm{x}, \bm{x}_n, \in \mathbb{R}^d$, where $n \in \mathbb{N}$, and any continuous function $g: \mathbb{R}^d \longrightarrow \mathbb{R}$, if $\bm{x}_n \overset{D}\longrightarrow \bm{x}$, then $g(\bm{x}_n) \overset{D}\longrightarrow g(\bm{x})$.
\end{fact}

\begin{fact}[\cite{billingsley1968convergence}]\label{fact:conv_dist_2}
For random variables $\bm{x}_n \in \mathbb{R}^d$, $n \in \mathbb{N}$, and constant vector $\bm{c} \in \mathbb{R}^d$, $\bm{x}_n \overset{D}\longrightarrow \bm{c}$ is equivalent to $\bm{x}_n \overset{P}\longrightarrow \bm{c}$. Convergence in probability means that for any $\varepsilon > 0, P(||\bm{x}_n - \bm{c}||_2 \ge \varepsilon) \longrightarrow 0$ as $n \longrightarrow \infty$.
\end{fact}


\newpage

\section{CREDIT ASSIGNMENT MODELS}\label{sec:credit_assignment_models}

This appendix contains the mathematical details of the credit assignment models evaluated in our experiments. Afterward, we also discuss possible avenues for extending our regret analysis techniques to additional credit assignment models besides Bayesian linear regression.

\subsection{BAYESIAN LINEAR REGRESSION}\label{sec:CR_bayes_lin_reg}

Define $X \in \mathbb{R}^{N \times d}$ as the observation matrix after $N$ preferences, in which the $i$\textsuperscript{th} row contains observation $\bm{x}_i = \bm{x}_{i2} - \bm{x}_{i1}$, while $\bm{y} \in \mathbb{R}^N$ is the vector of corresponding preference labels, with $i$\textsuperscript{th} element $y_i \in \left\{-\frac{1}{2}, \frac{1}{2}\right\}$.

Section \ref{sec:credit_assignment} defines the Bayesian linear regression credit assignment model to which our theoretical guarantees apply. Because the $\beta_i(\delta)$ factor necessary for the theoretical guarantees results in a conservative covariance matrix leading to over-exploration, our simulations implement the following, more practical, variant. We define a Gaussian prior over the reward vector $\bm{r} \in \mathbb{R}^d$: $\bm{r} \sim \mathcal{N}(0, \lambda^{-1}I)$. The likelihood of the data conditioned upon $\bm{r}$ is also Gaussian:
\begin{flalign*}
   p(\mathbf{y} | X, \bm{r}; \sigma^2) = \frac{1}{(2 \pi \sigma^2)^{\frac{N}{2}}}\mathrm{exp}\left(-\frac{1}{2 \sigma^2}||\mathbf{y} - X\bm{r}||^2\right).
\end{flalign*}
This conjugate prior and likelihood lead to the following closed-form posterior:
\begin{flalign*}
    \bm{r} | X, \bm{y}, \sigma^2, \lambda  \sim \mathcal{N}(\bm{\mu}, \Sigma)\text{, where } \bm{\mu} = (X^T X + \sigma^2 \lambda I)^{-1}X^T \bm{y} \, \text{ and } \, \Sigma = \sigma^2 (X^T X + \sigma^2 \lambda I)^{-1}.
\end{flalign*}

\subsection{GAUSSIAN PROCESS REGRESSION}

Credit assignment via Gaussian processes \citep{rasmussen2006gaussian} extends the linear credit assignment model in \ref{sec:CR_bayes_lin_reg} to larger state and action spaces by generalizing across nearby states and actions. In this and the following section, we consider two Gaussian process-based credit assignment approaches.

To perform credit assignment via Gaussian process regression, we assign binary labels to each trajectory based on whether it is preferred or dominated. We place a Gaussian process prior upon the utilities of the state-action pairs. Using that a trajectory's total reward is a sum over component state-action utilities, we will show how to perform inference over sums of Gaussian process variables to infer the state-action utilities from each trajectory's total utility. As the total utility of each trajectory is not observed in practice, the binary preference labels are instead substituted as approximations in their place.

Let $\{\tilde{s}_1, \ldots, \tilde{s}_d\}$ denote the $d = SA$ state-action pairs. In this section, the data matrix $Z \in \mathbb{R}^{2N \times d}$ holds all state-action visitation vectors $\bm{x}_{k1}, \bm{x}_{k2}$, for \algo~iterations $k \in \{1, \ldots, N\}$. (This contrasts with the other credit assignment methods, which learn from their differences, $\bm{x}_{k2} - \bm{x}_{k1}$.) Let $\bm{z}_i^T$ be the $i$\textsuperscript{th} row of $Z$, such that $Z = [\bm{z}_1 \ldots, \bm{z}_{2N}]^T$, and $\bm{z}_{i} = \bm{x}_{kj}$ for some \algo~iteration $k$ and $j \in \{1, 2\}$, that is, $\bm{z}_i$ contains the state-action visit counts for the $i$\textsuperscript{th} trajectory rollout. In particular, the $ij$\textsuperscript{th} matrix element $z_{ij} = [Z]_{ij}$ is the number of times that the $i$\textsuperscript{th} observed trajectory $\bm{z}_i$ visits state-action $\tilde{s}_j$.

The label vector is $\bm{y} \in \mathbb{R}^{2N}$, where the $i$\textsuperscript{th} element $y_i$ is the preference label corresponding to the $i$\textsuperscript{th} observed trajectory. For instance, if $\tau_{i2} \succ \tau_{i1}$, then $\bm{x}_{i2}$ receives a label of $\frac{1}{2}$, while $\bm{x}_{i1}$ is labelled $-\frac{1}{2}$. As before, we use $\rr(\tilde{s})$ to denote the true utility of state-action $\tilde{s}$, with $\rr(\tau)$ being trajectory $\tau$'s total utility along the state-action pairs it encounters. To infer $\bm{\rr}$, we approximate each $\rr(\tau_i)$ with its preference label $y_i$.

We place a Gaussian process prior upon the rewards $\bm{\rr}$: $\bm{\rr} \sim \mathcal{GP}(\bm{\mu_r}, K_r)$, where $\bm{\mu_r} \in \mathbb{R}^d$ is the prior mean and $K_r \in \mathbb{R}^{d \times d}$ is the prior covariance matrix, such that $[K_r]_{ij}$ models the prior covariance between $\rr(\tilde{s}_i)$ and $\rr(\tilde{s}_j)$. We model trajectory $\tau_i$'s total utility, $\rr(\tau_i)$, as a sum over the latent state-action utilities: $\rr(\tau_i) = \sum_{j = 1}^d z_{ij} \rr(\tilde{s}_j)$. Let $R_i$ be a noisy version of $\rr(\tau_i)$: $R_i = \rr(\tau_i) + \varepsilon_i$, where $\varepsilon_i \sim \mathcal{N}(0, \sigma_\varepsilon^2)$ is i.i.d. noise. Then, given rewards $\bm{\rr}$, we expect:
\begin{equation*}
    R_i = \sum_{j = 1}^d z_{ij} \rr(\tilde{s}_j) + \varepsilon_i.
\end{equation*}

Because any linear combination of jointly Gaussian variables is Gaussian, $R_i$ is a Gaussian process over the values $\{z_{i1}, \ldots, z_{id}\}$. Let $\bm{R} \in \mathbb{R}^{2N}$ be the vector with $i$\textsuperscript{th} element equal to $R_i$. We will calculate the relevant expectations and covariances to show that $\bm{\rr} \sim \mathcal{GP}(\bm{\mu_r}, K_r)$ and $\bm{R}$ have the following jointly-Gaussian distribution:
\begin{equation}\label{eqn:joint_Gaussian}
\begin{bmatrix} \bm{\rr} \\ \bm{R}\end{bmatrix} \sim \mathcal{N}\left(\begin{bmatrix} \bm{\mu_r} \\ X\bm{\mu_r}\end{bmatrix}, \begin{bmatrix} K_r & K_r Z^T \\ Z K_r^T & Z K_r Z^T + \sigma_\varepsilon^2 I\end{bmatrix}\right).
\end{equation}
The standard approach for obtaining a conditional distribution from a joint Gaussian distribution \citep{rasmussen2006gaussian} yields $\bm{\rr} | \bm{R} \sim \mathcal{N}(\bm{\mu}, \Sigma)$, where:
\begin{equation} \label{eqn:GP_update_mean}
\bm{\mu} = \bm{\mu_r} + K_r Z^T [Z K_r Z^T + \sigma_\varepsilon^2 I]^{-1}(\bm{R} - Z\bm{\mu_r})
\end{equation}
\begin{equation} \label{eqn:GP_update_cov}
\Sigma = K_r - K_r Z^T [Z K_r Z^T + \sigma_\varepsilon^2 I]^{-1} Z K_r^T.
\end{equation}

In practice, we do not observe the variable $\bm{R}$.  Instead, $\bm{R}$ is approximated with the observed preference labels $\bm{y}$, $\bm{R} \approx \bm{y}$, to perform credit assignment inference.

Next, we derive the posterior inference equations \eqref{eqn:GP_update_mean} and \eqref{eqn:GP_update_cov} used in Gaussian process regression credit assignment. We infer the state-action rewards $\bm{\rr}$ given noisy observations $\bm{R}$ of the trajectories' total utilities via the following four steps, corresponding to the next four subsections:
\begin{enumerate}[label={\Alph*)}]
    \item Model the state-action utilities $\rr(\tilde{s})$ as a Gaussian process over state-action pairs $\tilde{s}$.
    \item Model the trajectory utilities $\bm{R}$ as a Gaussian process that results from summing the state-action utilities $\rr(\tilde{s})$.
    \item Using the two Gaussian processes defined in A) and B), obtain the covariance matrix between the values of $\{\rr(\tilde{s}) | \tilde{s} \in 1, \ldots, d\}$ and $\{R_i | i \in 1, \ldots, 2N\}$.
    \item Write the joint Gaussian distribution in \eqref{eqn:joint_Gaussian} between the values of $\{\rr(\tilde{s}) | \tilde{s} \in 1, \ldots, d\}$ and $\{R_i | i \in 1, \ldots, 2N\}$, and obtain the posterior distribution of $\bm{\rr}$ over all state-actions given $\bm{R}$ (Equations \eqref{eqn:GP_update_mean} and \eqref{eqn:GP_update_cov}).
\end{enumerate}

\subsubsection{The state-action utility Gaussian process}

We model the state-action utilities as a Gaussian process over $\tilde{s}$, with mean $\mathbb{E}[\rr(\tilde{s})] = \mu_r(\tilde{s})$ and covariance kernel $\mathrm{Cov}(\rr(\tilde{s}), \rr(\tilde{s}')) = k_r(\tilde{s}, \tilde{s}')$, for all state-action pairs $\tilde{s}, \tilde{s}'$. For instance, $k_r$ could be the squared exponential kernel:
\begin{equation}\label{eqn:SE_kernel}
    k_r(\tilde{s}, \tilde{s}^\prime) = \sigma_f^2\exp{\left(-\frac{1}{2}\left(\frac{||\tilde{s} - \tilde{s}^\prime||}{l}\right)^2\right)}+\sigma_n^2\delta_{ij},
\end{equation}
where $\sigma_f$ is the signal variance, $l$ is the kernel lengthscale, $\sigma_n$ is the noise variance, and $\delta_{ij}$ is the Kronecker delta function. Thus,
\begin{equation*}
    \rr(\tilde{s}) \sim \mathcal{GP}(\mu_r(\tilde{s}), k_r(\tilde{s}, \tilde{s}')).
\end{equation*}

Define $\bm{\mu}_r \in \mathbb{R}^d$ such that the $i$\textsuperscript{th} element is $[\bm{\mu}_r]_i = \mu_r(\tilde{s}_i)$, the prior mean of state-action $\tilde{s}_i$'s utility. Let $K_r \in \mathbb{R}^{d \times d}$ be the covariance matrix over state-action utilities, such that $[K_r]_{ij} = k_r(\tilde{s}_i, \tilde{s}_j)$. Therefore, the reward vector $\bm{\rr}$ is also a Gaussian process:
\begin{equation*}
    \bm{\rr} \sim \mathcal{GP}(\bm{\mu}_r, K_r).
\end{equation*}

\subsubsection{The trajectory utility Gaussian process}

By assumption, the trajectory utilities $\bm{R} \in \mathbb{R}^{2N}$ are sums of the latent state-action utilities via the following relationship between $\bm{R}$ and $\bm{\rr}$:
\begin{equation*}
    R(\bm{z}_i) := R_i = \sum_{j = 1}^d z_{ij} \rr(\tilde{s}_j) + \varepsilon_i,
\end{equation*}
\noindent where $\varepsilon_i$ are i.i.d. noise variables distributed according to $\mathcal{N}(0, \sigma_\varepsilon^2)$. Note that $R(\bm{z}_i)$ is a Gaussian process over $\bm{z}_i \in \mathbb{R}^d$ because $\{\rr(\tilde{s}_j), \forall j\}$ are jointly normally distributed by definition of a Gaussian process, and any linear combination of jointly Gaussian variables has a univariate normal distribution. Next, we calculate the expectation and covariance of $\bm{R}$ over the observations. The expectation of the $i$\textsuperscript{th} element $R_i = R(\bm{z}_i)$ can be expressed:
\begin{flalign*}
    \mathbb{E}[R_i] = \mathbb{E}\left[\sum_{j = 1}^d z_{ij} \rr(\tilde{s}_j)  + \varepsilon_i\right] = \sum_{j = 1}^d z_{ij} \mathbb{E}[\rr(\tilde{s}_j)] = \sum_{j = 1}^d z_{ij} \mu_r(\tilde{s}_j).
\end{flalign*}

The expectation over $\bm{R}$ can thus be written as $\mathbb{E}[\bm{R}(Z)] = Z \bm{\mu}_r$. Next, we model the covariance matrix of $\bm{R}$. The $ij$\textsuperscript{th} element of this matrix is the covariance of $R(\bm{z}_i)$ and $R(\bm{z}_j)$:
\begin{flalign*}
    \mathrm{Cov}(R(\bm{z}_i), R(\bm{z}_j)) &= \mathbb{E}[R(\bm{z}_i)R(\bm{z}_j)] - \mathbb{E}[R(\bm{z}_i)]\mathbb{E}[R(\bm{z}_j)] \\
    &= \mathbb{E}\left[\left(\sum_{k = 1}^d z_{ik}\rr(\tilde{s}_k) + \varepsilon_i\right)\left(\sum_{m = 1}^d z_{jm}\rr(\tilde{s}_m) + \varepsilon_j\right)\right] - \left(\sum_{k = 1}^d z_{ik}\mu_r(\tilde{s}_k)\right)\left(\sum_{m = 1}^d z_{jm}\mu_r(\tilde{s}_m)\right) \\
    &= \sum_{k = 1}^d \sum_{m = 1}^d z_{ik} z_{jm} \mathbb{E}[\rr(\tilde{s}_k)r(\tilde{s}_m)] + \mathbb{E}[\varepsilon_i \varepsilon_j] - \sum_{k = 1}^d \sum_{m = 1}^d z_{ik} z_{jm} \mu_r(\tilde{s}_k) \mu_r(\tilde{s}_m) \\
    &= \sum_{k = 1}^d \sum_{m = 1}^d \left\{z_{ik} z_{jm} [\mathrm{Cov}(\rr(\tilde{s}_k), \rr(\tilde{s}_m)) + \mu_r(\tilde{s}_k) \mu_r(\tilde{s}_m)] - z_{ik} z_{jm} \mu_r(\tilde{s}_k) \mu_r(\tilde{s}_m) + \sigma_\varepsilon^2 \mathbb{I}_{[i = j]} \right\}\\
    &= \sum_{k = 1}^d \sum_{m = 1}^d z_{ik} z_{jm}  \mathrm{Cov}(\rr(\tilde{s}_k), \rr(\tilde{s}_m)) + \sigma_\varepsilon^2 \mathbb{I}_{[i = j]} \\ &= \sum_{k = 1}^d \sum_{m = 1}^d z_{ik} z_{jm} k_r(\tilde{s}_k, \tilde{s}_m) + \sigma_\varepsilon^2 \mathbb{I}_{[i= j]} = \bm{z}_i^T K_r \bm{z}_j + \sigma_\varepsilon^2 \mathbb{I}_{[i = j]}.
\end{flalign*}

We can then write the covariance matrix of $\bm{R}$ as $K_R$, where $[K_R]_{ij} :=  \mathrm{Cov}(R(\bm{z}_i), R(\bm{z}_j)) = \bm{z}_i^T K_r \bm{z}_j + \sigma_\varepsilon^2 \mathbb{I}_{[i = j]}$. From here, it can be seen that $K_R = Z K_r Z^T + \sigma_\varepsilon^2 I\,$:
\begin{equation*}
Z K_r Z^T = \begin{bmatrix}\bm{z}_1^T \\ \bm{z}_2^T \\ \vdots \\ \bm{z}_{2N}^T\end{bmatrix} K_r \begin{bmatrix}\bm{z}_1 & \bm{z}_2 & \ldots & \bm{z}_{2N}\end{bmatrix} = \begin{bmatrix}\bm{z}_1^T K_r \bm{z}_1 & \ldots & \bm{z}_1^T K_r \bm{z}_{2N} \\ \vdots & \ddots & \vdots \\ \bm{z}_{2N}^T K_r \bm{z}_1 & \ldots & \bm{z}_{2N}^T K_r \bm{z}_{2N}\end{bmatrix} = K_R - \sigma_\varepsilon^2 I.
\end{equation*}

\subsubsection{Covariance between state-action and trajectory utilities}

We next consider the covariance between $\bm{\rr}$ and $\bm{R}$, denoted $K_{r, R}$:
\begin{equation*}
    [K_{r, R}]_{ij} = \mathrm{Cov}([\bm{\rr}]_i, [\bm{R}]_j) = \mathrm{Cov}(\rr(\tilde{s}_i), R(\bm{z}_j)).
\end{equation*}

This covariance matrix can be expressed in terms of $Z, K_r$, and $\bm{\mu}_r$:
\begin{flalign*}
[K_{r, R}]_{ij} &= \mathrm{Cov}(\rr(\tilde{s}_i), R(\bm{z}_j)) = \mathrm{Cov}\left(\rr(\tilde{s}_i), \sum_{k = 1}^d z_{jk} \rr(\tilde{s}_k) + \varepsilon_j \right) \\
&= \mathbb{E}\left[\rr(\tilde{s}_i)\sum_{k = 1}^d z_{jk} \rr(\tilde{s}_k) + \varepsilon_j \rr(\tilde{s}_i) \right] - \mathbb{E}[\rr(\tilde{s}_i)]\mathbb{E}\left[\sum_{k = 1}^d z_{jk} \rr(\tilde{s}_k) + \varepsilon_j\right] \\
&= \sum_{k = 1}^d z_{jk} \mathbb{E}[\rr(\tilde{s}_i) \rr(\tilde{s}_k)] - [\mu_r(\tilde{s}_i)][\bm{z}_j^T \bm{\mu}_r] \\
&= \sum_{k = 1}^d z_{jk} \{\mathrm{Cov}(\rr(\tilde{s}_i), \rr(\tilde{s}_k)) + \mathbb{E}[\rr(\tilde{s}_i)]\mathbb{E}[\rr(\tilde{s}_k)]\} - \mu_r(\tilde{s}_i)\bm{z}_j^T \bm{\mu}_r \\
&= \sum_{k = 1}^d z_{jk} [k_r(\tilde{s}_i, \tilde{s}_k) + \mu_r(\tilde{s}_i)\mu_r(\tilde{s}_k)] - \mu_r(\tilde{s}_i)\bm{z}_j^T \bm{\mu}_r \\
&= \sum_{k = 1}^d z_{jk} k_r(\tilde{s}_i, \tilde{s}_k) + \mu_r(\tilde{s}_i)\bm{z}_j^T \bm{\mu}_r - \mu_r(\tilde{s}_i)\bm{z}_j^T \bm{\mu}_r = \sum_{k = 1}^d z_{jk} k_r(\tilde{s}_i, \tilde{s}_k) = \bm{z}_j^T [K_r]_{i, :}^T,
\end{flalign*}
where $[K_r]_{i, :}^T$ is the column vector obtained by transposing the $i$\textsuperscript{th} row of $K_r$. It is evident that $K_{r, R} = K_r Z^T$.

\subsubsection{Posterior inference over state-action utilities}

Merging the previous three subsections' results, one obtains the following joint probability density between $\bm{\rr}$ and $\bm{R}$:
\begin{equation*}
\begin{bmatrix} \bm{\rr} \\ \bm{R}\end{bmatrix} \sim \mathcal{N}\left(\begin{bmatrix} \bm{\mu}_r \\ Z\bm{\mu}_r\end{bmatrix}, \begin{bmatrix} K_r & K_r Z^T \\ Z K_r^T & Z K_r Z^T + \sigma_\varepsilon^2 I\end{bmatrix}\right).
\end{equation*}

This relationship expresses all components of the joint Gaussian density in terms of $Z, K_r$, and $\bm{\mu}_r$, or in other words, in terms of the observed state-action visitation counts (i.e., $Z$) and the Gaussian process prior on $\bm{\rr}$. Via the standard approach for obtaining a conditional distribution from a joint Gaussian distribution, we obtain $\bm{\rr} | \bm{R} \sim \mathcal{N}(\bm{\mu}, \Sigma)$, where the expressions for $\mu$ and $\Sigma$ are given by Equations \eqref{eqn:GP_update_mean} and \eqref{eqn:GP_update_cov} above. Substituting $\bm{y}$ for $\bm{R}$, we have expressed the conditional posterior density of $\bm{\rr}$ in terms of $Z$, $\bm{y}$, $K_r$, and $\bm{\mu_r}$.

\subsection{GAUSSIAN PROCESS PREFERENCE MODEL}

Finally, we show how to extend the preference-based Gaussian process model defined in \cite{chu2005preference} from the dueling bandit setting to the PBRL setting to perform credit assignment. Similarly to the GP regression model, this approach places a Gaussian prior over possible rewards $\bm{\rr}$; in contrast, however, this method explicitly models the likelihood of the observed preferences given the utilities $\bm{\rr}$, and thus it is a more theoretically-justified approach for handling preference data.

We elicit a preference feedback dataset $\mathcal{D} = \{\tau_{i2} \succ \tau_{i1} \,|\, i = 1, ... , N \}$, where $\tau_{i2} \succ \tau_{i1}$ indicates that trajectory $\tau_{i2}$ is preferred to $\tau_{i1}$ in preference $i$. Without loss of generality, we index each trajectory pair $i$ such that $\tau_{i2} \succ \tau_{i1}$. As before, we assume that each state-action pair $\tilde{s}_j, j \in \{1, \ldots, d\}$, has a latent, underlying utility $\rr(\tilde{s}_j)$. In vector form, these are written: $\bm{\rr} = [\rr(\tilde{s}_1),\rr(\tilde{s}_2),...,\rr(\tilde{s}_d)]^T$. We define a Gaussian prior over $\bm{\rr}$:
\begin{equation}\label{eqn:prior}
	p(\bm{\rr})=\frac{1}{{(2\pi)}^{\frac{d}{2}}\left | \Sigma\right |^\frac{1}{2}}\exp{\left(-\frac{1}{2}\bm{\rr}^{T}\Sigma^{-1}\bm{\rr}\right)},
\end{equation}
where $\Sigma\in \mathbb{R}^{d \times d}$ and $[\Sigma]_{ij} = k(\rr(\tilde{s}_i),\rr(\tilde{s}_j))$ for some kernel function $k$, such as the squared exponential kernel defined in \eqref{eqn:SE_kernel}. Next, we assume that the likelihood of the $i$\textsuperscript{th} preference given utilities $\bm{\rr}$ takes the following form:
\begin{equation*}
P(\tau_{i2} \succ \tau_{i1} \,|\, \bm{\rr}) = g\left(\frac{\rr(\tau_{i2}) - \rr(\tau_{i1})}{c}\right),
\end{equation*}
where $g(\cdot)$ is a monotonically-increasing link function that is bounded between 0 and 1, and $c > 0$ is a model hyperparameter controlling the degree of preference noise. The total return $\rr(\tau_{i1})$ of trajectory $\tau_{i1}$ can be written in terms of the corresponding state-action visitation vector $\bm{x}_{i1}$: $\rr(\tau_{i1}) = \bm{\rr}^T \bm{x}_{i1}$. Thus, the full likelihood expression is:
\begin{equation}\label{eqn:likelihood}
	P(\mathcal{D} \,|\, \bm{\rr}) = \prod_{i=1}^{N}g(z_i), \,\,\,\,\,\,\,\,\, z_i := \frac{\rr(\tau_{i2})- \rr(\tau_{i1})}{c} = \frac{\bm{\rr}^T(\bm{x}_{i2} - \bm{x}_{i1})}{c} = \frac{\bm{\rr}^T \bm{x}_i}{c}.
\end{equation}

Given the preference dataset $\mathcal{D}$, we are interested in the posterior probability of $\bm{\rr}$:
\begin{equation*}
p(\bm{\rr}\,|\, \mathcal{D}) \propto P(\mathcal{D} \,|\, \bm{\rr})p(\bm{\rr}),
\end{equation*}
where the expressions for the prior $p(\bm{\rr})$ and likelihood $P(\mathcal{D} \,|\, \bm{\rr})$ are given by Equations \eqref{eqn:prior} and \eqref{eqn:likelihood}, respectively. We sample reward vectors $\bm{\tilde{r}}$ from the posterior obtained via the Laplace approximation, $\bm{\tilde{r}} \sim \mathcal{N}(\bm{\hat{r}}_{\textup{MAP}},\, \alpha\Sigma_{\textup{MAP}})$, where: \begin{flalign}
\bm{\hat{r}}_{\textup{MAP}} &= \text{argmin}_{\bm{r}} S(\bm{r}),
\label{eqn:GP_model_mean} \\
\Sigma_{\textup{MAP}} &= \left(\nabla^2_{\bm{r}}S(\bm{r})|_{\bm{\hat{r}}_{\textup{MAP}}}\right)^{-1}, \label{eqn:GP_model_cov}
\end{flalign}
and $S(\bm{r}) := \frac{1}{2}\bm{r}^{T}\Sigma^{-1}\bm{r} - \sum_{i=1}^{N}\log g(z_{i})$ (note that $S(\bm{r})$ is equivalent to $-\log p(\bm{r}, \mathcal{D})$, neglecting constant terms); lastly, $\alpha > 0$ is a tunable hyperparameter that influences the balance between exploration and exploitation. In order for the Laplace approximation to be valid, $S(\bm{r})$ must be a convex function in $\bm{r}$: this guarantees that the optimization problem in \eqref{eqn:GP_model_mean} is convex and that the covariance matrix defined by \eqref {eqn:GP_model_cov} is positive semidefinite, and therefore a valid Gaussian covariance matrix. Convexity of $S(\bm{r})$ can be established by demonstrating that its Hessian matrix is positive definite. It can be shown that for any $\bm{r}$, $\nabla^2_{\bm{r}} S(\bm{r}) = \Sigma^{-1} + \Lambda$, where:
\begin{equation}\label{eqn:positive_semi}
	\Lambda_{mn} := \frac{1}{c^2}\sum_{i=1}^N 
	[\bm{x}_i]_m [\bm{x}_i]_n\left[-\frac{g^{\prime\prime}(z_i)}{g(z_i)}+\left(\frac{g^{\prime}(z_i)}{g(z_i)}\right)^2\right],
\end{equation}
for $\bm{x}_i = \bm{x}_{i2} - \bm{x}_{i1}$. To show that $\nabla^2_{\bm{r}} S(\bm{r})$ is positive definite, because the prior covariance $\Sigma$ is positive definite, it suffices to show that $\Lambda$ is positive semidefinite. From \eqref{eqn:positive_semi}, one can see that:
\begin{equation*}
    \Lambda = \frac{1}{c^2}\sum_{i=1}^N \left[-\frac{g^{\prime\prime}(z_i)}{g(z_i)}+\left(\frac{g^{\prime}(z_i)}{g(z_i)}\right)^2\right] \bm{x}_i \bm{x}_i^T.
\end{equation*}
Clearly $\bm{x}_i \bm{x}_i^T$ is positive semidefinite, and thus we arrive at the following sufficient condition for convexity of $S(\bm{r})$:
\begin{equation*}
    \left[-\frac{g^{\prime\prime}(z)}{g(z)}+\left(\frac{g^{\prime}(z)}{g(z)}\right)\right] \ge 0 \text{ for all } z \in \mathbb{R}.
\end{equation*}

This condition is in particular satisfied for the Gaussian link function, $g_{\text{Gaussian}}(\cdot) = \Phi(\cdot)$, where $\Phi$ is the standard Gaussian CDF, as well as for the sigmoidal link function, $g_{\text{sig}}(x) := \sigma (x) = \frac{1}{1 + \exp(-x)}$. Our experiments utilize the sigmoidal link function.

\subsubsection{Bayesian Logistic Regression}

Many of our experiments with the Gaussian process preference model fall under the special case of Bayesian logistic regression, in which $c = 1$, $g$ is the sigmoidal link function, and the prior covariance matrix is diagonal, i.e. $\Sigma = \lambda I$; for instance, the latter condition occurs with the squared exponential kernel defined in \eqref{eqn:SE_kernel} when its lengthscale $l$ is set to zero. In this case, the Gaussian prior over possible reward vectors $\bm{r} \in \mathbb{R}^{d}$ is: $\bm{r} \sim \mathcal{N}(\bm{0}, \lambda I)$, where $\lambda > 0$. Setting the $i$\textsuperscript{th} preference label $y_i$ equal to $1$ if $\tau_{i2} \succ \tau_{i1}$, while $y_i = -1$ if $\tau_{i1} \succ \tau_{i2}$, the logistic regression likelihood is:
\begin{flalign*}
p(\mathcal{D} | \bm{r}) = \prod_{i = 1}^N p(\bm{x}_i, y_i | \bm{r}) = \prod_{i = 1}^N \frac{1}{1 + \text{exp}(-y_i \bm{x}_i^T\bm{r})}.
\end{flalign*}

We approximate the posterior, $p(\bm{r} \,|\, \mathcal{D}) \propto p(\mathcal{D} \,|\, \bm{r})p(\bm{r})$, as Gaussian via the Laplace approximation:
\begin{flalign}
p(\bm{r}\,|\,\mathcal{D}) &\approx \mathcal{N}(\bm{\hat{r}}_{\textup{MAP}},\, \alpha\Sigma_{\textup{MAP}}) \text{, where:} \nonumber \\
\bm{\hat{r}}_{\textup{MAP}} &= \underset{\bm{r}}{\text{argmin}} \, f(\bm{r}), \,\,\,\,\,\, f(\bm{r}) := -\text{log} \, p(\mathcal{D}, \bm{r}) =  -\text{log} \, p(\bm{r}) - \text{log} \, p(\mathcal{D}| \bm{r}), \label{eqn:laplace_mean}\\
\Sigma_{\textup{MAP}} &= \left(\nabla^2_{\bm{r}} f(\bm{r}) \Bigr| _{\bm{\hat{r}}}\right)^{-1}, \label{eqn:laplace_cov} \text{   where the optimization problem in (\ref{eqn:laplace_mean}) is convex,} \nonumber
\end{flalign}
and $\alpha > 0$ is a tunable hyperparameter that influences the balance between exploration and exploitation.

\subsection{EXTENDING PROOF TECHNIQUES TO OTHER CREDIT ASSIGNMENT MODELS}\label{sec:credit_assignment_extend}

Currently, our proof methodology treats only the Bayesian linear regression credit assignment model. Extending it to other credit assignment models, such as the Gaussian process-based and Bayesian logistic regression methods detailed above, is an important direction for future work. Recall that our theoretical analysis follows three main steps:
\begin{enumerate}
    \item Prove that \algo~is asymptotically-consistent, that is, over time, the probability that \algo~selects the optimal policy approaches 1 (Appendix \ref{sec:asy_consistency}).
    \item Assume that in each iteration $i$, policy $\pi_{i1}$ is drawn from a fixed distribution while policy $\pi_{i2}$ is selected by \algo. Then, asymptotically bound the one-sided regret for $\pi_{i2}$ (Appendix \ref{sec:bound_one_sided_fixed}).
    \item Assume that policy $\pi_{i1}$ is drawn from a drifting but converging distribution while policy $\pi_{i2}$ is selected by \algo. Then, asymptotically bound the one-sided regret for $\pi_{i2}$ (Appendix \ref{sec:bound_one_sided_converging}).
\end{enumerate}

Notably, this proof outline does not depend upon any specific credit assignment model definition, and thus could likely extend to many models. Step 1) requires asymptotic consistency of the credit assignment model; this is not a restrictive requirement, as a non-asymptotically consistent model would not yield sublinear regret. Step 2) fixes one of the two distributions from which policies are drawn; this removes some mathematical difficulties inherent in analyzing preference-based sampling, making the required analysis more similar to the numeric feedback setting. Step 3) depends mainly upon continuity arguments.

The information-theoretic perspective used to prove 2) and 3) likely applies to a wide class of credit assignment models. For instance, recent work has applied the bandit analysis framework in \cite{russo2016information} to bandits with rewards generated via a class of general link functions \citep{dong2018information}. In particular, the information ratio has been studied for the logistic bandit problem \citep{dong2019performance}; we expect that this work could be extended toward analyzing credit assignment via Bayesian logistic regression in the PBRL setting.

Another interesting direction would be to analyze the information ratio for the state transition dynamics model. Bounding this quantity would strengthen the results significantly, since our current analysis only considers dynamics convergence asymptotically. Such a result would also be of independent interest outside of preference-based learning.

Finally, the concept of approximate linearity \citep{sui2017multi} could perhaps help to bridge the gap between the preference and absolute-reward domains, as it has previously done in the bandit setting, and could help to extend existing proof techniques toward a wider class of link functions. In practice, we expect that \algo~would perform well with any asymptotically-consistent credit assignment model that sufficiently captures users' preference behavior.


\section{WHY THE INFORMATION-THEORETIC REGRET ANALYSIS?}\label{sec:why_info}

Several existing regret analyses in the linear bandit domain \citep{abbasi2011improved, agrawal2013thompson, abeille2017linear} utilize martingale concentration properties introduced by \cite{abbasi2011improved}. In these analyses, a key step requires sublinearly upper-bounding an expression of the form (e.g. Lemma 11 in \cite{abbasi2011improved}, Prop. 2 in \cite{abeille2017linear}):
\begin{equation}\label{eqn:abbasi_lemma_11}
    \sum_{i = 1}^n \bm{x}_i^T \left(\lambda I + \sum_{s = 1}^{i - 1}\bm{x}_s \bm{x}_s^T \right)^{-1}\bm{x}_i,
\end{equation}
where $\lambda \ge 1$ and $\bm{x}_i$ is the observation vector in iteration $i$. We will demonstrate that in the preference-feedback setting, the analogous quantity cannot always be sublinearly upper-bounded. Consider the setting defined in Section \ref{sec:problem_statement}, with Bayesian linear regression credit assignment. Under preference feedback, we assume that the probability that one trajectory is preferred to another is fully determined by the \textit{difference} between the trajectories' total rewards: on iteration $i$, the algorithm receives a pair of observations $\bm{x}_{i1}$, $\bm{x}_{i2}$, with $\bm{x}_i := \bm{x}_{i2} - \bm{x}_{i1}$, and a preference generated according to $P(\tau_{i2} \succ \tau_{i1}) = \bm{\rr}^T (\bm{x}_{i2} - \bm{x}_{i1}) + \frac{1}{2}$. Thus, only \textit{differences} between compared trajectory feature vectors yield information about the rewards. Under this assumption, one can show that applying the martingale techniques yields the following variant of \eqref{eqn:abbasi_lemma_11}:
\begin{equation}\label{eqn:abbasi_lemma_11_pref}
    \sum_{i = 1}^n \sum_{j = 1}^2 \bm{x}_{ij}^T \left(\lambda I + \sum_{s = 1}^{i - 1}\bm{x}_s \bm{x}_s^T \right)^{-1}\bm{x}_{ij}.
\end{equation}
This is because the expression within the matrix inverse comes from the posterior---and learning occurs with respect to the observations $\bm{x}_i$---while regret is incurred with respect to $\bm{x}_{i1}$ and $\bm{x}_{i2}$; in contrast, in the non-preference case \eqref{eqn:abbasi_lemma_11}, learning and regret both occur with respect to the same vectors $\bm{x}_i$.

To see that \eqref{eqn:abbasi_lemma_11_pref} does not necessarily have a sublinear upper bound, consider a deterministic MDP as a counterexample. For the regret to have a sublinear upper-bound, the probability of choosing the optimal policy must approach 1. In a fully deterministic MDP, this means that $P(\bm{x}_{i1} = \bm{x}_{i2}) \longrightarrow 1$ as $i \longrightarrow \infty$, and thus $P(\bm{x}_i = 0) \longrightarrow 1$ as $i \longrightarrow \infty$. Clearly, in this case, the inverted quantity in \eqref{eqn:abbasi_lemma_11_pref} acquires nonzero terms at a rate that decays in $n$, and so \eqref{eqn:abbasi_lemma_11_pref} does not have a sublinear upper bound.

Intuitively, to be able to upper-bound Equation \eqref{eqn:abbasi_lemma_11_pref}, we would need the observations $\bm{x}_{i1}, \bm{x}_{i2}$ to contribute fully toward learning the rewards, rather than the contribution coming only from their difference. Notice that in the counterexample, even when the optimal policy is selected increasingly-often, corresponding to a low regret, \eqref{eqn:abbasi_lemma_11_pref} cannot be sublinearly bounded. In contrast, \cite{russo2016information} introduces a more direct approach for quantifying the trade-off between instantaneous regret and information gained, as encapsulated by the \textit{information ratio} defined therein; our theoretical analysis is thus based upon this framework.


\section{ADDITIONAL EXPERIMENTAL DETAILS}\label{sec:additional_experiments}

\begin{flushleft}
Python code for reproducing the experiments in this work is located at: https://github.com/ernovoseller/DuelingPosteriorSampling.
\end{flushleft}

Experiments were conducted in three simulated environments, described in Section \ref{sec:experiments}: RiverSwim and random MDPs \citep{osband2013more} and the simplified version of the Mountain Car problem described in \cite{wirth2017efficient}. We use a fixed episode horizon of 50 in the first two cases, while for Mountain Car, episodes have a maximum length of 500, but terminate sooner if the agent reaches the goal state. Figures \ref{fig:RiverSwim_performance_curves}, \ref{fig:RandomMDP_performance_curves}, and \ref{fig:MountainCar_performance_curves} display performance in the three environments for the five degrees of user preference noise evaluated. Experiments were run on an Ubuntu 16.04.3 machine with 32 GB of RAM and an Intel i7 processor. Some experiments were also run on an AWS server.

\begin{figure*}[ht]
  \centering
\subfloat[][$c = 0.0001$, logistic]{\includegraphics[width = 0.34\linewidth]{RiverSwim_log_noise_0,001_font_18.png}}
\subfloat[][$c = 1$, logistic]{\includegraphics[width = 0.34\linewidth]{RiverSwim_log_noise_1,0_font_18.png}}
\subfloat[][$c = 2$, logistic]{\includegraphics[width = 0.34\linewidth]{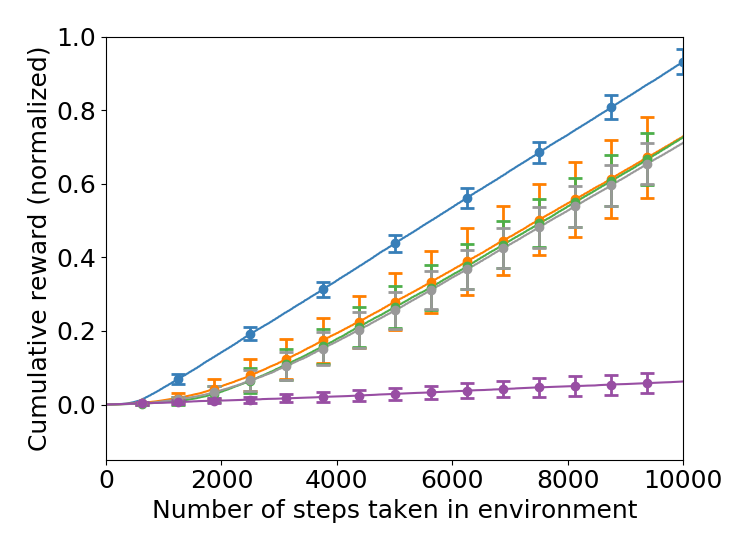}} \\
\subfloat[][$c = 10$, logistic]{\includegraphics[width = 0.34\linewidth]{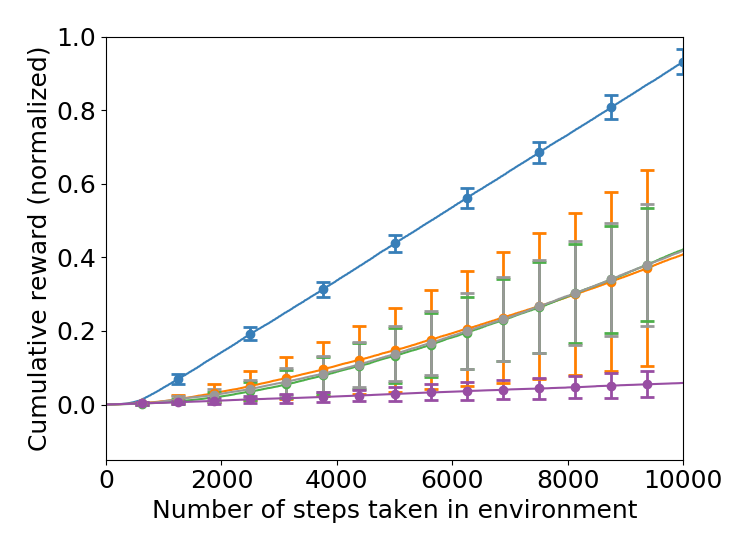}}
\subfloat[][$c = 100$, linear]{\includegraphics[width = 0.34\linewidth]{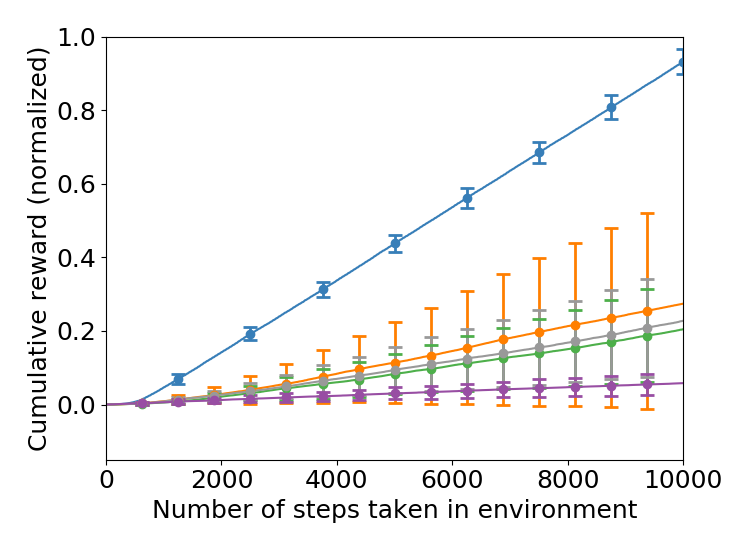}}
  \caption{Empirical performance of \algo~in the RiverSwim environment. Plots display  mean +/- one standard deviation over 100 runs of each algorithm tested. Normalization is with respect to the total reward achieved by the optimal policy. Overall, we see that \algo~performs well and is robust to the choice of credit assignment model.}
  \label{fig:RiverSwim_performance_curves}
\end{figure*}

We detail the ranges of hyperparameter values tested for the different \algo~credit assignment models, as well as the particular hyperparameters used in the displayed performance curves. Hyperparameters were tuned by considering mean performance over 30 experiment repetitions for each parameter setting considered; we only used the least-noisy preference feedback (logistic noise, $c = 0.001$) to tune the preferences; this value of $c$ is small enough that the preferences are close to deterministic, except that the preferences are uniformly-random in tie cases. For both Gaussian process-based credit assignment models, we use the squared exponential kernel:
\begin{equation*}
    \mathcal{K}(x_i, x_j) = \sigma_f^2\exp{\left(-\frac{1}{2}\left(\frac{x_i-x_j}{l}\right)^2\right)}+\sigma_n^2\delta_{ij},
\end{equation*}
where $\sigma_f$ is the signal variance, $l$ is the kernel lengthscale, $\sigma_n$ is the noise variance, and $\delta_{ij}$ is the Kronecker delta function. Please see Appendix \ref{sec:credit_assignment_models} for definitions of the other hyperparameters in the credit assignment models. Tables \ref{table:hyperparams_RiverSwim}, \ref{table:hyperparams_randomMDP}, and \ref{table:hyperparams_mountain_car} display both the tested ranges and optimized values (those appearing in the performance curves) for each case.

The dynamics model, meanwhile, has a Dirichlet prior and posterior. Not to assume domain knowledge that differentiates the state-action pairs, we set all prior parameters of the Dirichlet model to be equal; for the RiverSwim and Random MDP environments, the prior is set to 1 for each state-action, creating a uniform prior over all dynamics models. This is a reasonable choice with small numbers of states and actions, so we do not optimize over different values. For the Mountain Car problem, smaller prior values perform better because they favor sparse dynamics distributions. For this environment, we test prior parameters ranging from 0.0001 to 1, and found 0.0005 to be the best-performing value among those tested.

The EMPC algorithm \citep{wirth2013policy} has two hyperparameter values, $\alpha$ and $\eta$. We optimize both of these jointly via a grid search over values of $(0.1, 0.2, \ldots, 0.9)$, with 100 repetitions of each pair of values. The best-performing hyperparameter values (i.e. those achieving the highest total reward) are displayed in Table \ref{table:EPMC_hyperparameters}; these are the hyperparameter values depicted in the performance curve plots.

Finally, Figures \ref{fig:RiverSwim_hyperparams}, \ref{fig:RandomMDP_hyperparams}, and \ref{fig:MountainCar_hyperparams} illustrate how \algo's performance varies as the hyperparameters are modified over a set of representative values from the ranges that we tested. These plots demonstrate that \algo~is largely robust across many choices of model hyperparameters.

\begin{table}[ht]
\caption{Hyperparameters for the EPMC baseline algorithm. Each table element shows the best-performing $\alpha$/$\eta$ values for the corresponding simulation domain and noise parameter.}
\label{table:EPMC_hyperparameters}
\begin{center}
\begin{tabular}{llllll}
\multicolumn{1}{c}{\bf NOISE}  &\multicolumn{1}{c}{\bf LOGISTIC, 10} &\multicolumn{1}{c}{\bf LOGISTIC, 2}&\multicolumn{1}{c}{\bf LOGISTIC, 1}&\multicolumn{1}{c}{\bf LOGISTIC, 0.0001}&\multicolumn{1}{c}{\bf LINEAR}\\
\hline \\
RiverSwim       & 0.1/0.8 & 0.3/0.7 & 0.1/0.2 & 0.8/0.8 & 0.3/0.1 \\
Random MDPs     & 0.2/0.2 & 0.7/0.7 & 0.6/0.4 & 0.2/0.8 & 0.7/0.1\\
\hline \\
\multicolumn{1}{c}{\bf NOISE}  &\multicolumn{1}{c}{\bf LOGISTIC, 100} &\multicolumn{1}{c}{\bf LOGISTIC, 20}&\multicolumn{1}{c}{\bf LOGISTIC, 10}&\multicolumn{1}{c}{\bf LOGISTIC, 0.0001}&\multicolumn{1}{c}{\bf LINEAR}\\
\hline \\
MountainCar      & 0.1/0.8 & 0.1/0.7 & 0.1/0.6 & 0.1/0.4 & 0.2/0.5 \\
\end{tabular}
\end{center}
\end{table}

\begin{table}[ht]
\caption{Credit assignment hyperparameters for the RiverSwim Environment}
\label{table:hyperparams_RiverSwim}
\begin{center}
\begin{tabular}{llll}
\multicolumn{1}{c}{\bf MODEL}  &\multicolumn{1}{c}{\bf HYPERPARAMETER} &\multicolumn{1}{c}{\bf RANGE TESTED}  &\multicolumn{1}{c}{\bf OPTIMIZED VALUE} \\
\hline \\
Bayesian linear regression         & $\sigma$ & [0.05, 5] & 0.5 \\
               & $\lambda$ & [0.01, 10] & 0.1 \\
 & & & \\
GP regression         & $\sigma_f^2$ & [0.001, 0.5] & 0.1 \\
               & $l$ & [0, 0] ([state, action]) & 0 \\
               & $\sigma_n^2$ & [0.0001, 0.1] & 0.001 \\
 & & & \\
GP preference (special case:    & $\lambda = \sigma_f^2 + \sigma_n^2$ & [0.1, 30] & 1 \\
Bayesian logistic regression)               & $\alpha$ & [0.01, 1] & 1 \\
 & & & \\
GP preference (varying $c$)       & $c$ & [1, 13] & N/A \\ 
         & $\sigma_f^2$ & [1] &  \\
               & $l$ & [0, 0] ([state, action]) &  \\
               & $\sigma_n^2$ & [0.001] \\
               & $\alpha$ & [1] \\
\end{tabular}
\end{center}
\end{table}

\begin{table}[ht]
\caption{Credit assignment hyperparameters for the Random MDP Environment}
\label{table:hyperparams_randomMDP}
\begin{center}
\begin{tabular}{llll}
\multicolumn{1}{c}{\bf MODEL}  &\multicolumn{1}{c}{\bf HYPERPARAMETER} &\multicolumn{1}{c}{\bf RANGE TESTED}  &\multicolumn{1}{c}{\bf OPTIMIZED VALUE} \\
\hline \\
Bayesian linear regression         & $\sigma$ & [0.05, 5] & 0.1 \\
               & $\lambda$ & [0.01, 20] & 10 \\
 & & & \\
GP regression         & $\sigma_f^2$ & [0.001, 1] & 0.05 \\
               & $l$ & [0, 0] ([state, action]) & 0 \\
               & $\sigma_n^2$ & [0.0001, 0.1] & 0.0005 \\
 & & & \\
GP preference (special case:    & $\lambda = \sigma_f^2 + \sigma_n^2$ & [1, 15] & 0.1 \\
Bayesian logistic regression)               & $\alpha$ & [0.01, 1] & 0.01 \\
 & & & \\
GP preference (varying $c$)       & $c$ & [0.0001, 1000] & N/A \\ 
         & $\sigma_f^2$ & [1] &  \\
               & $l$ & [0, 0] ([state, action]) &  \\
               & $\sigma_n^2$ & [0.03] \\
               & $\alpha$ & [1] \\
\end{tabular}
\end{center}
\end{table}

\begin{figure*}[ht]
  \centering
\subfloat[][$c = 0.0001$, logistic]{\includegraphics[width = 0.34\linewidth]{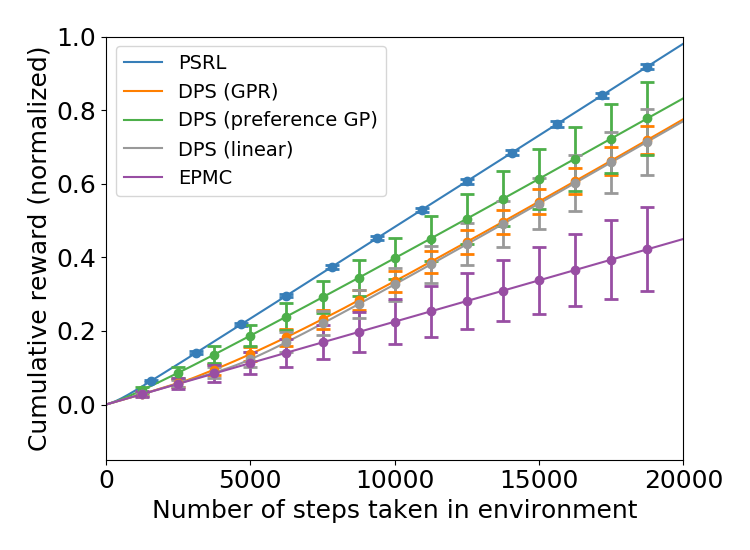}}
\subfloat[][$c = 1$, logistic]{\includegraphics[width = 0.34\linewidth]{RandomMDPs_log_noise_1,0_font_18.png}}
\subfloat[][$c = 2$, logistic]{\includegraphics[width = 0.34\linewidth]{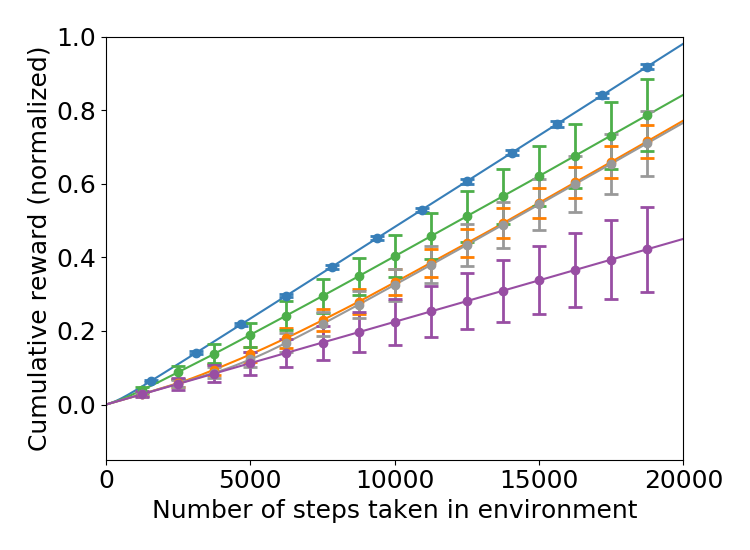}} \\
\subfloat[][$c = 10$, logistic]{\includegraphics[width = 0.34\linewidth]{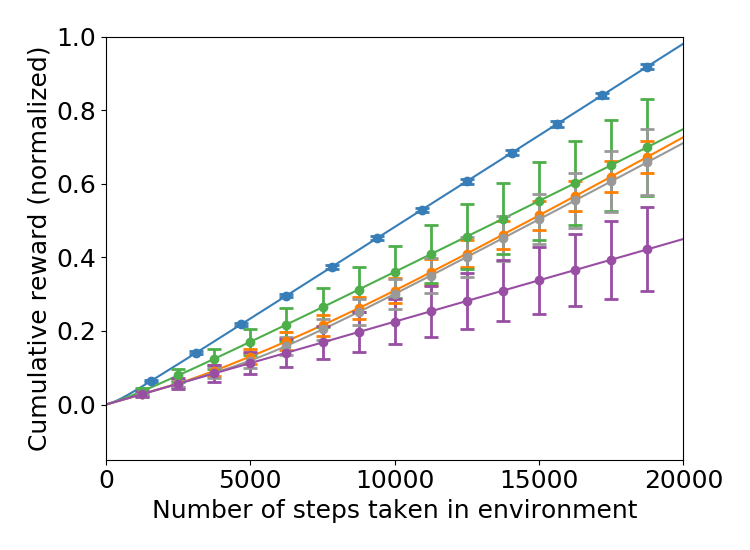}}
\subfloat[][$c$ varies, linear]{\includegraphics[width = 0.34\linewidth]{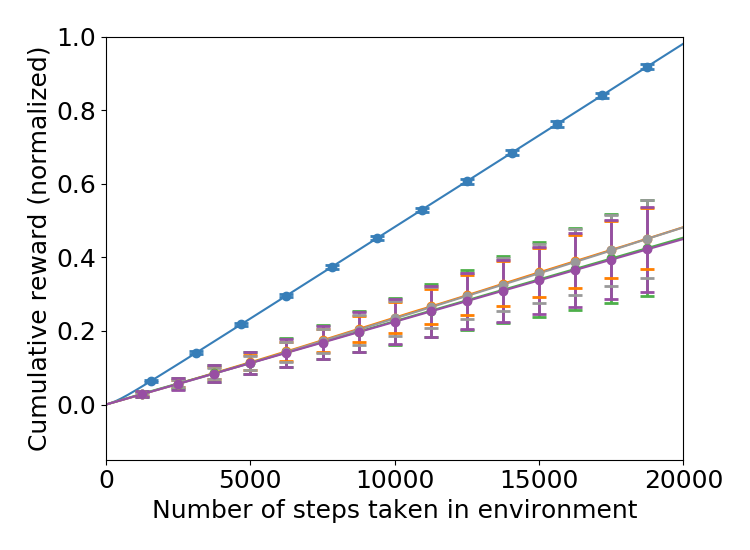}}
  \caption{Empirical performance of \algo~in the Random MDP environment. Plots display mean +/- one standard deviation over 100 runs of each algorithm tested. Normalization is with respect to the total reward achieved by the optimal policy. Overall, we see that \algo~performs well and is robust to the choice of credit assignment model.}
  \label{fig:RandomMDP_performance_curves}
\end{figure*}

\begin{table}[ht]
\caption{Credit assignment hyperparameters for the Mountain Car Environment}
\label{table:hyperparams_mountain_car}
\begin{center}
\begin{tabular}{llll}
\multicolumn{1}{c}{\bf MODEL}  &\multicolumn{1}{c}{\bf HYPERPARAMETER} &\multicolumn{1}{c}{\bf RANGE TESTED}  &\multicolumn{1}{c}{\bf OPTIMIZED VALUE} \\
\hline \\
Bayesian linear regression         & $\sigma$ & [0.001, 30] & 10 \\
               & $\lambda$ & [0.001, 10] & 1 \\
 & & & \\
GP regression         & $\sigma_f^2$ & [0.0001, 10] & 0.01 \\
               & $l$ & [$x$, $x$, 0], $x \in [1, 3]$  & $x = 2$ \\
               & & ([position, velocity, action]) & \\
               & $\sigma_n^2$ & [1e-7, 0.01] & 1e-5 \\
 & & & \\
GP preference (special case:    & $\lambda = \sigma_f^2 + \sigma_n^2$ & [0.0001, 10] & 0.0001 \\
Bayesian logistic regression)               & $\alpha$ & [0.0001, 1] & 0.01 \\
 & & & \\
GP preference (varying $c$)       & $c$ & [10, 10000] & N/A \\ 
         & $\sigma_f^2$ & [1] &  \\
               & $l$ & [2, 2, 0]  &  \\
               & & ([position, velocity, action]) & \\
               & $\sigma_n^2$ & [0.001] \\
               & $\alpha$ & [1] \\
\end{tabular}
\end{center}
\end{table}

\begin{figure*}[ht]
  \centering
\subfloat[][$c = 0.0001$, logistic]{\includegraphics[width = 0.34\linewidth]{MountainCar_log_noise_0,001_font_18.png}}
\subfloat[][$c = 10$, logistic]{\includegraphics[width = 0.34\linewidth]{MountainCar_log_noise_10,0_font_18.png}}
\subfloat[][$c = 20$, logistic]{\includegraphics[width = 0.34\linewidth]{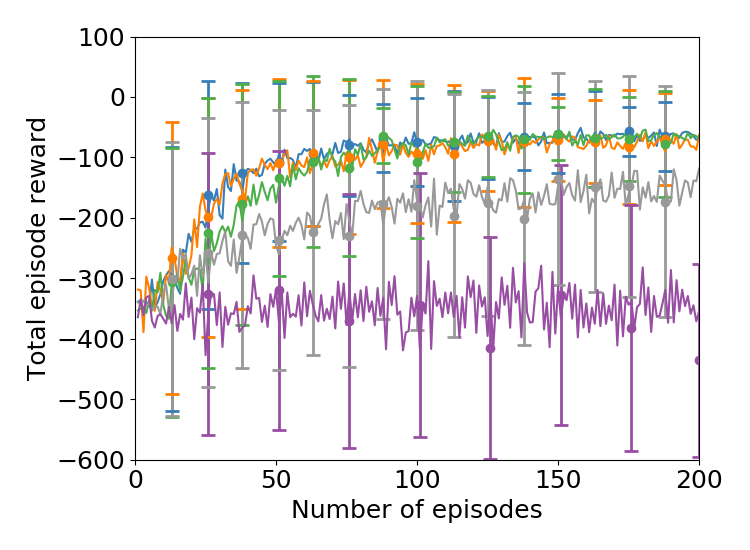}} \\
\subfloat[][$c = 100$, logistic]{\includegraphics[width = 0.34\linewidth]{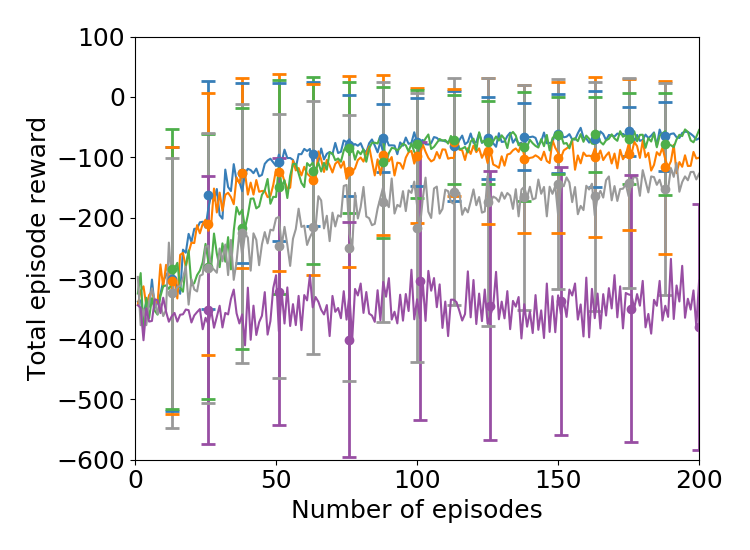}}
\subfloat[][$c = 1,000$, linear]{\includegraphics[width = 0.34\linewidth]{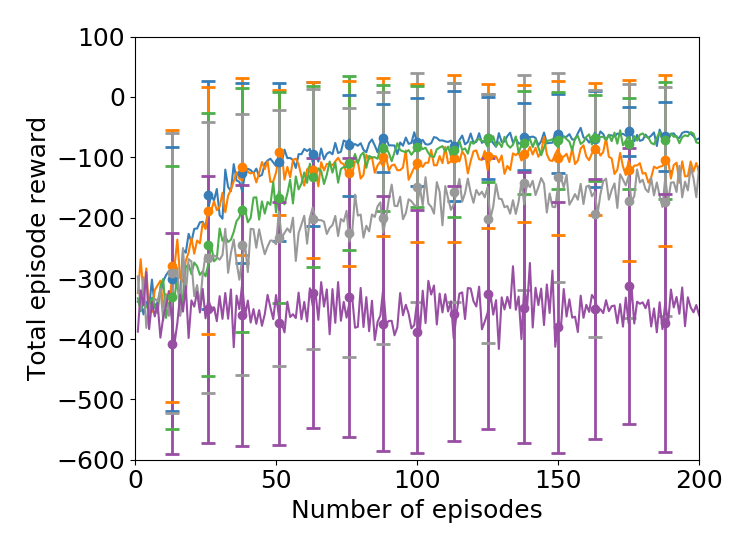}}
\hspace{5mm}\subfloat[][Legend]{\includegraphics[width = 0.25\linewidth]{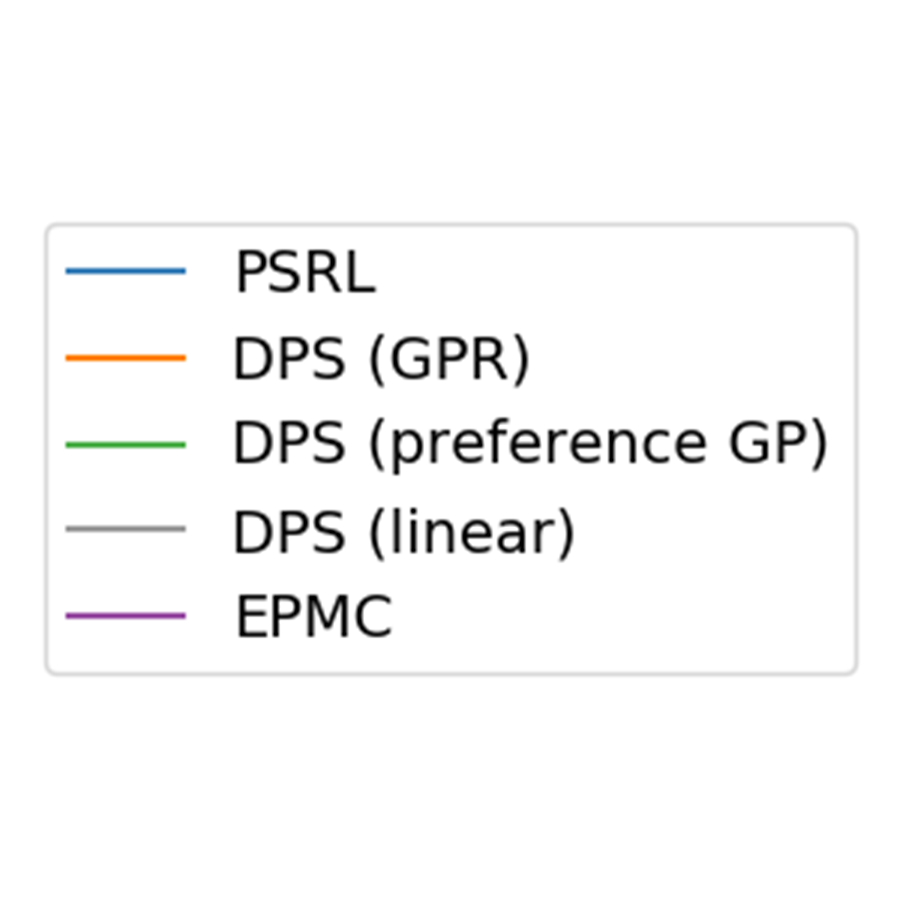}}
  \caption{Empirical performance of \algo~in the Mountain Car environment. Plots display mean +/- one standard deviation over 100 runs of each algorithm tested. Overall, we see that \algo~performs well and is robust to the choice of credit assignment model.}
  \label{fig:MountainCar_performance_curves}
\end{figure*}

\begin{figure*}[ht]
  \centering
\subfloat[][Bayesian linear regression]{\includegraphics[width = 0.4\linewidth]{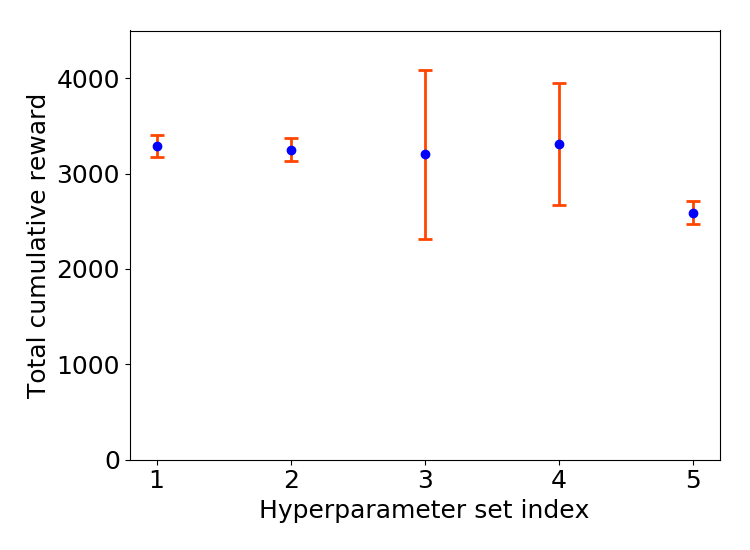}} \hspace{10mm}
\subfloat[][Gaussian process regression]{\includegraphics[width = 0.4\linewidth]{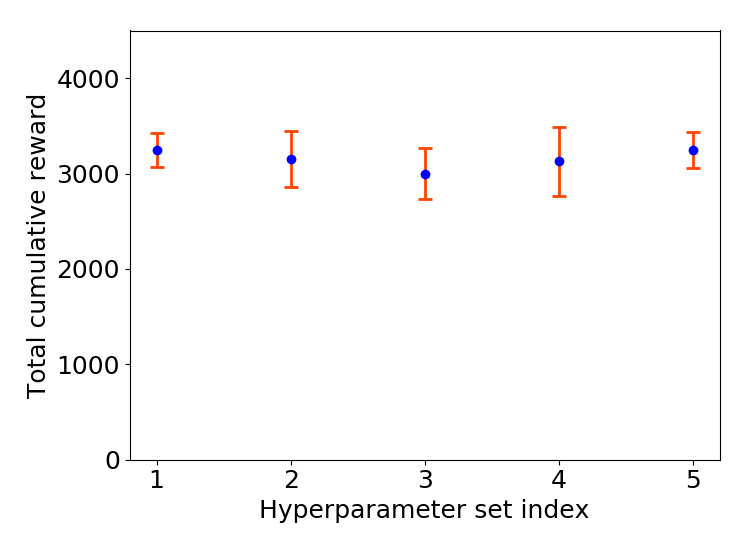}} \\
\subfloat[][Gaussian process preference model (special case: Bayesian logistic regression)]{\includegraphics[width = 0.4\linewidth]{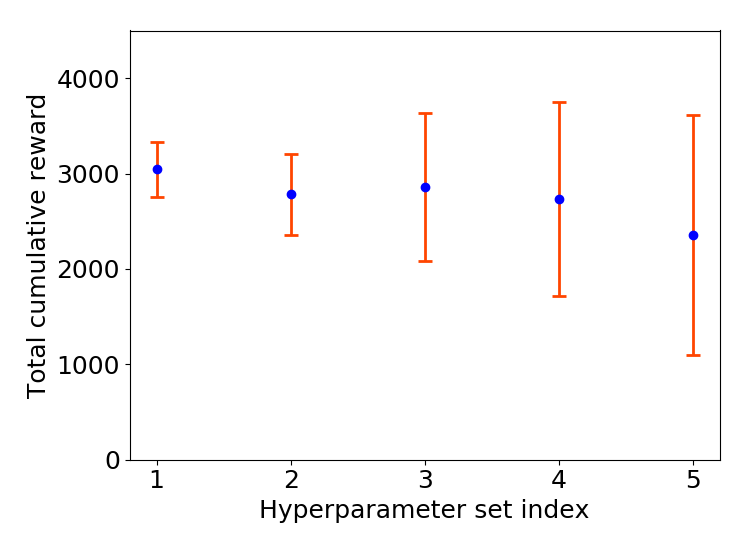}} \hspace{10mm}
\subfloat[][Gaussian process preference model (varying $c$)]{\includegraphics[width = 0.4\linewidth]{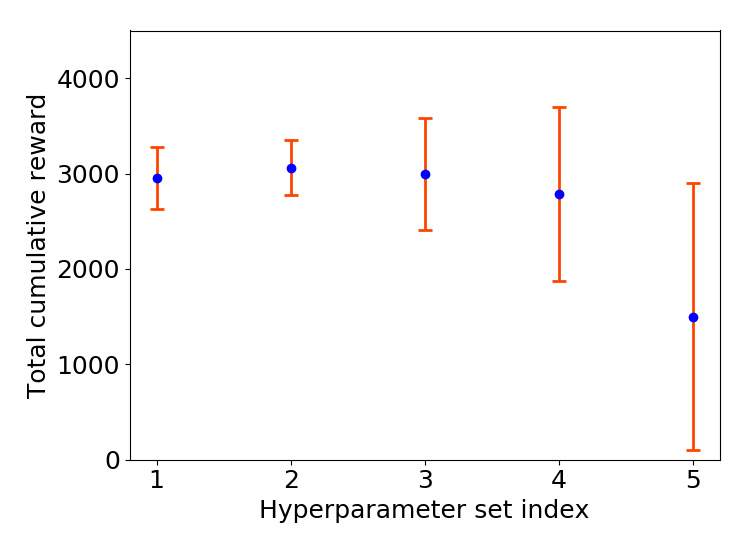}}
  \caption{Empirical performance of \algo~in the RiverSwim environment for different hyperparameter combinations. Plots display mean +/- one standard deviation over 30 runs of each algorithm tested with logistic user noise and $c = 0.001$. Overall, we see that \algo~is robust to the choice of hyperparameters. The hyperparameter values depicted in each plot are (from left to right): for Bayesian linear regression, $(\sigma, \lambda) = \{(0.5, 0.1), (0.5, 10), (0.1, 0.1), (0.1, 10), (1, 0.1)\}$; for GP regression, $(\sigma_f^2, \sigma_n^2) = \{(0.1, 0.001), (0.1, 0.1), (0.01, 0.001), (0.001, 0.0001), (0.5, 0.1)\}$; for Bayesian logistic regression (special case of the GP preference model), $(\lambda, \alpha) = \{(1, 1), (30, 1), (20, 0.5), (1, 0.5), (30, 0.1)\}$; and additionally for the GP preference model, $c \in \{0.5, 1, 2, 5, 13\}$. See Table \ref{table:hyperparams_RiverSwim} for the values of any hyperparameters not specifically mentioned here.}
  \label{fig:RiverSwim_hyperparams}
\end{figure*}

\begin{figure*}[ht]
  \centering
\subfloat[][Bayesian linear regression]{\includegraphics[width = 0.4\linewidth]{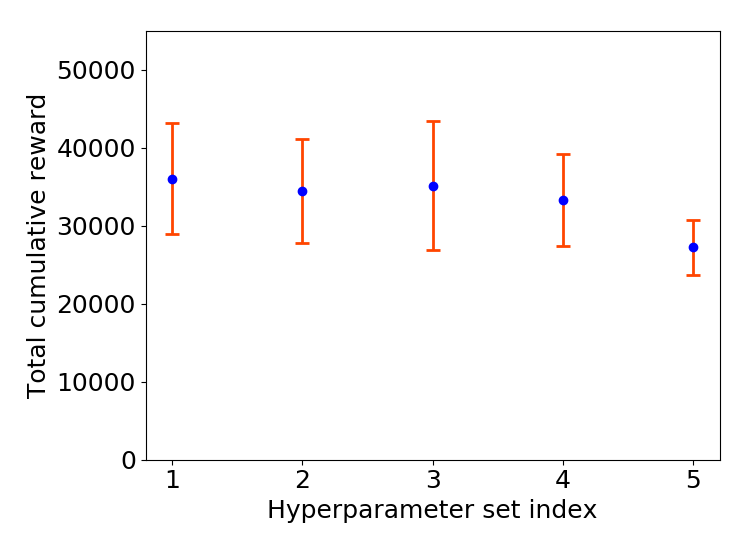}} \hspace{10mm}
\subfloat[][Gaussian process regression]{\includegraphics[width = 0.4\linewidth]{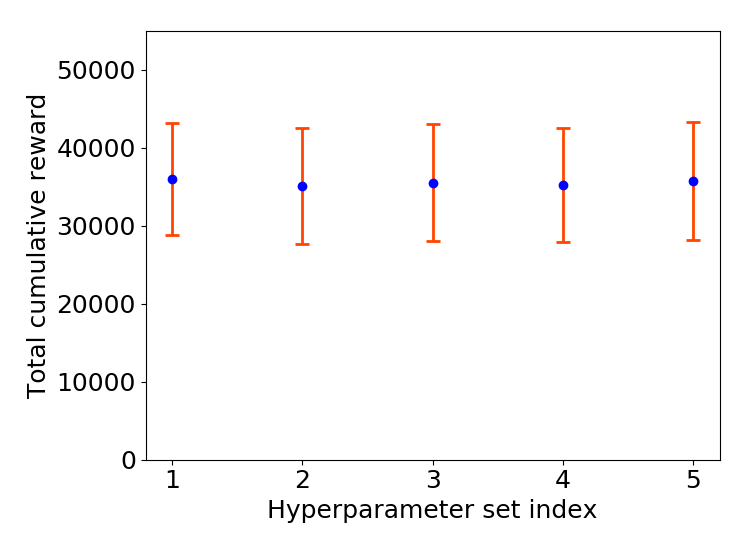}} \\
\subfloat[][Gaussian process preference model (special case: Bayesian logistic regression)]{\includegraphics[width = 0.4\linewidth]{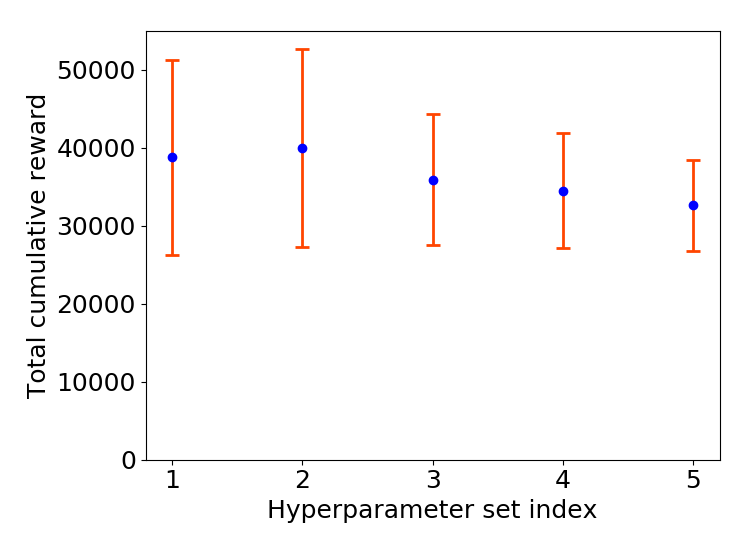}} \hspace{10mm}
\subfloat[][Gaussian process preference model (varying $c$)]{\includegraphics[width = 0.4\linewidth]{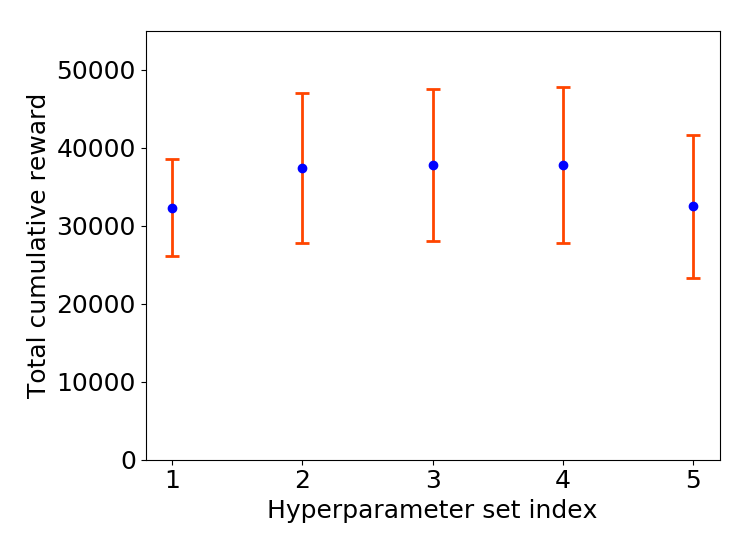}}
  \caption{Empirical performance of \algo~in the Random MDP environment for different hyperparameter combinations. Plots display mean +/- one standard deviation over 30 runs of each algorithm tested with logistic user noise and $c = 0.001$. Overall, we see that \algo~is robust to the choice of hyperparameters. The hyperparameter values depicted in each plot are (from left to right): for Bayesian linear regression, $(\sigma, \lambda) = \{(0.1, 10), (0.1, 0.1), (0.05, 0.01), (0.5, 20), (1, 10)\}$; for GP regression, $(\sigma_f^2, \sigma_n^2) = \{(0.05, 0.0005), (0.001, 0.0001), (0.05, 0.1), (0.001, 0.0005), (1, 0.1)\}$; for Bayesian logistic regression (special case of the GP preference model), $(\lambda, \alpha) = \{(0.1, 0.01), (1, 0.01), (0.1, 1), (30, 0.1), (5, 0.5)\}$; and additionally for the GP preference model, $c \in \{1, 10, 15, 19, 100\}$. See Table \ref{table:hyperparams_randomMDP} for the values of any hyperparameters not specifically mentioned here.}
  \label{fig:RandomMDP_hyperparams}
\end{figure*}

\begin{figure*}[ht]
  \centering
\subfloat[][Bayesian linear regression]{\includegraphics[width = 0.4\linewidth]{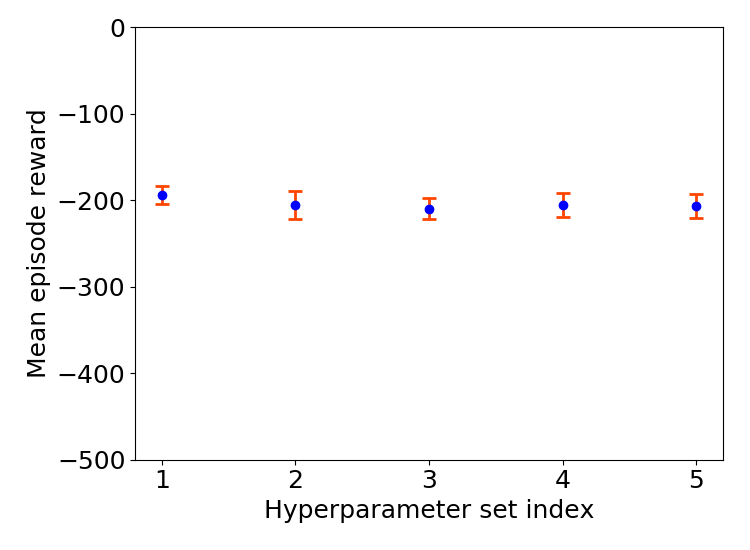}} \hspace{10mm}
\subfloat[][Gaussian process regression]{\includegraphics[width = 0.4\linewidth]{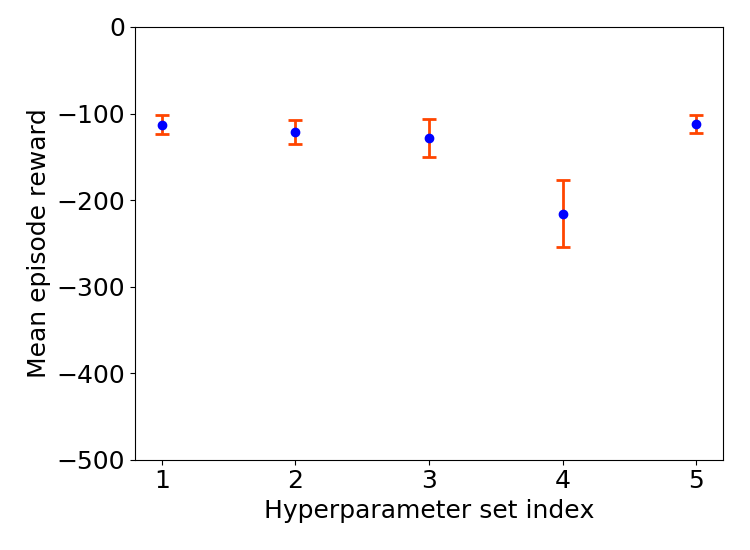}} \\
\subfloat[][Gaussian process preference model (special case: Bayesian logistic regression)]{\includegraphics[width = 0.4\linewidth]{MountainCar_hyperparams_GPR.png}} \hspace{10mm}
\subfloat[][Gaussian process preference model (varying $c$)]{\includegraphics[width = 0.4\linewidth]{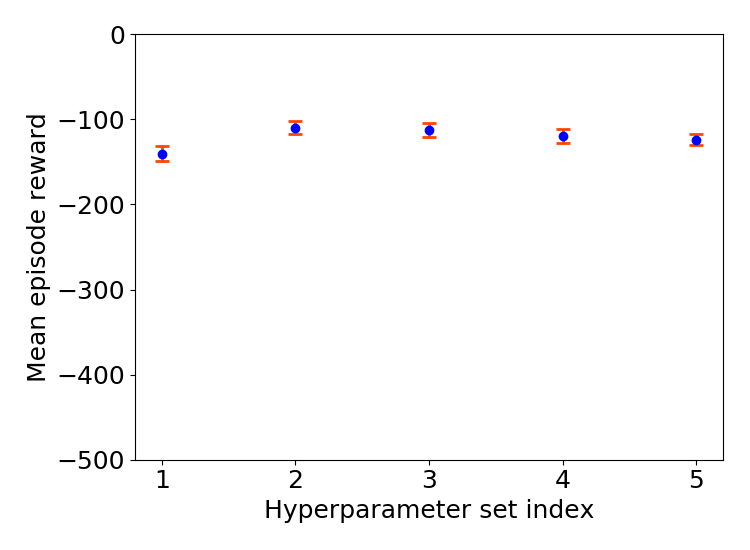}}
  \caption{Empirical performance of \algo~in the Mountain Car environment for different hyperparameter combinations. Plots display mean +/- one standard deviation over 30 runs of each algorithm tested with logistic user noise and $c = 0.001$. Overall, we see that \algo~is robust to the choice of hyperparameters. The hyperparameter values depicted in each plot are (from left to right): for Bayesian linear regression, $(\sigma, \lambda) = \{(10, 1), (10, 10), (30, 0.001), (0.001, 10), (0.1, 0.1)\}$; for GP regression, $(\sigma_f^2, l, \sigma_n^2) = \{(0.01, 2, 10^{-5}), (0.01, 1, 10^{-5}), (0.1, 2, 0.01), (1, 2, 0.001), (0.001, 3, 10^{-6})\}$; for Bayesian logistic regression (special case of the GP preference model), $(\lambda, \alpha) = \{(0.0001, 0.01), (0.1, 0.01), (0.0001, 0.0001), (0.001, 0.0001), (0.001, 0.01)\}$; and additionally for the GP preference model, $c \in \{10, 300, 400, 700, 1000\}$. See Table \ref{table:hyperparams_mountain_car} for the values of any hyperparameters not specifically mentioned here.}
  \label{fig:MountainCar_hyperparams}
\end{figure*}

\end{document}